%% file: main.tex
\documentclass[wcp]{jmlr}

\usepackage{longtable}

\usepackage{booktabs}

\usepackage{bm}
\usepackage{makecell}
\usepackage[normalem]{ulem}

\usepackage{outlines}
\usepackage{multirow}
\usepackage{caption}
\usepackage{relsize}

\usepackage{tikz}
\usepackage{pgfplotstable}
\usepackage{amsfonts}
\usepackage{xifthen}
\usepackage{wrapfig}

\usepgfplotslibrary{groupplots}
\usetikzlibrary{shapes,fit,backgrounds,arrows,decorations.markings,arrows.meta,positioning,decorations.pathreplacing,calligraphy,shapes.misc,matrix}

\pgfdeclarelayer{background}
\pgfdeclarelayer{foreground}
\pgfsetlayers{background,main,foreground}
\pgfplotsset{
    line and fill/.style={
        legend image code/.code={%
          \draw [##1,fill=none, thick] (0mm,0mm) -- (4mm,0mm);
        },
    },
}

\newcommand{\ppw}[1]{$\text{PP}_{#1}^{\text{Obs}}$}
\newcommand{\ppwn}[1]{$\text{PP}_{#1}^{\text{Obs, Ns}}$}
\newcommand{\ppwna}[1]{$\text{PP}_{#1}^{\text{Obs, Ns, Att}}$}
\newcommand{\lj}[3]{$\text{LJ}_{#1}^{#2,#3}$}
\newcommand*{\eg}{\emph{e.g.}}
\newcommand*{\ie}{\emph{i.e.}}

\newcommand*{\proposedre}{{\coolname}\textsubscript{on}}
\newcommand*{\proposedac}{{\coolname}\textsubscript{off}}
\newcommand{\std}[1]{\footnotesize{\color{black}$\pm$#1}}
\newcommand{\errorband}[6][]{
    \addplot [draw=none, stack plots=y, forget plot] 
      table [x={#3},y expr=\thisrow{#4}-\thisrow{#5}] {#2};
    \addplot [fill=gray!40, stack plots=y, opacity=0.16, #1, draw=none, forget plot] 
      table [x={#3},y expr=2*\thisrow{#5}] {#2} \closedcycle;
    \addplot [stack plots=y, draw=none, forget plot] 
      table [x={#3},y expr=-(\thisrow{#4}+\thisrow{#5})] {#2};
    \addplot [forget plot, thick, #1, fill=none] 
      table [x={#3},y expr=\thisrow{#4}] {#2};
    \addlegendimage{line and fill,#1}
    \addlegendentry{#6}
}
\newcommand{\parag}[1]{\textbf{#1}\hspace{0.3cm}}
\newcommand{\subparag}[2]{
    \ifthenelse{\isempty{#2}}
        {\textit{#1}\hspace{.2cm}}
        {(#2) \textit{#1}\hspace{.2cm}}
}

\newcommand{\coolname}{{\fontfamily{lmtt}\selectfont LAUREL}}

\input{math_commands}

\input{data/ablation_arch_a.tex}
\input{data/ablation_arch_b.tex}
\input{data/ablation_noise.tex}
\input{data/ablation_rss.tex}
\input{data/ablation_when.tex}
\input{data/sota_steps_a.tex}
\input{data/sota_steps_b.tex}
\input{data/sota_pp_tarmac.tex}
\input{data/sota_lj.tex}
\input{data/traj_comm.tex}
\input{data/dummy.tex}
\input{data/bw_limit_steps.tex}

\definecolor{color_env_inner}{HTML}{AAF0D1}
\definecolor{color_env_outer}{HTML}{64AA8B}
\definecolor{color_nn_inner}{HTML}{AED6F1}
\definecolor{color_nn_outer}{HTML}{21618C}
\definecolor{color_light_purple}{HTML}{C241C1}
\definecolor{color_nn_background}{HTML}{E0E0E0}

\definecolor{color_commnet_}{HTML}{2ca02c}
\definecolor{color_ic3net_}{HTML}{1f77b4}
\definecolor{color_laurel_on}{HTML}{ff7f0e}
\definecolor{color_qmix_}{HTML}{9467bd}
\definecolor{color_tarmac_}{HTML}{8c564b}
\definecolor{color_laurel_off}{HTML}{e377c2}

\makeatletter
\let\Ginclude@graphics\@org@Ginclude@graphics 
\makeatother

\jmlrvolume{}
\jmlryear{2022}

\title[Learning Practical Communication Strategies in Cooperative MARL]{Learning Practical Communication Strategies in Cooperative Multi-Agent Reinforcement Learning}

 \author{\Name{Diyi Hu} \Email{diyihu@usc.edu}\\
\Name{Chi Zhang} \Email{zhan527@usc.edu}\\
  \Name{Viktor Prasanna} \Email{prasanna@usc.edu}\\
  \Name{Bhaskar Krishnamachari} \Email{bkrishna@usc.edu}\\
\addr University of Southern California}

\editors{Emtiyaz Khan and Mehmet Gonen}

\begin{document}

\maketitle

\begin{abstract}
    In Multi-Agent Reinforcement Learning, communication is critical to encourage cooperation among agents. Communication in realistic wireless networks can be highly unreliable due to network conditions varying with agents' mobility, and stochasticity in the transmission process. 
    We propose a framework to learn practical communication strategies by addressing three fundamental questions:
    (1) \emph{When}: Agents learn the timing of communication based on not only message importance but also wireless channel conditions.
    (2) \emph{What}: Agents augment message contents with wireless network measurements to better select the game and communication actions. 
    (3) \emph{How}: Agents use a novel neural message encoder to preserve all information from received messages, regardless of the number and order of messages. 
    Simulating standard benchmarks under realistic wireless network settings, we show significant improvements in game performance, convergence speed and communication efficiency compared with state-of-the-art. 
\end{abstract}
\begin{keywords}
Multi-agent Reinforcement Learning; Wireless communication
\end{keywords}

\input{1_intro.tex}

\input{5_related.tex}

\input{3_problem.tex}

\input{4_method.tex}

\input{6_exp.tex}

\section{Conclusion}

We have proposed a general framework to learn practical multi-agent communication strategies. 
Our techniques comprehensively address the fundamental aspects of communication, ``when'', ``what'' and ``how'', with theoretical and empirical justifications. 
The two implementations of our framework  (on- / off-policy) significantly improve the agents' performance, especially in complicated environments.

\bibliography{citation}
\appendix
\input{7_appendix}

\end{document}

%% file: math_commands.tex
\newcommand{\paren}[1]{\left( #1 \right)}
\newcommand{\func}[2][]{
    {\text{\fontfamily{lmtt}\selectfont #1}\left(#2\right)}}
\newcommand{\R}[1][]{\mathbb{R}^{#1}}

\newcommand{\attset}{\Delta}
\newcommand{\att}{\delta}

%% file: data/ablation_arch_a.tex
\pgfplotstableread[col sep=comma]{
0.0,40.0,0.0,40.0,0.0
40000.0,39.95156,0.09687999999999876,39.91406,0.10528421724076313
80000.0,39.79376,0.16316122823759305,39.91406,0.16800716175211108
120000.0,39.77656,0.14171182872293994,39.5875,0.25182901341982084
160000.0,38.62814,0.8646382147464923,39.31562,0.5541002992238858
200000.0,38.829699999999995,0.4296204278197213,38.843740000000004,0.4544916043228951
240000.0,38.80466,0.45782357125862355,38.742180000000005,0.25670242227139156
280000.0,38.959379999999996,0.5333649234811009,38.149980000000006,0.3577284355485322
320000.0,37.310959999999994,0.9092813483185492,37.818760000000005,0.815869236091175
360000.0,37.05624,0.7574045381432571,37.61248,0.883235976169448
400000.0,36.19534,1.123759055313904,37.5953,1.0637240563228791
440000.0,36.74218,0.9769602497543084,36.781220000000005,1.3599940402810602
480000.0,35.44376,0.6590425649379551,36.45624,0.26218043100124844
520000.0,35.28752,1.9444043976498302,35.10938,1.0274376504683882
560000.0,33.445319999999995,1.3497894078707249,35.23906,0.7658466676822465
600000.0,33.781240000000004,1.6427902381010189,35.92342,0.7218701140787014
640000.0,33.89532,1.615513997958545,35.05782,1.5587492709220445
680000.0,32.415620000000004,0.4217670892803284,34.5047,0.7412388548909198
720000.0,31.34686,1.8061872655956814,34.1625,1.2451807884801314
760000.0,31.545299999999997,1.5516874259979028,34.09846,1.434144788506376
800000.0,30.735940000000006,1.4260263414116865,34.21718,1.5387858244733088
840000.0,30.507800000000003,2.73113493039066,32.07968,1.3645340412023446
880000.0,30.3375,1.3278418083491723,32.54218,0.4606297576145061
920000.0,29.565620000000003,2.28296337894413,31.93282,0.894980246485921
960000.0,28.448439999999998,1.9652020075300143,30.931220000000003,0.853665886398187
1000000.0,28.078120000000002,1.2332626636690174,31.190639999999995,1.1731644839492885
1040000.0,27.296879999999998,1.1616103622127352,30.81096,1.0861661504576545
1080000.0,25.868759999999998,0.9810898401267856,31.11404,1.4915461569793946
1120000.0,25.13124,1.689044978205139,30.837500000000006,1.1889867669574807
1160000.0,25.0203,1.7342750312450443,29.66406,1.3780526486313942
1200000.0,23.87656,1.2033925387835838,28.9375,0.66942874452775
1240000.0,23.94842,1.1665470353140504,29.27812,0.6372065266457974
1280000.0,23.3594,0.8601975610288604,29.879700000000003,0.5755541156138143
1320000.0,22.87342,0.7914075824756792,29.020339999999997,0.6055687990641526
1360000.0,21.776559999999996,0.8260173281475404,29.54842,0.773312866568248
1400000.0,21.829659999999997,0.4265210925616693,27.60626,0.9756211285124983
1440000.0,21.8672,0.8807864031648083,28.082819999999998,1.7231819699613853
1480000.0,21.69844,1.0958030308408542,27.67502,1.3895923667032717
1520000.0,21.14532,0.41633896478710647,27.56094,1.2406928541746345
1560000.0,20.97656,0.7244183283158979,27.518759999999997,2.074775716649874
1600000.0,21.087500000000002,0.868156969677719,27.107840000000003,1.1133705198180883
1640000.0,20.63282,0.5788149942771011,27.696859999999997,1.4742552704331777
1680000.0,20.73594,0.3434005509605362,27.66096,1.848329805635347
1720000.0,20.30312,0.4122994125632482,27.6031,2.800051610952913
1760000.0,20.25626,0.5159866843243146,26.506259999999997,2.241899670904119
1800000.0,20.017180000000003,0.7856509693241656,26.41562,1.6351976827283
1840000.0,20.2594,0.6882086689369726,26.514059999999994,1.8828705729284743
1880000.0,19.803140000000003,0.3587140231437853,26.531240000000004,2.099029081837601
1920000.0,19.91562,0.5100360512748092,26.665640000000003,1.9404588576932007
1960000.0,19.96252,0.3686757024811914,26.070299999999996,2.9178365608786248
2000000.0,19.53282,0.3181152646447506,26.31248,2.5677284921891568
2040000.0,20.32814,0.9394887452226351,25.51408,1.9806442430684013
2080000.0,19.671879999999998,0.642977540509776,25.3875,2.550890469620364
2120000.0,19.29218,0.29580057065529713,25.81094,2.5404790104230344
2160000.0,19.51248,0.6593391521819404,25.35314,2.1278274371762382
2200000.0,19.203120000000002,0.3074192602944715,24.3453,2.065951884241257
2240000.0,19.7,0.5155495786052009,24.55312,2.3529930874526594
2280000.0,19.29844,0.5407850131059472,24.31716,2.9223971164781832
2320000.0,19.94062,0.5596436094515859,25.31566,2.626033207406182
2360000.0,19.1125,0.8148507347974837,24.44688,2.248256760603646
2400000.0,19.595299999999998,0.3587086227009326,23.882799999999996,2.2946914537688947
2440000.0,19.445320000000002,0.412946681304015,23.64064,2.1559713046327866
2480000.0,19.05936,1.193079599356221,24.390639999999998,2.7355368603621484
2520000.0,19.09062,0.7209798621320852,23.64846,2.0461531649414715
2560000.0,18.90312,0.4594472827213148,23.484379999999998,2.4377561645086656
2600000.0,18.781240000000004,0.4828187096623331,23.6125,2.459696087731165
2640000.0,18.41252,0.2524219118856371,23.64844,2.56867578226603
2680000.0,18.86562,0.5918415916442513,23.510939999999998,2.3802386683692034
2720000.0,18.91718,0.6558594739728933,23.740640000000003,2.8135679445145807
2760000.0,19.085919999999998,0.8748736077857194,24.00156,3.3167075177651704
2800000.0,19.10782,0.4373684003217424,23.2297,2.578981001868762
2840000.0,19.30468,0.4282154569839811,23.040660000000003,3.0533363297219647
2880000.0,19.010920000000002,0.9427185760342268,23.28282,2.3564796756178477
2920000.0,18.782819999999997,0.6782473513402024,22.7328,2.2301167816955236
2960000.0,18.70626,0.9991030629519659,23.69218,3.364004849223616
3000000.0,18.639059999999997,0.5343507971361136,23.240640000000003,3.823204604307753
3040000.0,19.02812,0.5653059256721097,23.393760000000004,3.1647022928547317
3080000.0,18.81094,0.3799433147194465,22.36408,2.613155747673682
3120000.0,18.59374,0.25454325054889965,23.571859999999997,2.7904768593199267
3160000.0,18.63124,0.6609563150466146,22.38438,3.016018449810942
3200000.0,18.79376,0.20577640875474487,22.787520000000004,3.2954159484957284
3240000.0,18.4375,0.4236990205322634,22.72344,3.4447381036009106
3280000.0,18.6422,0.4839833096295781,22.959380000000003,2.8039052526075117
3320000.0,18.49844,0.6387068610873072,22.59842,2.9004524311906925
3360000.0,18.20626,0.7658374046754315,22.690640000000002,4.033215667231298
3400000.0,18.529680000000003,0.5998981493553712,22.21874,3.4020553990786224
3440000.0,18.545319999999997,0.6873030289472035,22.45468,3.104293420667576
3480000.0,18.53124,0.316271184903083,22.114060000000002,3.570112894349421
3520000.0,18.54218,0.2974981102460983,22.751560000000005,3.715978485190677
3560000.0,18.33594,0.4484561632980424,22.450000000000003,3.2513062039740275
3600000.0,18.91094,0.8151657809304805,22.489060000000002,3.163892422064947
3640000.0,18.910919999999997,0.49060044598430524,22.10312,3.259150012748722
3680000.0,18.1578,0.5815671517546362,22.598419999999997,4.436992049305475
3720000.0,18.05626,0.5969382232693765,22.12344,3.6965445500358842
3760000.0,18.935899999999997,0.6149039242027976,22.517200000000003,3.330350159968167
3800000.0,18.40624,0.6677361772436775,22.132820000000002,2.982218357129472
3840000.0,18.828120000000002,0.47808696865737815,22.4875,4.256578268515687
3880000.0,18.71562,0.5457531068166266,22.298440000000003,3.7711884877847197
3920000.0,19.10938,0.4692482771412167,22.457820000000005,2.8593988811636613
3960000.0,18.29842,0.5875508093773676,22.20938,4.3278801348466205
4000000.0,18.732799999999997,0.6035989396942308,22.17344,2.680349898502059
}\dataAblationArchA

%% file: data/ablation_arch_b.tex
\pgfplotstableread[col sep=comma]{
0,0,0,0,0
}\dataAblationArchB

%% file: data/ablation_noise.tex
\pgfplotstableread[col sep=comma]{

8000.0,40.670925,0.06113053144706004,40.76526,0.09613296208897339
1608000.0,39.511250000000004,0.053749999999998875,40.13974,0.5992720971311779
3208000.0,39.436575000000005,0.2489985680179698,39.97098,0.9751651991329475
4808000.0,39.465325,0.12289046698177863,40.05748000000001,1.1719897190675346
6408000.0,39.586575,0.13416980612268772,39.46674,1.9716082040811238
8008000.0,39.4156,0.20196283073872817,38.93052,1.5756114386485018
9608000.0,39.816275,0.3955150526528689,38.75726,1.8439991611711757
11208000.0,39.41437500000001,0.19367954686801625,37.27274,2.4041828770707085
12808000.0,39.7506,0.5766335101257997,36.43876,2.8797273875143117
14408000.0,39.868425,0.4386707955574428,35.30924,3.4613916869374957
16008000.0,39.286225,0.9506295292462799,35.6665,4.462902267807352
17608000.0,39.41035,0.6042125722789932,34.522980000000004,4.116422927931482
19208000.0,38.949400000000004,0.9437037167458864,33.34778,4.93210798600355
20808000.0,38.973150000000004,0.6718050665929821,32.4865,3.8632661896379847
22408000.0,38.888725,0.6031570084770619,30.942500000000003,4.36004634746008
24008000.0,38.8447,1.060909524888905,29.66452,2.8513294347724902
25608000.0,38.935649999999995,2.1956204743762067,28.86676,1.9579188099612301
27208000.0,38.68905,1.042302764315629,27.427000000000003,2.022960330802362
28808000.0,38.106575,1.3049614714140036,28.338480000000004,2.1929876747487653
30408000.0,37.275025,2.2055741104925497,26.104000000000003,1.0438759868873313
32008000.0,36.5569,3.4169747562719857,25.844279999999998,1.0892504971768426
33608000.0,37.368424999999995,2.591981609092743,25.53824,1.4627601534086165
35208000.0,36.308425,3.841166481926421,25.4105,1.4068615610642008
36808000.0,36.436875,3.858643594823835,24.90128,0.8427855751019954
38408000.0,36.145325,3.4030641989058927,25.077499999999997,1.6311123125033415
40008000.0,35.452825,4.181292896566205,24.754260000000002,0.9477385010645082
41608000.0,34.551525,4.959640001237487,24.156259999999996,1.1770100366606906
43208000.0,36.345,4.529298717792853,24.252000000000002,1.2746371452299665
44808000.0,34.393125,4.993965785513053,23.8195,0.8550015649108473
46408000.0,34.032174999999995,5.561226745231217,23.954,1.1787187433819835
48008000.0,33.34155,6.151527926661797,23.471,0.6295879414347136
49608000.0,33.005,6.436668543586815,23.27524,0.503658148350644
51208000.0,33.2128,6.400097137153466,23.33474,0.9523827331488117
52808000.0,33.19685,6.459780350174454,23.078480000000003,1.1646241735426923
54408000.0,32.5428,7.250214970261227,23.27202,0.6682592472985318
56008000.0,32.3947,7.199642770304651,22.92226,1.0201139183444168
57608000.0,33.114025,6.496995378009362,22.735000000000003,0.7673428177809444
59208000.0,31.869675,7.469286932959197,22.95176,0.6917685786446216
60808000.0,31.9212,7.541159766574369,22.62378,0.9667098456103562
62408000.0,31.745599999999996,7.694442490200312,22.436,0.8169009585010902
64008000.0,31.632524999999998,7.685700889110568,22.675500000000003,0.7172350047229984
65608000.0,31.819999999999997,7.473169681935503,22.567500000000003,0.8155965767461266
67208000.0,31.5744,8.072906949172648,22.50726,0.8376147720760427
68808000.0,31.53595,8.09875912887277,22.63824,0.9694818546006938
70408000.0,31.93655,7.730257213217423,22.72674,1.4895200859337212
72008000.0,31.359075,8.005210789659133,22.43324,0.9879958837970929
73608000.0,31.461550000000003,8.475084945444499,22.08,0.8058916651759097
75208000.0,31.2428,8.231580960848285,22.14976,0.7418645633806747
76808000.0,31.257175000000004,8.361562729949169,22.022,0.9481776078351571
78408000.0,30.98155,8.551098384564405,21.775,0.774366112378376
80008000.0,31.788724999999996,7.957415873063,22.03648,1.1904049149764127
81608000.0,31.328749999999996,8.32294401233722,22.06148,0.6659908870247391
83208000.0,31.12595,8.39834623914137,21.88148,0.7560997034783173
84808000.0,30.873425,8.519225423527365,22.248,0.883614990818966
86408000.0,31.2072,8.375482144032068,21.79948,0.6263712011259774
88008000.0,30.9553,8.450324534596291,22.15774,0.7880546341466432
89608000.0,31.1625,8.392418600737216,21.797739999999997,0.6903648357209401
91208000.0,30.981875,8.554777598913661,21.65326,0.7895317513564599
92808000.0,30.839350000000003,8.504042637034459,22.003500000000003,0.806147242133843
94408000.0,30.865000000000002,8.242603356646493,21.90826,0.783616798186461
96008000.0,31.010650000000002,8.304187914691,21.81326,0.5924175979830443
97608000.0,30.6644,8.644698576873575,21.788220000000003,0.6402562155887281
99208000.0,30.927199999999996,8.823334383610314,21.771759999999997,0.7597580314810762
100808000.0,31.018725000000003,8.538226332610012,21.775760000000002,0.8786665957005529
102408000.0,30.932499999999997,8.443373798133067,21.834780000000002,0.7831094212177507
104008000.0,30.84815,8.719886878423367,21.864259999999998,0.5468218781285185
105608000.0,30.7641,9.02451739125146,21.73752,0.739654598850031
107208000.0,30.814725,8.555774399017016,21.40102,0.7441399637165046
108808000.0,30.805,8.816025101767803,21.51376,0.4304331845943112
110408000.0,30.575650000000003,9.16595718119499,21.58102,0.7077918914483262
112008000.0,30.512175,9.167229306713944,21.55248,0.8264774053777879
113608000.0,30.62,8.85631563320775,21.28174,0.39871168831625653
115208000.0,30.48875,8.938812928040278,21.63074,1.084729397776238
116808000.0,30.437825,9.078494157175793,21.44198,0.855997419154987
118408000.0,30.251575,9.203544049271185,21.19726,0.749373825537028
120008000.0,30.366549999999997,9.082757804351054,21.282,0.7939666718446061
121608000.0,30.31155,9.337536175699672,21.387539999999998,0.822920075341464
123208000.0,30.424725,9.127621243614076,21.12038,0.7190771012902579
124808000.0,30.083099999999998,9.187074329458753,21.25976,0.7946416213614793
126408000.0,30.375024999999997,9.11181050378436,21.271,0.8747642402384767
128008000.0,30.415000000000003,8.968480492536068,21.28588,0.8206520490439291
129608000.0,30.33185,9.077572606291838,21.27478,0.39309002225953227
131208000.0,30.584675000000004,9.114346125031405,21.33348,0.7970972598121254
132808000.0,30.248425,9.28142477138478,21.228019999999997,0.5672750704904986
134408000.0,30.281899999999997,9.306851935805147,21.28676,0.738524468382734
136008000.0,30.626575,9.168704349136524,21.326019999999996,0.7208566165334129
137608000.0,30.356275,9.296519209191954,21.3115,0.7021476169581433
139208000.0,30.529075000000006,9.11692394351708,21.308999999999997,0.7252892360982619
140808000.0,30.2528,9.014045582866773,21.279239999999998,0.6300566723716214
142408000.0,30.284099999999995,9.165808455613721,21.355240000000002,0.5985290355529962
144008000.0,30.328125,9.207756880308851,21.35874,0.5598139426630961
145608000.0,30.212175000000002,8.976264502111944,21.23726,0.5955500670808468
147208000.0,30.383775,9.155894250802321,21.260119999999997,0.6944238141077831
148808000.0,30.410925,9.110133418994202,20.87886,0.7157674275908348
150408000.0,30.244699999999998,9.202409311968252,21.178620000000002,0.7140018862720184
152008000.0,30.259050000000002,9.315644164924937,21.136740000000003,0.7544094951682403
153608000.0,30.302825,9.215442247492792,21.196,0.7948035455381409
155208000.0,30.34,9.12376842209402,21.137999999999998,0.7114808866020222
156808000.0,30.118425000000002,9.294556616206876,21.24874,0.829276568100173
158408000.0,30.305325000000003,9.269296133842905,21.121019999999998,0.8480616496458264
}\dataAblationNoise

%% file: data/ablation_rss.tex
\pgfplotstableread[col sep=comma]{
8000.0,40.7665,0.09527066704920384,40.76974,0.08713680278733911
1592000.0,39.40652,0.25284514153924353,39.974000000000004,0.5895398001831605
3176000.0,39.37274,0.28121612044831473,40.15926,0.6813516466553844
4760000.0,39.490759999999995,0.25577643832065433,39.50224000000001,1.4548415537095436
6344000.0,39.756780000000006,0.4703005906864231,39.36326,1.3973704542461176
7928000.0,39.66302,0.6653329855042522,38.54174,2.0242786148156577
9512000.0,39.88374,0.31073513866313773,37.794,2.4991335050372943
11096000.0,39.7085,0.6029438215953463,37.30574,3.400582121696225
12680000.0,39.40574,0.19907006404781205,36.114,3.3669647684524393
14264000.0,39.407740000000004,0.44286951396545876,35.55574,2.018757970238135
15848000.0,39.34674,0.8856318142433663,35.6295,4.004442917060998
17432000.0,39.61376,0.43509220448084324,35.46476,3.883346933561306
19016000.0,39.393980000000006,0.49980252260267793,32.8805,3.4538768038249423
20600000.0,38.768739999999994,1.3296234994914922,31.52326,3.460967723975478
22184000.0,38.56924,1.7829382452569704,29.831259999999997,2.901694406101373
23768000.0,38.03654,2.0462646041995654,29.21148,2.4634512095026353
25352000.0,38.084720000000004,1.717225985594209,28.21524,2.0275188794188823
26936000.0,38.220240000000004,1.53017695264306,28.879,3.387671066086551
28520000.0,38.09776,1.8445040922697886,27.761020000000002,1.9507520758927837
30104000.0,38.141479999999994,1.7393558536423772,26.708,1.987002837441356
31688000.0,38.10752,1.726494606304924,26.509000000000004,1.4680557128392633
33272000.0,37.69802,2.33846280825674,26.05402,1.482748213824586
34856000.0,38.61376,2.1495646420612706,25.823,1.7227868922185352
36440000.0,37.80048,2.4608531336916477,25.38598,1.1292553039946283
38024000.0,37.55998,2.79475464497333,25.316,1.4878067750887547
39608000.0,37.593239999999994,3.0754472628221072,24.322499999999998,0.9509369463849844
41192000.0,38.08374,3.871324696069808,24.649520000000003,1.1700002758974035
42776000.0,37.2105,3.779412054275108,24.50926,1.1448950565008131
44360000.0,37.23026,3.639673923636567,24.544239999999995,1.213989780187626
45944000.0,36.662279999999996,3.8723608485780363,24.191,1.0956193116224264
47528000.0,36.63376,3.723867579600542,24.855240000000002,0.8837131923876655
49112000.0,36.186,4.141262066085653,23.496000000000002,1.2422782602943674
50696000.0,35.985240000000005,4.011681618274312,23.532980000000002,1.1643137230145497
52280000.0,35.963,4.068561152545208,23.21926,1.0069772045086223
53864000.0,36.132259999999995,4.717816457896598,24.30274,2.838141769961465
55448000.0,35.00926,4.579752991854474,23.61022,1.0638539211752713
57032000.0,35.9275,4.64464800345516,23.23702,1.1304576708572511
58616000.0,35.155240000000006,4.2758050856417675,23.090980000000002,1.0413147293685996
60200000.0,34.53148,4.806063241947613,22.999519999999997,1.1928352868690633
61784000.0,36.370259999999995,4.753249051585661,22.83396,1.0534324232716594
63368000.0,34.7205,4.681004654558676,22.188499999999998,0.9404739634886228
64952000.0,34.35224,5.2244274467543335,22.446,0.6242693136779989
66536000.0,34.4615,5.130386140633081,22.067479999999996,0.7950104411892968
68120000.0,33.416740000000004,5.693094816213762,21.82,0.5248951095218927
69704000.0,32.210480000000004,6.3318452085944115,21.711,0.6815645794787161
71288000.0,31.924,6.520144285213327,21.92652,0.4906469867430145
72872000.0,31.800739999999998,6.720880445774943,21.9495,0.5793931963701328
74456000.0,31.23448,7.207418088719428,21.753,0.6472369550636002
76040000.0,30.780520000000003,6.5184622424004255,21.66624,0.4819875790930721
77624000.0,30.704040000000003,6.6162896573835095,21.64748,0.621668662230935
79208000.0,30.69198,6.175118612755547,21.54222,0.7530748207183667
80792000.0,30.596719999999998,5.902830347011508,21.64902,0.45943752741803845
82376000.0,30.9005,6.899230245179529,21.34698,0.6362484432986844
83960000.0,30.61074,7.274318441641115,21.3865,0.5527279403106014
85544000.0,30.84074,6.848504865326445,21.33452,0.6200279200165095
87128000.0,31.28752,6.904673730568302,21.155,0.5952415476090366
88712000.0,29.78426,6.3560382603002,21.37972,0.4998762462850184
90296000.0,30.05378,6.496726177514333,21.08026,0.34319272486461566
91880000.0,30.64626,7.82500232756515,21.42048,0.6194594592061692
93464000.0,29.57448,6.5351519802985445,21.416019999999996,0.6704537162250646
95048000.0,29.28548,6.511280184848445,21.29878,0.6275610931216177
96632000.0,29.15874,6.72979699919693,21.370759999999997,0.5143154541718536
98216000.0,28.745240000000003,6.646952813613167,21.316,0.6363848489711241
99800000.0,29.275239999999997,7.30196821455695,21.38352,0.6211523046725331
101384000.0,28.746979999999997,6.7717337574361265,21.369500000000002,0.5964560972946787
102968000.0,28.806239999999995,6.848448170965448,21.2595,0.7287034046853352
104552000.0,28.710240000000006,6.719778858444673,21.301840000000002,0.6307738504408696
106136000.0,28.39624,6.504642497970199,21.40002,0.41154878398556827
107720000.0,28.351260000000003,6.390377314869601,21.32226,0.5390544558762131
109304000.0,27.86974,6.378187536157901,21.348779999999998,0.719370854010642
110888000.0,27.67826,6.639360736275746,21.337020000000003,0.5461895692889057
112472000.0,27.674500000000002,6.53645673710153,21.1875,0.6686583402605548
114056000.0,27.633259999999996,6.217316495916868,21.08498,0.6814229417916595
115640000.0,27.42526,6.2900172294199646,21.3695,0.5327673038015758
117224000.0,27.15724,6.506886787888658,21.393,0.7464507190699191
118808000.0,26.8705,6.4588656338400465,20.997120000000002,0.5040018388855333
120392000.0,26.862479999999998,6.672075794053901,20.987740000000002,0.634620442784504
121976000.0,26.726980000000005,6.540182014409081,20.92752,0.43130580983798517
123560000.0,26.87038,6.519604287500892,21.059019999999997,0.5767173845134204
125144000.0,26.743239999999997,6.3828440716031905,21.30574,0.39433560123326505
126728000.0,26.58876,6.534262791348386,21.036,0.652710560662228
128312000.0,27.06176,6.581377241155532,21.06322,0.36402904499503835
129896000.0,26.96126,6.4776789613564505,20.988500000000002,0.4309853547395784
131480000.0,26.55898,6.646135844353468,21.22826,0.43244694749761037
133064000.0,26.05036,6.920988694283499,21.075480000000002,0.6116793454090139
134648000.0,26.180380000000003,6.94135094758938,20.947499999999998,0.5330657107711956
136232000.0,25.910379999999996,6.829906562581951,21.194499999999998,0.4419372127350219
137816000.0,26.16126,7.16461518508845,20.87774,0.5406385173107814
139400000.0,25.83088,7.001157031919797,21.03076,0.5205810084895537
140984000.0,26.019479999999998,7.0155147410293415,21.06636,0.4484844594854989
142568000.0,25.93398,6.921849172121564,21.058,0.5951126582421188
144152000.0,25.602359999999997,7.005471239709716,20.75514,0.4519351617212368
145736000.0,25.58648,7.135614401409314,20.79352,0.558420796174354
147320000.0,25.730740000000004,7.244877175660053,20.686020000000003,0.6341446440678971
148904000.0,25.66476,7.089041644002381,20.94236,0.5743952388382061
150488000.0,25.64024,7.1865064187266965,20.86286,0.508150991733756
152072000.0,25.759500000000003,7.055856932506496,20.81622,0.5482073746311699
153656000.0,25.616859999999996,6.99327066017039,20.52888,0.5085606095638943
155240000.0,25.36988,7.1227367867695355,20.675900000000002,0.2959898444203778
156824000.0,25.299,6.9829714789622335,20.77178,0.3406834976925056
158408000.0,25.871019999999998,6.751558367191978,20.532519999999998,0.15822042093231806
}\dataAblationRss

%% file: data/ablation_when.tex
\pgfplotstableread[col sep=comma]{

0.0,40.0,0.0,1.0,0.0,40.0,0.0,0.0,0.0,40.0,0.0,0.83026,0.3393800088396487
40000.0,39.923825,0.13193897026655788,1.0,0.0,39.72136666666666,0.22966947748642808,0.0,0.0,40.0,0.0,0.40408,0.11682471313895874
80000.0,39.6289,0.3644928051416103,1.0,0.0,39.95053333333333,0.06995643088539019,0.0,0.0,39.676579999999994,0.175523416101669,0.48768,0.16649699576869248
120000.0,39.552749999999996,0.5422518257230686,1.0,0.0,39.57293333333333,0.46291218989734606,0.0,0.0,39.40468,0.48468282164731236,0.60634,0.016512492240724888
160000.0,39.005874999999996,0.6447425081961006,1.0,0.0,38.98696666666667,0.4200473015691875,0.0,0.0,39.39378,0.7518042415416412,0.52684,0.13187482853069424
200000.0,39.08985,0.7152378013639925,1.0,0.0,38.375033333333334,1.1324150309061696,0.0,0.0,37.86406,1.2450487662738359,0.60476,0.13920061206761988
240000.0,38.209,0.8356188245845112,1.0,0.0,38.080733333333335,0.21206634706043426,0.0,0.0,38.51404,0.7030079959715952,0.5876,0.15300709787457575
280000.0,37.794925,0.5731010442103563,1.0,0.0,39.0026,0.4272813826976308,0.0,0.0,38.38438,1.01504800162357,0.51512,0.15857225986912085
320000.0,37.480475,1.5833983601339885,1.0,0.0,38.27863333333334,0.447281017805236,0.0,0.0,36.493739999999995,1.728187180371386,0.5253599999999999,0.16815592288111653
360000.0,37.322275,1.5442587985422005,1.0,0.0,38.08076666666667,0.31180478223116176,0.0,0.0,36.26094,1.493811539117301,0.5720000000000001,0.14895591294070876
400000.0,36.12695000000001,2.001828176317836,1.0,0.0,38.140633333333334,0.21001302076035383,0.0,0.0,35.62656,1.4824707816344995,0.55018,0.07913258241710555
440000.0,36.25585,1.7603226472723692,1.0,0.0,37.2578,0.47664251453963435,0.0,0.0,34.981260000000006,0.6348252912416156,0.5608199999999999,0.10885203535074571
480000.0,35.513675,1.4098407044325967,1.0,0.0,37.5807,0.33903121783497486,0.0,0.0,35.04376,1.2182393420013993,0.55696,0.12484980736869401
520000.0,35.695325,3.0543664395214605,1.0,0.0,37.916666666666664,0.7608654078671786,0.0,0.0,35.4625,1.2633156248539008,0.58714,0.09041905993760385
560000.0,34.537125,2.3560783532970637,1.0,0.0,36.91143333333333,0.3963811998008427,0.0,0.0,33.818740000000005,1.3857626399928675,0.62532,0.04051317810293336
600000.0,33.982425,3.1200279793416916,1.0,0.0,36.802099999999996,0.22100381595498977,0.0,0.0,33.6172,0.5336223870865976,0.55454,0.1149687018279323
640000.0,34.541,3.567308786045862,1.0,0.0,36.830733333333335,0.9653237464993579,0.0,0.0,34.28438,1.6130342406781069,0.5748599999999999,0.166910737821148
680000.0,34.46095,2.958975454866091,1.0,0.0,37.393233333333335,0.666623713616276,0.0,0.0,32.63438,1.2771804060507665,0.6257999999999999,0.10338458298992166
720000.0,34.072275,3.499279331787476,1.0,0.0,37.080733333333335,0.4316997207421959,0.0,0.0,31.584339999999997,2.1728534774346855,0.53026,0.14774887613785764
760000.0,33.04295,3.0217285587060916,1.0,0.0,36.60413333333333,1.1274108429888765,0.0,0.0,31.246879999999997,0.9983059839548197,0.60828,0.12339905023945685
800000.0,31.703125000000004,3.5451892815299733,1.0,0.0,36.60156666666666,1.1883622969822354,0.0,0.0,31.99842,0.5914790322572734,0.59202,0.10749994232556594
840000.0,31.52345,3.08958234920191,1.0,0.0,36.0599,1.775417455135551,0.0,0.0,29.66406,1.296725021891688,0.58002,0.13804583876379614
880000.0,31.248075,2.4573858461533873,1.0,0.0,36.2474,0.9066790538369507,0.0,0.0,31.07654,0.91202837148852,0.6104,0.13645954711928368
920000.0,29.773425000000003,3.0877125986521166,1.0,0.0,35.1901,0.7672080856369173,0.0,0.0,30.35,1.441572993642708,0.59512,0.12054077152565434
960000.0,29.425775,3.4427807273881106,1.0,0.0,35.21353333333334,0.40150717165312305,0.0,0.0,28.13906,1.248387456841825,0.6060000000000001,0.14068914670293514
1000000.0,28.24415,4.533859602204284,1.0,0.0,35.51303333333333,1.9077924316398311,0.0,0.0,29.34686,0.8813294380650176,0.6051200000000001,0.13175708557796809
1040000.0,28.341799999999996,3.5158832588696676,1.0,0.0,35.1328,1.1842630225869053,0.0,0.0,27.89064,1.7342202185420392,0.6025,0.1414705340344766
1080000.0,28.334000000000003,3.6829558244703393,1.0,0.0,34.55206666666667,0.9444554139232254,0.0,0.0,27.095319999999997,0.9668569085443821,0.6044999999999999,0.1158973856478221
1120000.0,27.9043,3.656403553356767,1.0,0.0,34.645833333333336,0.8034853009787359,0.0,0.0,26.021859999999997,1.0134242657446089,0.6062399999999999,0.15076268238526402
1160000.0,26.2539,3.227680220374999,1.0,0.0,33.393233333333335,0.44436525766785834,0.0,0.0,26.48594,1.1315795413491703,0.59628,0.13911433283454294
1200000.0,26.740225,3.6129679941669837,1.0,0.0,33.890633333333334,0.6308626598203075,0.0,0.0,25.70624,1.160170223027638,0.63832,0.11792813743971368
1240000.0,27.164099999999998,4.845275470497007,1.0,0.0,33.60156666666666,1.1638656575777493,0.0,0.0,24.80624,1.1980525507672866,0.5997199999999999,0.11871246606822722
1280000.0,25.951175,4.167155158123465,1.0,0.0,33.1745,0.305431268864209,0.0,0.0,25.0625,1.8954536480747826,0.62296,0.11250442835728731
1320000.0,25.68945,4.7001136467643,1.0,0.0,32.95053333333333,1.4769764821719036,0.0,0.0,24.234379999999998,1.0161886132013087,0.63964,0.1050384805678376
1360000.0,25.755875000000003,4.215259705744712,1.0,0.0,32.9427,0.22212217358921874,0.0,0.0,24.3625,0.6291870691614699,0.6359,0.12044253401518917
1400000.0,24.650375,4.4780468450960855,1.0,0.0,32.015633333333334,0.780979412162845,0.0,0.0,23.118740000000003,0.7444272337844716,0.63294,0.11174597263436384
1440000.0,24.2637,3.594743277342626,1.0,0.0,31.807333333333332,1.8031426017434724,0.0,0.0,22.91874,1.2033304277711918,0.6281399999999999,0.10298669040220682
1480000.0,24.265625,4.359374350394217,1.0,0.0,31.666666666666668,1.2804392848636836,0.0,0.0,22.96876,1.0235134422175407,0.6377599999999999,0.11726942653564908
1520000.0,23.1543,3.746503628051092,1.0,0.0,31.877633333333335,0.11417417493558639,0.0,0.0,22.85624,0.2658325232171563,0.64828,0.12389116836966226
1560000.0,23.466825,3.479229318667426,1.0,0.0,31.86196666666667,1.8263516388934777,0.0,0.0,21.895319999999998,0.9060112877884021,0.64324,0.10221214409256858
1600000.0,23.9824,3.8791811971084824,1.0,0.0,29.763033333333336,0.38083501531357244,0.0,0.0,21.775,0.6363373287808913,0.64796,0.12567546459034876
1640000.0,23.1504,2.308075017195065,1.0,0.0,30.362,1.3835580893719877,0.0,0.0,21.83592,1.0956938977652477,0.6042799999999999,0.10163879967807572
1680000.0,22.24415,2.2857120919967158,1.0,0.0,29.4375,1.4270884508910666,0.0,0.0,22.13126,1.0530865502891957,0.6311,0.1151733476113289
1720000.0,22.955075,3.7306216274067516,1.0,0.0,30.3203,1.0405229870919086,0.0,0.0,21.473419999999997,0.9637699256565335,0.64118,0.08477717617377922
1760000.0,22.289025,2.7858625113014814,1.0,0.0,29.223966666666666,0.7590399213626536,0.0,0.0,21.1547,0.8189381710483402,0.6457200000000001,0.12431584613395028
1800000.0,22.125,3.1247512949033234,1.0,0.0,29.3776,0.9589163084788305,0.0,0.0,21.0125,0.763851236825601,0.65196,0.12015324548259194
1840000.0,22.214824999999998,3.306399944331447,1.0,0.0,28.843733333333333,1.3346747077679846,0.0,0.0,20.72498,0.7102854676818323,0.67884,0.11064152204303772
1880000.0,21.720699999999997,2.824769803895532,1.0,0.0,28.97136666666667,0.7855000374849692,0.0,0.0,21.329700000000003,1.161752918653532,0.6549600000000001,0.10817098686801374
1920000.0,21.625025,2.7002973469370004,1.0,0.0,28.42186666666667,1.8884391461969032,0.0,0.0,21.318760000000005,0.44221856406080523,0.63558,0.0968481987442203
1960000.0,21.742199999999997,2.0110233203521037,1.0,0.0,28.17446666666667,1.9058636280932826,0.0,0.0,20.635939999999998,0.42074230402943824,0.6481,0.0982220341878542
2000000.0,21.2969,2.8461090061696517,1.0,0.0,27.65886666666667,2.4486036270132048,0.0,0.0,20.9625,0.9410568569432988,0.6540000000000001,0.0970197917952827
2040000.0,20.578125,2.4540059029829164,1.0,0.0,27.992199999999997,1.184184929251622,0.0,0.0,21.13906,0.5512060580218627,0.6461,0.09159181186110471
2080000.0,20.302725,1.8355772870884526,1.0,0.0,27.880200000000002,2.3022206033885344,0.0,0.0,20.72656,0.7885135346967731,0.66686,0.09033410430175308
2120000.0,20.392575,1.736777905166633,1.0,0.0,26.718733333333333,2.612932193702869,0.0,0.0,20.7172,0.3743053993732928,0.6382199999999999,0.12368298832094898
2160000.0,20.623025,1.861096781974275,1.0,0.0,26.562466666666666,1.9749742383692561,0.0,0.0,20.00626,0.7571136548762014,0.664,0.11331739495770277
2200000.0,20.527350000000002,0.9359868388497782,1.0,0.0,26.406266666666667,2.192964424902714,0.0,0.0,19.73594,0.27197353253579715,0.6607800000000001,0.085031697619182
2240000.0,20.61915,1.4549222358944132,1.0,0.0,25.2474,0.9441391881850194,0.0,0.0,20.84686,1.8627629785885262,0.65176,0.09869114651274448
2280000.0,20.304675,2.1405652937658792,1.0,0.0,26.7422,2.1240175579939704,0.0,0.0,20.521859999999997,1.134369986556414,0.6300399999999999,0.113639334739341
2320000.0,20.41015,1.3111909862792674,1.0,0.0,25.059933333333333,1.068735755720541,0.0,0.0,20.16718,0.41053800506165067,0.6548,0.10213569405452726
2360000.0,20.4707,1.3268707981563241,1.0,0.0,26.309866666666665,1.9325469452397666,0.0,0.0,20.0172,0.4727464648202033,0.63714,0.09751571360555178
2400000.0,20.351575,1.7652797092458177,1.0,0.0,25.2396,1.8386919970457267,0.0,0.0,20.071859999999997,0.7283598660003168,0.62586,0.08554085807378838
2440000.0,20.011725,1.612465229663884,1.0,0.0,25.6354,1.504782624390203,0.0,0.0,20.15626,0.5493564074442022,0.63462,0.08798105250563894
2480000.0,20.00195,1.4107071143579033,1.0,0.0,25.773433333333333,2.75068768654111,0.0,0.0,19.851580000000002,0.4104066685618059,0.63706,0.0889039616665084
2520000.0,20.527325,1.04634481738813,1.0,0.0,24.927066666666665,1.6208466354209938,0.0,0.0,19.75624,0.7341053292273529,0.6364,0.10936260786941757
2560000.0,19.52735,0.7955166890141272,1.0,0.0,25.4245,0.917242054567204,0.0,0.0,20.06406,1.1362443656185934,0.6346,0.059411413044969705
2600000.0,19.64065,1.01673184886675,1.0,0.0,23.916666666666668,1.279853745116561,0.0,0.0,20.0344,0.803107830867064,0.61816,0.06323421858456069
2640000.0,19.556625,0.9802247736488812,1.0,0.0,23.60676666666667,1.0590335919548952,0.0,0.0,19.665640000000003,0.6832510127325097,0.6279399999999999,0.07150277197423888
2680000.0,19.3125,1.0727959102271045,1.0,0.0,24.552066666666665,1.985416114795306,0.0,0.0,19.80312,0.7093772745161775,0.62428,0.055857439969980716
2720000.0,19.882825,1.1632867173981658,1.0,0.0,24.632800000000003,0.7631479061536279,0.0,0.0,20.13124,0.6540832748205692,0.6615400000000001,0.08586208942251519
2760000.0,20.015625,0.9905697688073263,1.0,0.0,23.34113333333333,0.47324032184739045,0.0,0.0,19.784380000000002,0.6505676671953498,0.6213200000000001,0.08871395380660246
2800000.0,19.75,1.536078754491448,1.0,0.0,24.450533333333336,1.8549857741293387,0.0,0.0,20.1578,0.6119636688562491,0.6618999999999999,0.06948677572027646
2840000.0,20.082025,1.3594888513242769,1.0,0.0,23.846333333333334,1.0918472736096791,0.0,0.0,19.81094,0.8068658069344612,0.62866,0.09257398338626245
2880000.0,19.009775,0.845909783531909,1.0,0.0,24.286466666666666,1.707356116599255,0.0,0.0,19.39844,0.9357448575332922,0.6514800000000001,0.06450427582726588
2920000.0,19.730475000000002,0.8169975133836086,1.0,0.0,23.726566666666667,1.5916394322277343,0.0,0.0,19.6172,0.8196571551569594,0.64144,0.058428369821517355
2960000.0,19.40625,0.7616349798295765,1.0,0.0,23.1016,1.9182511566528508,0.0,0.0,19.86564,0.9729456564474706,0.6301599999999999,0.07711216765206387
3000000.0,20.228525,0.9934323061361542,1.0,0.0,23.82813333333333,1.5186423790858592,0.0,0.0,19.3953,1.0685760094630614,0.63498,0.06691610867347263
3040000.0,19.1836,0.9276884363836815,1.0,0.0,23.763033333333336,1.7321259506424156,0.0,0.0,19.56252,0.39518997659353655,0.6444799999999999,0.06875089526689818
3080000.0,18.99805,0.8589027578835692,1.0,0.0,22.895833333333332,1.4761081720373879,0.0,0.0,19.140620000000002,0.404559174410864,0.6386000000000001,0.06718258702967608
3120000.0,19.36135,0.7498258147730039,1.0,0.0,23.458333333333332,1.008557310661566,0.0,0.0,19.16252,0.4961859062891649,0.63464,0.0670067638376903
3160000.0,19.4375,0.7982543830133337,1.0,0.0,22.822933333333335,1.2186390177388695,0.0,0.0,19.19842,0.5446334708774327,0.6365999999999999,0.035215223980545676
3200000.0,18.814475,1.303094548325255,1.0,0.0,23.843766666666667,0.891066740236418,0.0,0.0,19.7531,0.6695935393953559,0.64258,0.08982479390457848
3240000.0,19.148425,0.7139194680599488,1.0,0.0,23.078100000000003,1.60111185742908,0.0,0.0,19.53594,0.5598677793908132,0.6147599999999999,0.05875173529352133
3280000.0,18.6973,0.6307464585711124,1.0,0.0,22.533866666666665,1.3705162295848805,0.0,0.0,19.62186,0.874277572856585,0.62958,0.06617060979014776
3320000.0,19.1289,0.5416591225115658,1.0,0.0,22.166666666666668,1.0073758693865074,0.0,0.0,19.6125,1.0664698082927624,0.64206,0.08439781039813771
3360000.0,19.066375,0.4624179352869007,1.0,0.0,22.244833333333332,0.10980243875049184,0.0,0.0,19.445320000000002,0.44431236489658854,0.62032,0.06623749391394576
3400000.0,18.912125,1.1607237276264328,1.0,0.0,22.51823333333333,1.0998921290542798,0.0,0.0,18.996879999999997,0.8250147208383611,0.64578,0.062497276740670855
3440000.0,18.847649999999998,0.6908853975153914,1.0,0.0,22.343733333333333,1.246824300194519,0.0,0.0,19.06408,0.6686016374493852,0.67256,0.08635229238416316
3480000.0,18.269550000000002,0.6090331374399918,1.0,0.0,22.3828,1.1588802296469938,0.0,0.0,19.92344,0.4245860035375642,0.65838,0.06769831312521753
3520000.0,19.001949999999997,1.0187827503938212,1.0,0.0,22.208333333333332,0.5049103506784369,0.0,0.0,19.559379999999997,0.6058005361503072,0.62936,0.08219015999497747
3560000.0,18.3965,0.9119337832320943,1.0,0.0,22.541666666666668,1.0017346066809423,0.0,0.0,19.198439999999998,0.6943498991142711,0.64086,0.06941321488016526
3600000.0,19.427725,0.9582413562746074,1.0,0.0,22.8047,0.9407265312866795,0.0,0.0,19.08906,0.9294583317179961,0.6373599999999999,0.051682863697748
3640000.0,19.0332,0.80317399111276,1.0,0.0,22.26823333333333,1.1133143381613104,0.0,0.0,19.14686,0.756566638439735,0.63184,0.04038804773692335
3680000.0,19.06445,0.662677329097654,1.0,0.0,22.539033333333332,0.5889443682461777,0.0,0.0,19.604680000000002,0.7969448799007365,0.64028,0.05881728997497248
3720000.0,18.652324999999998,0.6000383086728709,1.0,0.0,22.611966666666664,1.502924283595891,0.0,0.0,18.93438,0.48079827953103244,0.65746,0.06915873336029225
3760000.0,18.490225000000002,0.5177302018184761,1.0,0.0,22.580766666666666,1.0913167378090662,0.0,0.0,19.065620000000003,0.5375108944012201,0.6215200000000001,0.04939159442658234
3800000.0,19.083975,1.238190555558795,1.0,0.0,22.742199999999997,1.066993836283353,0.0,0.0,19.4875,1.0339974584108036,0.63748,0.038263998745557165
3840000.0,18.554675,0.8856092320402944,1.0,0.0,21.76046666666667,0.5621140651346683,0.0,0.0,18.88438,0.5038152335926332,0.64302,0.05724269036304987
3880000.0,19.290975,0.8745753608895006,1.0,0.0,21.60936666666667,0.6590755158216363,0.0,0.0,19.004700000000003,0.9543455307172563,0.65826,0.05182407162699587
3920000.0,18.615225,0.5775553626060449,1.0,0.0,22.16926666666667,0.33356648845803205,0.0,0.0,19.76248,0.11998450566635702,0.6498999999999999,0.08104435822436994
3960000.0,18.195325,0.5059765378700875,1.0,0.0,22.341166666666666,0.683229396973585,0.0,0.0,18.97344,0.5880135529050335,0.6619999999999999,0.08458186566871177
4000000.0,18.949225,0.6228959318176677,1.0,0.0,21.385433333333335,0.8007652187473904,0.0,0.0,19.1922,0.21409000910831832,0.64378,0.08471395162545543
}\dataAblationWhen

%% file: data/sota_steps_a.tex
\pgfplotstableread[col sep=comma]{
8000.0,40.73623333333334,0.06309550080807992,40.784533333333336,0.10410998457827715,40.7404,0.06744988262900739
1608000.0,39.57916666666667,0.2133502962006127,39.42166666666667,0.060770515511681854,39.85536666666667,0.7288242052938565
3208000.0,39.818333333333335,0.03723370635444326,39.7029,0.22661160605758848,40.5054,0.489363348307437
4808000.0,39.5833,0.0651753020706462,39.73333333333334,0.7346692377450489,40.861666666666665,0.04237155754617485
6408000.0,39.577933333333334,0.20181520149769747,40.085433333333334,0.5464778149405738,40.43833333333333,0.33277453161095455
8008000.0,40.3029,0.7112323717792014,40.0,0.5572684511675368,39.12253333333334,0.9370822850149739
9608000.0,40.3054,0.6567438516397893,39.82083333333333,0.5203415694415452,38.050000000000004,1.5384461858208336
11208000.0,40.086666666666666,0.3957710477311629,38.45706666666666,1.2782436526560785,37.28496666666666,0.8248207131788658
12808000.0,39.4904,0.20730127833662873,39.6592,0.8164182873992969,38.52833333333333,1.7745097131946688
14408000.0,39.861266666666666,0.6753725210742757,39.43873333333333,0.5364187191198899,37.74626666666666,2.1252765409601557
16008000.0,38.9546,1.3237679126896336,39.785000000000004,0.37437564913688737,36.581266666666664,2.7941643358176993
17608000.0,39.50083333333333,0.4782990858819965,39.1096,1.5260850194752165,35.46586666666667,0.871429309939838
19208000.0,38.369566666666664,1.6004180710746245,38.9754,0.5015291483719247,32.99543333333333,1.456289134142743
20808000.0,37.961666666666666,2.760204048415423,39.77956666666666,0.9525632309137744,31.260433333333335,1.5948976066053764
22408000.0,37.46746666666666,3.153136434656064,39.29333333333333,1.2904930151776184,30.41333333333333,0.9145472553248536
24008000.0,35.52583333333333,3.4336917288668882,38.35833333333333,1.3164092634469293,29.65206666666667,1.4895072033692516
25608000.0,36.4996,3.9112351561452634,39.319966666666666,1.319493605727424,28.523733333333336,0.9649977766238051
27208000.0,34.83083333333334,3.71199987278496,39.15876666666667,1.4309084674507366,29.870833333333337,1.8462281410727355
28808000.0,35.793366666666664,3.317084042080065,38.21876666666667,1.3654830801669515,27.337100000000003,1.1586679794775836
30408000.0,33.861266666666666,4.297263048603018,37.9879,2.507275510722081,26.329166666666666,1.183554556785994
32008000.0,33.067499999999995,4.928155042880313,36.0425,2.4955234921755394,26.638766666666665,1.2511511534938091
33608000.0,34.56126666666666,5.656723083395741,37.329166666666666,2.821693738795115,27.734166666666667,1.5646809223892544
35208000.0,35.24373333333333,6.370995045952833,36.8183,2.7210259682700566,26.0625,1.3084775351529727
36808000.0,32.95246666666667,4.865532503459638,36.0713,3.750000555555515,25.887533333333334,0.9258310008971512
38408000.0,31.830433333333335,5.744929217047752,35.21876666666667,3.8910141405836876,25.436233333333334,0.6031332430639921
40008000.0,30.941266666666667,6.330759094207337,37.54913333333334,4.331396430970297,25.455833333333334,0.5354543387525115
41608000.0,31.977500000000003,5.618643471515165,35.98413333333333,3.79267061258364,25.03,0.4682057952083328
43208000.0,30.4104,6.323608301173206,35.8987,3.6541010066316817,24.7546,0.6548243937626835
44808000.0,29.932100000000002,6.587666012076407,36.21413333333333,4.4273740571835845,24.475866666666665,0.2032700393291871
46408000.0,30.459999999999997,6.71683429600582,36.444199999999995,4.954873542281376,28.932533333333335,6.845141576979172
48008000.0,31.833766666666666,6.159676455969277,33.655433333333335,4.1871013890545115,24.138333333333332,0.774742788858914
49608000.0,31.819166666666664,5.444908923225643,34.0829,4.134490112053318,23.642033333333334,0.6644368960923899
51208000.0,32.92293333333333,6.308032488467031,31.8429,4.703950231454411,23.657933333333336,0.8108883249594584
52808000.0,31.148333333333337,6.128822050135102,30.6046,4.7014157782523345,23.9808,0.4625024180116973
54408000.0,29.89623333333334,6.735036624663267,30.94083333333333,3.9308657514033145,23.963333333333335,0.40417268036763226
56008000.0,29.560433333333332,7.293423021843417,30.63416666666667,3.916663297870892,23.848366666666667,0.8447855993616888
57608000.0,29.854133333333333,7.072687021838935,30.37086666666667,3.374693424429676,23.5904,0.7546287078203863
59208000.0,29.586233333333336,6.970988595760447,30.205,4.143781988312932,23.178766666666665,0.9479928140843441
60808000.0,29.1825,7.216630524466849,29.733733333333333,4.013079184643909,23.444999999999997,0.6130175038283981
62408000.0,30.065033333333332,6.430810773836288,29.222466666666666,3.473047539105806,24.013333333333332,0.7812850625020872
64008000.0,29.02376666666667,7.325026214439251,28.491233333333337,2.779452171841703,23.450833333333335,0.6960373952278388
65608000.0,28.746266666666667,7.496198483824242,28.95873333333334,3.1514454422622578,23.605433333333334,1.0009867009883575
67208000.0,28.788299999999996,7.462417100823743,28.3346,3.8344928591231824,24.16586666666667,1.6372764098411194
68808000.0,28.87,7.444408416164892,28.4525,3.237206376903806,24.067966666666667,2.2031202620122414
70408000.0,29.286666666666665,7.518107630839618,27.734166666666667,3.041409516801197,23.446233333333335,0.9947811328238106
72008000.0,29.819600000000005,6.9916070270003035,27.574166666666667,2.8512601335004297,23.6867,0.7584401668336567
73608000.0,29.240400000000005,7.6401391204610904,27.046266666666668,2.7697088671715826,23.5017,0.581125620383981
75208000.0,28.60043333333333,7.659145734059089,28.0204,3.9413552855162215,23.162900000000004,0.7908827262411708
76808000.0,28.19496666666667,7.874513426802249,27.9983,3.1758549221692522,22.90706666666667,0.28074016852282135
78408000.0,27.857066666666665,8.220905944129406,27.416666666666668,3.1385309284582323,23.485833333333332,0.6287710358751875
80008000.0,27.816666666666666,8.278321169711218,27.75916666666667,2.962698491053196,23.218766666666667,0.7689514780241706
81608000.0,28.106666666666666,8.147836431566082,27.179199999999998,2.812824467802189,22.949166666666667,0.7371520753699481
83208000.0,28.065866666666665,8.316949042500836,27.41333333333333,2.977776248963123,23.11293333333333,0.8253026973312423
84808000.0,27.858766666666668,8.143932365605425,27.0296,2.255167285738835,23.0712,0.8230431337420894
86408000.0,28.0275,8.358004396983766,26.7971,2.378481612289655,22.98913333333333,0.46859226294177075
88008000.0,28.015866666666668,8.208115316089453,26.632900000000003,2.4804412550995845,23.101266666666664,1.005396854757144
89608000.0,27.673333333333336,8.436045796593463,26.3446,2.262355250618257,22.605866666666667,0.7407955468428667
91208000.0,28.160433333333334,8.281622166513568,26.272066666666664,2.371732138238962,22.7142,0.6162122091184713
92808000.0,27.37793333333333,8.563654284760032,26.1771,2.419619870971472,22.711233333333336,0.6465243296960209
94408000.0,27.692899999999998,8.263784412725201,26.4646,2.6281758857935418,22.391266666666667,0.6382079093489488
96008000.0,27.557466666666667,8.48117354549999,25.9954,2.136685865228359,22.830000000000002,0.8690604850450094
97608000.0,27.796666666666667,8.47796188256483,25.749133333333333,2.1641178053783383,22.53916666666667,0.41049443628656235
99208000.0,27.642933333333332,8.283942036789542,26.156666666666666,2.460660762658861,22.541633333333333,0.32935397708578223
100808000.0,27.481266666666667,8.587461256713508,26.2821,2.175982188958969,22.5925,1.050272405934131
102408000.0,27.675833333333333,8.76495325461326,25.492499999999996,2.297921697244418,23.167066666666667,0.310363296083083
104008000.0,27.305866666666663,8.629411164281272,25.084999999999997,2.274833327227881,22.649600000000003,0.21065252589671676
105608000.0,27.481233333333336,8.434134842544447,25.1829,2.256001068262158,22.789599999999997,0.6985082819838292
107208000.0,27.445800000000002,8.421284135253163,25.3175,2.281077961549467,22.168366666666667,0.21162004210901686
108808000.0,27.547066666666666,8.320686710575968,25.549566666666667,2.4926543834947426,22.442133333333334,0.654294573482685
110408000.0,27.6425,8.439796391303918,25.11876666666667,2.2566259198684704,22.406666666666666,0.7218538741016463
112008000.0,27.8275,8.478234879187216,24.8733,2.8716102602314724,22.355433333333334,0.8610369730092252
113608000.0,27.532500000000002,8.721581708612263,24.71666666666667,2.4389225795192617,22.617500000000003,0.7745934159286401
115208000.0,27.474966666666663,8.15275182840466,25.4029,2.7750889751501653,22.3471,0.4979984805867053
116808000.0,27.746700000000004,8.271265611742862,24.753766666666667,2.9491335239731375,22.344166666666666,0.3737227344198147
118408000.0,27.365800000000004,8.551387011473636,25.059166666666666,2.947305502017431,22.159599999999998,0.5501970737835661
120008000.0,27.497933333333332,8.351500745907222,25.23,2.9732345394648347,22.445466666666665,0.22684735445277324
121608000.0,27.4592,8.679577235480233,25.09293333333333,2.9634189998865983,22.388733333333334,0.4767529082822219
123208000.0,27.374166666666667,8.349522524605158,25.101233333333337,2.2649618191533003,22.4729,0.585743533183821
124808000.0,27.462933333333336,8.65315043065562,24.550833333333333,2.74817194068267,22.665833333333335,0.2735288934565329
126408000.0,27.3079,8.600388275343542,24.85996666666667,3.1153917424441007,22.3404,0.3483895520821473
128008000.0,27.514166666666668,8.115568804190842,24.5596,2.746230110533347,22.38123333333333,0.46695186285335893
129608000.0,27.648333333333337,8.837172767098966,24.499166666666667,2.6830380272287524,22.172933333333333,0.2545937984755768
131208000.0,27.659566666666663,8.584635463560593,24.51546666666667,2.7806326718132968,22.209166666666665,0.5892096872553568
132808000.0,27.585800000000003,8.85885051723228,24.606233333333336,2.6587315275438312,22.662033333333337,0.5696502338179886
134408000.0,27.757499999999997,8.389270677478466,24.61836666666667,2.696202831555684,22.592933333333335,0.39679003963753423
136008000.0,27.87166666666667,9.140773201552602,24.47666666666667,2.7869117344393155,22.46623333333333,0.04343626237245712
137608000.0,28.093766666666667,8.960237462007107,24.58543333333333,3.0417371750731887,22.638733333333334,0.5685294265813231
139208000.0,27.6217,9.387611614249922,24.857933333333335,2.7483739511856005,22.3467,0.08664767740684151
140808000.0,27.784966666666666,8.864529647734026,24.777500000000003,3.0756788074613155,22.100833333333338,0.21006282129136689
142408000.0,27.708366666666667,8.91155888208617,24.66163333333333,2.77883261620575,22.465033333333334,0.19454744636948756
144008000.0,27.786699999999996,8.919253546121446,24.58083333333333,3.003640571410339,22.451266666666665,0.24899339125544956
145608000.0,27.74,9.007081628363318,24.296633333333332,2.350673883993458,22.47583333333333,0.4817826434768656
147208000.0,27.701233333333334,9.244703711254836,24.775833333333335,2.815371026269106,22.1775,0.22151552240569217
148808000.0,27.438366666666667,8.626047165932311,25.39,2.932014687321102,22.390833333333333,0.26148082572575476
150408000.0,27.739166666666666,9.049396005013568,24.941666666666663,2.9196053983753063,22.3625,0.18981164347847537
152008000.0,27.674599999999998,9.26952924838509,24.87,2.889120396706698,22.432500000000005,0.035484174876509435
153608000.0,27.7825,9.0919568392435,24.50333333333333,2.455707920109574,22.270033333333334,0.2301116588866269
155208000.0,27.762533333333334,9.13807553822041,24.6837,3.2894648754268023,22.233766666666668,0.38068321330062527
156808000.0,27.242900000000002,8.504866223913618,24.230433333333334,2.875665031876202,22.2871,0.22856469543654342
158408000.0,27.690033333333332,8.6488604476094,24.284166666666668,2.7279642621477938,22.2129,0.25371996373955213
}\dataSotaStepsA

%% file: data/sota_steps_b.tex
\pgfplotstableread[col sep=comma]{
0.0,45.0,0.0,45.0,0.0
40000.0,45.0,0.0,45.0,0.0
80000.0,45.0,0.0,45.0,0.0
120000.0,44.7969,0.14488411461118336,45.0,0.0
160000.0,44.78123333333334,0.15794311915651302,44.86976666666666,0.09405829161866774
200000.0,44.72393333333333,0.02946278254943948,44.091133333333325,0.42121988979101443
240000.0,43.770833333333336,0.2115958621733612,44.28906666666666,0.5390464069900546
280000.0,43.32293333333333,0.6185165281183328,43.359399999999994,1.0006783832314303
320000.0,43.1875,0.18501812523822084,43.33596666666667,0.6222617098581237
360000.0,43.9974,0.10170732520325194,42.872366666666665,1.1131963239049782
400000.0,43.96876666666666,0.3382510934531062,42.625,1.8416830400478816
440000.0,43.34116666666666,0.6324431797051435,41.934866666666665,0.6417085856437412
480000.0,43.22656666666666,1.0717516638143791,40.47656666666666,1.0164623794754486
520000.0,43.46873333333334,0.23237465916537106,40.427099999999996,0.7080753820509983
560000.0,42.52603333333334,0.2475887629832082,39.638,0.8169291809371632
600000.0,43.7578,0.21312958499466736,41.38803333333333,1.2770944896739458
640000.0,43.020833333333336,0.3305202599270148,40.7526,1.0559491275624957
680000.0,42.97656666666666,1.1153531707739743,39.22656666666666,0.6041848962767016
720000.0,43.562466666666666,0.733664638797743,39.31513333333333,0.8140088956652867
760000.0,43.0703,0.7213101598249311,37.59113333333334,0.4280378825394904
800000.0,43.15626666666666,0.6225386431557671,38.21353333333334,1.0343103896907473
840000.0,43.4297,0.6473916022521965,35.674499999999995,0.8398964340917293
880000.0,43.1927,1.036343157453168,37.177099999999996,0.9923480572191723
920000.0,42.796899999999994,0.6608019219100394,34.205733333333335,0.6118570711821135
960000.0,42.333333333333336,0.38386783605240443,34.49216666666667,1.1344456188915453
1000000.0,42.51303333333333,1.2509467596806623,33.8099,0.8193187698732803
1040000.0,41.705733333333335,0.8089172612545467,32.61976666666667,1.5082406093046163
1080000.0,41.356766666666665,1.2963538362996772,31.895866666666667,1.2874883697425088
1120000.0,41.4323,0.8625282178958958,30.885433333333335,0.7174406100081652
1160000.0,40.458333333333336,2.027559360961405,30.4896,1.0915951294626898
1200000.0,39.018233333333335,1.9166778138806264,29.6771,1.3748071282910928
1240000.0,40.0104,0.5456940046094221,29.265633333333337,0.16278120967182455
1280000.0,40.63543333333333,0.6827832517636095,27.567700000000002,1.3082649604214982
1320000.0,40.34116666666666,1.8929067670883544,26.666666666666668,1.1214133027370223
1360000.0,40.0703,1.903442724819074,25.968766666666667,1.275437144232866
1400000.0,38.919266666666665,1.5090106965234615,24.51823333333333,1.4627783526183618
1440000.0,38.640633333333334,2.653861168603629,25.09113333333333,0.6346169150667901
1480000.0,36.7474,2.062273940742758,24.60936666666667,0.9221963071325369
1520000.0,37.390633333333334,2.9199650732309976,24.15363333333333,1.1491664873096308
1560000.0,38.08593333333334,3.2673102213016487,23.067733333333337,1.3864725537212124
1600000.0,38.86456666666667,1.782004317864827,23.0625,0.8683283979386303
1640000.0,38.161500000000004,3.0338808886748776,23.317700000000002,0.7718692505858745
1680000.0,36.88803333333333,2.8811141321062275,22.7526,0.8307970911520259
1720000.0,36.34376666666667,1.5117059730281173,23.374966666666666,0.7188686497236861
1760000.0,36.770833333333336,1.241772830360778,22.119799999999998,1.4883521648678
1800000.0,35.94266666666667,2.547338953147425,22.447933333333335,0.8035054628867744
1840000.0,35.59636666666666,1.1332723160045093,21.739566666666665,1.1159860821513667
1880000.0,35.01303333333333,1.911977092494107,23.1901,1.9312508709814655
1920000.0,33.50000000000001,1.7791139423881772,21.48696666666667,1.1247164096883353
1960000.0,34.138,0.9010820199441695,21.6901,0.7055320687254417
2000000.0,34.828133333333334,1.4261743379482827,21.52603333333333,0.5811357978610119
2040000.0,33.421866666666666,1.1060086839120602,22.02606666666667,0.3067405019809995
2080000.0,30.625,2.847555340989882,21.351566666666667,0.36483167199256294
2120000.0,31.526033333333334,2.027837775124583,21.46356666666667,0.6496809285248336
2160000.0,30.127599999999997,2.1855337395397636,21.062466666666666,0.574345925572927
2200000.0,30.088533333333334,2.3924002805179216,20.90103333333333,0.3352675979307013
2240000.0,29.651033333333334,1.3595515690435909,20.838533333333334,0.3147017247419452
2280000.0,28.609366666666663,2.779131344303267,20.648433333333333,0.4783500972672172
2320000.0,28.1198,1.5200816710514824,20.367166666666666,0.5513487845476962
2360000.0,28.528666666666666,0.4762307937217921,20.5599,0.3904462660426745
2400000.0,28.630200000000002,1.4336602177643063,20.4844,0.9863108941910765
2440000.0,27.109366666666663,1.2763146564316432,20.929666666666666,0.634727785922613
2480000.0,26.609333333333336,2.048058341182909,21.3698,0.6611437413049198
2520000.0,26.320333333333334,2.2926136065392457,20.606766666666665,0.6166422968157659
2560000.0,26.835966666666668,3.5961376580375157,20.40886666666667,0.4807769152343125
2600000.0,26.885433333333335,2.2962760132198587,20.309866666666668,0.5584523932758779
2640000.0,25.513033333333336,2.034969772639278,20.4974,0.6655414988313406
2680000.0,25.473966666666666,0.925830813677939,20.51823333333333,1.0093922538945015
2720000.0,25.947933333333335,2.4634201432606297,20.3021,0.5981837566054984
2760000.0,25.255200000000002,0.49388643094001594,19.854166666666668,0.5364448423546347
2800000.0,25.994766666666663,0.8510661718626161,21.3073,0.5007077457626022
2840000.0,25.700499999999995,1.6307433294870988,20.523433333333333,0.22378855099301947
2880000.0,26.419300000000003,2.2431369240418646,20.588533333333334,0.7849951903603533
2920000.0,24.763,0.7963562310088791,19.908833333333334,0.9645571568802374
2960000.0,25.023433333333333,1.4198810427481425,19.778666666666666,0.5195587957317457
3000000.0,25.583333333333332,1.4546126868994678,20.0625,0.6622172654549767
3040000.0,25.317700000000002,0.8346437483541507,20.182299999999998,0.8245646608968873
3080000.0,25.273433333333333,1.3421907622324867,19.820333333333334,0.2180237958470493
3120000.0,24.656266666666667,1.2504554379194095,20.085933333333333,0.33973143641542736
3160000.0,24.8125,1.003384994240331,19.888,1.1048957597891305
3200000.0,24.479166666666668,1.590815401679975,20.218766666666667,0.28406656888052795
3240000.0,24.536466666666666,1.200620174002679,20.3177,0.5688410908739511
3280000.0,24.60676666666667,1.714488410252133,19.58073333333333,0.649260647882565
3320000.0,25.1224,0.9325663122087715,20.3073,0.24358189587898307
3360000.0,24.3646,1.2212802708633261,20.364566666666665,0.45668123334431837
3400000.0,23.8802,1.3279066407947016,19.721366666666665,0.5943654898760148
3440000.0,24.51823333333333,1.340435228656059,19.6198,1.0585406211698571
3480000.0,24.51303333333333,0.7700256330509753,19.76563333333333,0.9549889225651901
3520000.0,24.65363333333333,1.4714325091178628,19.95573333333333,0.45953894527256656
3560000.0,24.854200000000002,1.4380309407890592,19.9349,0.5888061537269007
3600000.0,25.40366666666667,1.0238632862296069,19.08073333333333,0.8733724075228286
3640000.0,24.97133333333333,1.2095523395960273,19.984333333333332,0.871150049583244
3680000.0,23.5625,0.9030402575005542,20.0,0.4783136209643203
3720000.0,24.04946666666667,1.1059412923940495,19.997400000000003,0.9443189397655851
3760000.0,24.403666666666666,1.678563277197365,19.729166666666668,0.15286251629777914
3800000.0,23.807299999999998,0.8830197091043137,19.411466666666666,0.4411789533612043
3840000.0,24.52863333333333,1.1741778580020237,19.28643333333333,0.5478501031811124
3880000.0,23.66926666666667,0.7773092277565958,19.453133333333334,0.8138757986053864
3920000.0,23.585933333333333,1.2445885995871175,19.5677,0.4663733268530696
3960000.0,24.10676666666667,1.7305350701123894,19.627566666666663,0.19993819600622176
4000000.0,25.0651,1.9826009734689425,19.2005,0.29650741418498644
}\dataSotaStepsB

%% file: data/sota_pp_tarmac.tex
\pgfplotstableread[col sep=comma]{
0.0,45.0,0.0,44.893233333333335,0.15099086800936892,45.0,0.0
40000.0,45.0,0.0,44.90103333333334,0.1399600022228577,44.83593333333334,0.23202530513334432
80000.0,45.0,0.0,44.78906666666666,0.22068120798009871,44.643233333333335,0.5045442586026427
120000.0,44.7969,0.14488411461118336,44.7682,0.19142211645122653,44.05466666666666,1.0186511648035146
160000.0,44.78123333333334,0.15794311915651302,44.625,0.3724341910548332,43.770833333333336,0.9325861616434648
200000.0,44.72393333333333,0.02946278254943948,44.604166666666664,0.2157789815734819,41.89066666666667,1.328697009186905
240000.0,43.770833333333336,0.2115958621733612,44.44793333333333,0.7642535763010143,42.5651,1.2031834855914534
280000.0,43.32293333333333,0.6185165281183328,43.890633333333334,0.6086405306546352,42.8099,0.3161503123515764
320000.0,43.1875,0.18501812523822084,43.34633333333334,0.9331523145886846,41.911433333333335,0.46020137862558486
360000.0,43.9974,0.10170732520325194,44.0052,0.39226328743162664,40.8802,0.7064869472726748
400000.0,43.96876666666666,0.3382510934531062,42.395833333333336,1.8835453420492843,41.4141,1.7860862707794007
440000.0,43.34116666666666,0.6324431797051435,41.96616666666667,1.199372964881603,40.312533333333334,0.4106359404089657
480000.0,43.22656666666666,1.0717516638143791,42.57816666666667,0.27739546179096597,39.46093333333334,1.6211574966328492
520000.0,43.46873333333334,0.23237465916537106,41.5,0.816496580927726,38.4427,1.3196324968212414
560000.0,42.52603333333334,0.2475887629832082,41.958333333333336,0.575734605139864,38.18226666666667,0.9881516291653917
600000.0,43.7578,0.21312958499466736,41.7396,1.3129920131770296,36.4271,1.347765093280972
640000.0,43.020833333333336,0.3305202599270148,41.182300000000005,0.9100506396166455,35.54683333333333,0.988267798839069
680000.0,42.97656666666666,1.1153531707739743,41.0469,1.9284686688320007,35.90106666666667,1.101202507362848
720000.0,43.562466666666666,0.733664638797743,39.79946666666667,0.7538166060492145,35.2448,2.5187605933606867
760000.0,43.0703,0.7213101598249311,39.96353333333334,0.89654672431998,33.958333333333336,0.9761323999449161
800000.0,43.15626666666666,0.6225386431557671,39.9948,1.2940906021862093,34.14843333333334,2.956610894626182
840000.0,43.4297,0.6473916022521965,40.169266666666665,1.879405917363842,32.38023333333333,0.8838965261210648
880000.0,43.1927,1.036343157453168,39.624966666666666,0.5113852841927395,32.52343333333334,0.7655422580686774
920000.0,42.796899999999994,0.6608019219100394,38.7552,1.228035213935931,31.367166666666666,1.519023230749142
960000.0,42.333333333333336,0.38386783605240443,37.89843333333334,2.3122359226995464,32.28386666666666,2.3574296940146953
1000000.0,42.51303333333333,1.2509467596806623,38.90103333333334,1.7348823770567936,30.625033333333334,2.5494141867408575
1040000.0,41.705733333333335,0.8089172612545467,37.77863333333334,2.8153203492478243,30.526033333333334,3.393128827825762
1080000.0,41.356766666666665,1.2963538362996772,36.080733333333335,2.022071812660359,29.54686666666667,3.5099172092167024
1120000.0,41.4323,0.8625282178958958,35.627633333333335,0.9716085848849949,26.637999999999995,3.415630562575526
1160000.0,40.458333333333336,2.027559360961405,35.46353333333334,3.4946637800445925,27.091133333333335,2.4116716834778495
1200000.0,39.018233333333335,1.9166778138806264,34.69533333333334,3.1696556377976175,24.54946666666667,1.3348235572122955
1240000.0,40.0104,0.5456940046094221,34.916666666666664,3.7432859542860935,24.854166666666668,2.0452908682032374
1280000.0,40.63543333333333,0.6827832517636095,32.27603333333334,2.8954633496020783,25.148433333333333,2.1437469631983666
1320000.0,40.34116666666666,1.8929067670883544,30.924466666666664,2.7711831664872353,23.458333333333332,0.4235577043200706
1360000.0,40.0703,1.903442724819074,31.067733333333337,4.03364472103773,22.721333333333334,0.8463998516593029
1400000.0,38.919266666666665,1.5090106965234615,29.8229,4.1833253885714745,22.096333333333334,1.1397167903572463
1440000.0,38.640633333333334,2.653861168603629,30.01563333333333,4.673076064958593,23.28643333333333,2.1024383373174635
1480000.0,36.7474,2.062273940742758,28.1771,2.26336694476761,22.960966666666668,0.7957944472175097
1520000.0,37.390633333333334,2.9199650732309976,27.645833333333332,4.100070902096967,22.20053333333333,1.0042180683939572
1560000.0,38.08593333333334,3.2673102213016487,28.718733333333333,3.992514350075095,21.76563333333333,0.5399124640984767
1600000.0,38.86456666666667,1.782004317864827,27.710966666666668,3.3011743590560148,21.804666666666666,0.8573171577012138
1640000.0,38.161500000000004,3.0338808886748776,27.625,3.9923810497830328,20.8854,0.013257450735340535
1680000.0,36.88803333333333,2.8811141321062275,26.718766666666667,3.4348841846883604,20.984366666666663,0.40339969701966305
1720000.0,36.34376666666667,1.5117059730281173,25.440133333333335,3.3718290756732543,21.039033333333332,0.8456716160675037
1760000.0,36.770833333333336,1.241772830360778,24.960966666666668,1.7905183762127546,21.8021,0.13279023558480005
1800000.0,35.94266666666667,2.547338953147425,25.6797,2.598117056639289,21.786466666666666,0.5436818084954554
1840000.0,35.59636666666666,1.1332723160045093,23.9974,1.29690312925317,21.088533333333334,0.9130904202517713
1880000.0,35.01303333333333,1.911977092494107,25.645833333333332,3.498737433544976,20.984400000000004,0.5516797863495342
1920000.0,33.50000000000001,1.7791139423881772,23.76823333333333,2.3119654788849147,20.9297,0.23385358667337133
1960000.0,34.138,0.9010820199441695,22.84116666666667,1.274396788898793,20.6224,0.6899033120662633
2000000.0,34.828133333333334,1.4261743379482827,22.960933333333333,1.6620325354483558,21.04426666666667,0.2606102624396987
2040000.0,33.421866666666666,1.1060086839120602,22.098966666666666,1.087643818331882,20.361966666666664,0.2862841517715492
2080000.0,30.625,2.847555340989882,22.893199999999997,1.751001047020436,20.1328,0.4651681918044996
2120000.0,31.526033333333334,2.027837775124583,22.445300000000003,1.5587558265060844,20.549466666666667,0.803428944517739
2160000.0,30.127599999999997,2.1855337395397636,20.9375,1.185137370378078,20.156233333333333,0.7707841908659574
2200000.0,30.088533333333334,2.3924002805179216,21.73696666666667,0.6118467146452795,19.7604,0.43195222729680055
2240000.0,29.651033333333334,1.3595515690435909,21.619799999999998,0.6646790854740848,20.0495,0.8473271269114429
2280000.0,28.609366666666663,2.779131344303267,20.78386666666667,0.7168974465632371,20.487000000000002,0.11455691453014487
2320000.0,28.1198,1.5200816710514824,21.351566666666667,0.45506402724110073,20.554666666666666,1.419842297189688
2360000.0,28.528666666666666,0.4762307937217921,21.023433333333333,0.6324067379639653,19.6927,0.7934933017990758
2400000.0,28.630200000000002,1.4336602177643063,21.2604,0.4543801345422863,20.117166666666666,0.4425261147347372
2440000.0,27.109366666666663,1.2763146564316432,21.182266666666667,0.7259549358526934,19.557299999999998,0.6740630880464132
2480000.0,26.609333333333336,2.048058341182909,20.958333333333332,0.4671049156476755,19.2266,0.1823199385695374
2520000.0,26.320333333333334,2.2926136065392457,21.041666666666668,0.254840425015777,19.161466666666666,0.3274199783492489
2560000.0,26.835966666666668,3.5961376580375157,20.726566666666667,0.4530415163708021,19.450533333333336,0.20503323849778268
2600000.0,26.885433333333335,2.2962760132198587,20.765633333333334,0.8686537719687607,20.3021,1.3780672625093457
2640000.0,25.513033333333336,2.034969772639278,20.489566666666665,1.2954479851473082,19.90103333333333,0.36097118382989446
2680000.0,25.473966666666666,0.925830813677939,21.3203,0.2917248475304529,20.1198,0.3026116433098154
2720000.0,25.947933333333335,2.4634201432606297,20.5365,0.40188088616736495,19.73436666666667,0.7863798672114856
2760000.0,25.255200000000002,0.49388643094001594,21.145866666666667,0.9273811202640596,20.067733333333333,0.32022476828350105
2800000.0,25.994766666666663,0.8510661718626161,20.6354,1.0336152701400405,19.999966666666666,0.0710431949982235
2840000.0,25.700499999999995,1.6307433294870988,20.679733333333335,0.8579738198544029,18.981800000000003,0.6539150760356174
2880000.0,26.419300000000003,2.2431369240418646,20.541700000000002,0.7106033633469514,19.450533333333336,0.34271020151466514
2920000.0,24.763,0.7963562310088791,20.835933333333333,0.39516305776504856,19.51823333333333,0.5585250178421337
2960000.0,25.023433333333333,1.4198810427481425,20.7526,0.4259051850666619,20.112,0.4580889506053016
3000000.0,25.583333333333332,1.4546126868994678,20.184900000000003,0.7509153125796982,19.8203,1.2651459230723803
3040000.0,25.317700000000002,0.8346437483541507,20.651033333333334,0.3912642272883575,18.7474,0.4338764878011561
3080000.0,25.273433333333333,1.3421907622324867,20.320333333333334,0.7305457792332769,19.007833333333334,0.49460810300232094
3120000.0,24.656266666666667,1.2504554379194095,19.971366666666665,0.66673481151721,19.648433333333333,1.0164153656628554
3160000.0,24.8125,1.003384994240331,20.09633333333333,0.3431713889913057,19.3802,0.3218798222939734
3200000.0,24.479166666666668,1.590815401679975,20.700533333333333,0.681708864219585,18.875,0.36925680314202286
3240000.0,24.536466666666666,1.200620174002679,19.9375,0.8739659070390938,19.132833333333334,0.315326078140639
3280000.0,24.60676666666667,1.714488410252133,19.3151,0.5965171637653573,19.1432,0.3572361777125417
3320000.0,25.1224,0.9325663122087715,20.3854,0.7916013685351138,18.682266666666667,0.48586379664355434
3360000.0,24.3646,1.2212802708633261,20.612000000000002,0.14690861104782146,19.40363333333333,0.4953621054900706
3400000.0,23.8802,1.3279066407947016,20.28386666666667,0.31719169107794937,18.937466666666666,0.2906963402284623
3440000.0,24.51823333333333,1.340435228656059,20.10676666666667,0.3941456358025837,19.093733333333333,0.8037944278372556
3480000.0,24.51303333333333,0.7700256330509753,20.614566666666665,0.33810247296082035,19.174466666666664,0.4748681524615244
3520000.0,24.65363333333333,1.4714325091178628,19.822933333333335,0.5618268554951377,19.257833333333334,0.1053957726329139
3560000.0,24.854200000000002,1.4380309407890592,20.1875,0.6411368392680834,19.1172,0.29732006323152943
3600000.0,25.40366666666667,1.0238632862296069,19.2422,0.35223532853288136,18.890599999999996,0.2801691036975112
3640000.0,24.97133333333333,1.2095523395960273,20.04426666666667,0.7199779826880502,19.049500000000002,0.6174939729800343
3680000.0,23.5625,0.9030402575005542,19.4844,0.412352009816855,19.0755,0.18336653638727723
3720000.0,24.04946666666667,1.1059412923940495,19.734366666666666,0.7765600484758974,18.7474,0.27252659809029023
3760000.0,24.403666666666666,1.678563277197365,19.640600000000003,0.6804004164215856,18.518233333333335,0.6835657799769946
3800000.0,23.807299999999998,0.8830197091043137,20.135433333333335,0.9716085848849948,19.09636666666667,0.22612023841802079
3840000.0,24.52863333333333,1.1741778580020237,20.080699999999997,0.4668296548706669,19.015633333333334,0.2008265807993449
3880000.0,23.66926666666667,0.7773092277565958,20.23436666666667,0.3156671805698032,18.73176666666667,0.20829713285486073
3920000.0,23.585933333333333,1.2445885995871175,19.710933333333333,0.6744514083477198,18.75,0.3181400425389222
3960000.0,24.10676666666667,1.7305350701123894,19.364566666666665,0.4564786109142717,19.499966666666666,0.21545539883903742
4000000.0,25.0651,1.9826009734689425,19.747400000000003,0.405152177171327,18.791666666666668,0.4774405745453793
}\dataSotaPpTarmac

%% file: data/sota_lj.tex
\pgfplotstableread[col sep=comma]{

0.0,-19.50235,0.1910500000000006,-19.64395,0.14475000000000016,-19.6125,0.29235557973125814,0.0,-19.703425000000003,0.08598384659341494,-19.201833333333337,0.22234133418887397,-19.560533333333332,0.18738945778482033
40000.0,-19.366600000000002,0.2350000000000012,-18.825,0.16639999999999944,-19.443849999999998,0.352255720322609,40000.0,-19.18545,0.3354428006382015,-19.181766666666665,0.12165024546716834,-19.335800000000003,0.18530689859437618
80000.0,-19.32795,0.4482499999999998,-18.94455,0.10545000000000115,-18.4703,0.5554729471360421,80000.0,-19.4123,0.24275695664594205,-19.2159,0.13135496437769942,-19.279833333333332,0.23208396947848162
120000.0,-19.0258,0.4297000000000004,-19.05565,0.19784999999999897,-17.454875,0.5954310345245701,120000.0,-18.9924,0.4388073609227627,-19.191533333333332,0.07468718468087761,-18.738,0.10011566643970773
160000.0,-18.879900000000003,0.22049999999999947,-18.9348,0.29959999999999987,-17.824025,0.4072436210365975,160000.0,-19.067775,0.35846385713904216,-18.764466666666667,0.3422178286153746,-19.124066666666668,0.21558250289751266
200000.0,-18.65625,0.6925500000000007,-18.506999999999998,0.09999999999999964,-18.099825000000003,0.19546230295123404,200000.0,-18.896774999999998,0.2394627776398653,-18.792199999999998,0.05337995878604689,-19.0862,0.13290066465848283
240000.0,-18.95235,0.5394500000000004,-17.3822,0.6661999999999999,-17.45495,0.5032090296288417,240000.0,-18.809874999999998,0.38104490795049284,-18.916433333333334,0.09204438542838471,-18.0754,0.2311770028931662
280000.0,-19.10605,0.2021499999999996,-17.999450000000003,0.7224500000000003,-17.120524999999997,0.5575669124643254,280000.0,-18.541499999999996,0.6219555570939139,-18.819100000000002,0.2009815414409985,-18.2332,0.6434472990592679
320000.0,-19.068550000000002,0.07364999999999888,-18.3723,0.3995999999999995,-16.128325,0.38063817711180753,320000.0,-19.0122,0.3265325634603693,-19.07683333333333,0.0958845254575639,-17.843233333333334,0.07551235366187116
360000.0,-18.694100000000002,0.4961000000000002,-17.969549999999998,0.08904999999999852,-16.31445,0.14921113396794483,360000.0,-18.093149999999998,0.6081755441482337,-18.869666666666667,0.11745593026984782,-17.780333333333335,1.0070402386311206
400000.0,-16.74145,0.3562499999999993,-18.27445,0.6357499999999998,-16.66425,0.566164099444675,400000.0,-18.695300000000003,0.33060202661205745,-18.9276,0.13342446052604634,-17.7565,1.0134223831486395
440000.0,-17.2576,0.7735999999999983,-18.438650000000003,0.08985000000000021,-15.62715,0.9250505297009459,440000.0,-18.1749,0.5455366486314186,-18.816433333333332,0.18812216479960228,-17.498966666666664,0.8883177522085715
480000.0,-16.2281,0.6340000000000012,-18.37305,0.11524999999999963,-15.556550000000001,0.31116700098178846,480000.0,-17.445425,1.0553663448656105,-18.86236666666667,0.36704092717600556,-17.159733333333335,0.17258416948131608
520000.0,-16.5705,1.1123000000000003,-18.8541,0.15140000000000065,-15.058200000000003,0.24509929620461982,520000.0,-18.080274999999997,0.4772072263440689,-18.67343333333333,0.23796427369577097,-16.635299999999997,1.8717497286407354
560000.0,-16.13165,0.6812499999999995,-18.37675,0.695549999999999,-15.199875,0.5435218228139511,560000.0,-17.1862,0.8162674745449562,-18.616799999999998,0.17156946115203586,-17.046466666666664,0.8974116570572399
600000.0,-16.0041,1.056400000000001,-18.34845,0.5906500000000001,-15.558975,0.7435941210600043,600000.0,-16.802925000000002,0.9019463076453053,-18.630233333333333,0.2896086593241908,-15.743733333333333,0.6219961914853029
640000.0,-15.907,0.6961000000000004,-18.76265,0.035349999999999326,-14.845324999999999,1.4580232240520041,640000.0,-16.69905,1.1192660597462956,-18.665133333333333,0.19751938864042837,-15.929433333333336,1.0727467092572336
680000.0,-17.111150000000002,0.7267500000000009,-18.39845,0.13945000000000007,-14.555575,1.2415449173811637,680000.0,-16.335924999999996,1.3221612427669323,-18.5389,0.3127587035825984,-15.739333333333333,0.31493258890682535
720000.0,-15.371899999999998,0.6714999999999991,-18.4096,0.07949999999999946,-14.325875000000003,0.7325468120707375,720000.0,-16.646950000000004,1.447757221532671,-18.7383,0.27217308218607245,-15.732166666666666,0.7167815861350111
760000.0,-16.9369,0.6799,-18.647299999999998,0.5327999999999999,-14.477999999999998,0.5735232296951884,760000.0,-16.98575,0.7093356698348103,-18.644133333333333,0.5220783168154836,-15.733966666666667,1.0207432237780898
800000.0,-17.4992,1.798799999999999,-18.4203,0.12069999999999936,-15.389150000000003,0.7446921159109979,800000.0,-16.153225,1.0007103461416795,-19.04296666666666,0.14127331272709892,-15.133200000000002,0.40127526711722494
840000.0,-15.816999999999998,0.7794999999999997,-17.8547,0.39259999999999984,-14.737474999999998,0.7854189212611319,840000.0,-16.6236,0.8414774269105505,-18.906666666666666,0.2170610717952172,-14.9285,0.800180733751236
880000.0,-15.3793,1.3840000000000003,-18.4217,0.13339999999999996,-14.211025,1.2534995958814665,880000.0,-16.682425,0.7307550765304335,-18.811433333333337,0.39281768742707507,-15.546500000000002,0.796156630318432
920000.0,-15.693950000000001,0.97675,-18.879649999999998,0.07654999999999923,-13.990249999999998,1.5094719614818948,920000.0,-16.418175,1.3494024017597566,-18.521600000000003,0.4804410959385835,-14.5608,1.137577622259979
960000.0,-16.149,1.0334000000000003,-18.6553,0.583400000000001,-13.522450000000001,1.0106569558955203,960000.0,-16.058,1.3583533284826887,-18.747266666666665,0.4424135797594313,-14.9595,0.8848995460879542
1000000.0,-15.50195,1.1507499999999995,-18.053700000000003,0.0923999999999996,-13.72035,1.1242132326654053,1000000.0,-15.737200000000001,1.613806458346229,-18.76353333333333,0.27908565391689716,-13.967966666666667,1.4027763360168604
1040000.0,-14.8168,1.0554999999999994,-18.47715,0.4048500000000015,-13.6952,1.869493788435789,1040000.0,-15.841425,1.586104699215976,-18.8069,0.3471475766874953,-13.377999999999998,0.7718399315920371
1080000.0,-14.941999999999998,0.5834000000000001,-18.121850000000002,0.1941499999999987,-13.556724999999998,1.1049188191333335,1080000.0,-15.088275,1.1731052433072664,-18.249333333333336,0.1663673512308093,-13.895466666666666,1.533910247990048
1120000.0,-14.559149999999999,0.2693500000000002,-18.218700000000002,0.025000000000000355,-13.115425,0.6532697236785122,1120000.0,-15.510925,1.3795585315147016,-18.9836,0.2719207727752079,-13.808966666666663,1.4528517298364922
1160000.0,-14.400400000000001,1.1078000000000001,-18.366799999999998,0.10740000000000016,-13.179975,0.7468862308779032,1160000.0,-15.39415,1.3444539811760008,-19.05886666666667,0.1258788394537462,-13.627733333333333,1.4308030526790034
1200000.0,-14.79315,0.6052499999999998,-18.564449999999997,0.14605000000000068,-12.917775,0.6915335254888227,1200000.0,-14.798624999999998,2.219986554435635,-19.078,0.19360022382907083,-13.250533333333331,1.4205897398224756
1240000.0,-14.89495,0.3816500000000005,-18.821450000000002,0.185550000000001,-13.163574999999998,1.0613758697440787,1240000.0,-15.486775,1.5529942117970041,-18.994133333333334,0.22950390459036238,-12.935533333333332,1.5651151210764727
1280000.0,-15.4922,0.0007999999999999119,-18.51285,0.1816500000000012,-13.19755,0.8177358512991838,1280000.0,-14.9621,1.6583012738944636,-18.896766666666668,0.14250511881644337,-12.783166666666668,0.8746533992896208
1320000.0,-14.0789,0.6223000000000001,-18.208399999999997,0.5264000000000006,-13.780175,1.413960587100998,1320000.0,-15.526174999999997,2.162256409835568,-19.033333333333335,0.06780488346883408,-12.214566666666665,1.9781332956996491
1360000.0,-14.556649999999998,0.6265499999999999,-18.7816,0.31680000000000064,-13.2218,0.7650246172509746,1360000.0,-14.519625000000001,1.2554868585035046,-18.752333333333333,0.19863551769230336,-12.2931,1.165013736685824
1400000.0,-15.969149999999999,1.80785,-18.5324,0.19650000000000034,-12.668849999999999,1.1622237510479638,1400000.0,-14.852250000000002,1.5385361053611968,-18.9629,0.11832945533551636,-12.921100000000001,1.4474547477094641
1440000.0,-16.5117,0.8906000000000001,-18.8006,0.5014000000000003,-12.556650000000001,0.8669569928779624,1440000.0,-14.500875,1.092119662342456,-19.0754,0.33374973657917206,-11.4902,1.3058601941504564
1480000.0,-14.707,0.6034999999999995,-18.525550000000003,0.64785,-12.858200000000002,0.8692698344012637,1480000.0,-13.715625000000003,1.3164152087677354,-19.1056,0.14016506935277076,-11.896866666666668,1.5920335870270526
1520000.0,-14.68285,1.40585,-18.7498,0.12209999999999965,-12.99005,0.31442447026273207,1520000.0,-14.037599999999998,1.4334939065095462,-18.483966666666664,0.5502979213319107,-11.683333333333335,1.119607348830632
1560000.0,-14.059,1.1969000000000003,-18.71895,0.04005000000000081,-12.7758,1.693432381289551,1560000.0,-13.721,1.6654778098191523,-18.687766666666665,0.2613916899129653,-12.408333333333333,1.0559593878975124
1600000.0,-13.024600000000001,1.3516000000000004,-18.6186,0.3291000000000004,-12.60995,1.024501372619871,1600000.0,-13.937699999999998,1.3385295289981465,-18.400133333333333,0.40171560697698655,-11.942333333333332,1.596085146712278
1640000.0,-15.04705,1.3650499999999992,-17.9369,0.03220000000000134,-12.367475,1.0198357842687225,1640000.0,-13.326274999999999,1.7292740809007112,-18.791933333333336,0.27415701502768214,-11.372133333333332,1.8427803456251162
1680000.0,-14.53105,1.555250000000001,-18.199,0.3994,-12.833850000000002,0.823866665486594,1680000.0,-13.769349999999998,2.003798957605278,-18.6457,0.11026250495975629,-11.552600000000004,0.8286647130574987
1720000.0,-13.468150000000001,1.1923500000000002,-18.20645,0.3126500000000014,-12.215225,1.5298758403462025,1720000.0,-13.370999999999999,1.9899858730654345,-18.803,0.3060680425439197,-10.864333333333335,0.5039171514798402
1760000.0,-13.5242,1.2949000000000002,-18.257849999999998,0.3351500000000005,-12.598925,0.8522157484316991,1760000.0,-13.065700000000001,2.016115190161515,-18.688833333333335,0.36817535978872246,-11.2875,1.1780657904661638
1800000.0,-14.4047,0.6254,-18.33225,0.4681499999999996,-11.80675,1.3890521165528666,1800000.0,-12.521775,2.372370190732256,-18.938266666666667,0.2250234111277207,-11.385033333333332,1.253438950337121
1840000.0,-13.716999999999999,1.2045000000000003,-18.7652,0.0995999999999988,-11.380575,1.475628319352472,1840000.0,-12.8888,2.0962078057768982,-18.755733333333332,0.3375789224594583,-11.217733333333333,1.391573201651841
1880000.0,-12.78595,1.2257499999999997,-18.79065,0.42735000000000056,-12.1718,1.546342022322358,1880000.0,-12.841775,2.204584367148375,-18.502333333333336,0.6157462915483574,-11.775,0.93700168623114
1920000.0,-13.66095,0.4054500000000001,-18.8701,0.21229999999999905,-11.7925,0.9194272347499828,1920000.0,-12.907600000000002,2.5537584782042333,-18.73266666666667,0.5131325386508082,-10.316133333333333,0.9498130950642637
1960000.0,-12.551,0.6290999999999993,-18.688850000000002,0.5513500000000011,-11.572824999999998,0.9415190289500259,1960000.0,-12.70625,2.232735591712552,-19.0918,0.11404104524249042,-10.786233333333335,0.5874845605536276
2000000.0,-12.119150000000001,1.0441500000000001,-18.2457,0.13480000000000025,-11.788675000000001,1.6049282326244367,2000000.0,-12.97325,2.4684818061918143,-18.695966666666667,0.08942976138971925,-10.951166666666667,0.816070186251703
2040000.0,-12.38965,0.6955499999999999,-18.62775,0.5890500000000003,-12.319825,1.7790475603184415,2040000.0,-12.61545,2.5956386155433884,-18.938933333333335,0.11797065548497836,-10.924100000000001,1.0129695388641593
2080000.0,-12.878699999999998,2.2221,-18.8332,0.15080000000000027,-12.073725000000001,1.5803565030951081,2080000.0,-12.604575,2.097954356480379,-18.753666666666664,0.5002136165901748,-10.658733333333334,0.5369345666735278
2120000.0,-11.4016,0.35389999999999944,-18.8361,0.42440000000000033,-11.559549999999998,1.8278412643060669,2120000.0,-12.477250000000002,2.130717420846791,-18.9275,0.15838777309712626,-10.621633333333335,0.9474539930196549
2160000.0,-13.35275,0.90625,-19.09235,0.11425000000000018,-11.893675,0.7590821871675029,2160000.0,-11.841425,1.7971336850315283,-18.950233333333333,0.2701594631982301,-10.508366666666667,0.6579846215697012
2200000.0,-11.2947,0.6989999999999998,-18.0092,0.02990000000000137,-12.253799999999998,1.1647347874087046,2200000.0,-11.917275,1.8612953223158866,-18.770566666666667,0.05021675240616688,-10.246200000000002,0.6536254024031397
2240000.0,-10.6299,0.7825999999999995,-18.51935,0.2947500000000005,-11.0046,1.4392489412884768,2240000.0,-11.96945,1.841372131726773,-18.7156,0.3732460939737571,-10.328366666666666,0.9582772087913229
2280000.0,-11.67365,1.4490499999999997,-18.24555,0.07285000000000075,-11.708499999999999,0.7316265884998988,2280000.0,-12.138649999999998,2.4043918456649283,-18.65273333333333,0.2981974439118411,-10.102466666666666,0.680553002262784
2320000.0,-10.8469,0.8672000000000004,-17.96075,0.2259499999999992,-11.08155,0.9631234253718477,2320000.0,-11.892675,2.9602032957003135,-18.766966666666665,0.08701978063763642,-10.081766666666667,0.6321378242827181
2360000.0,-11.28105,0.49904999999999955,-17.582050000000002,0.13085000000000058,-11.921575,1.1007707172136263,2360000.0,-11.790625,2.1736368928307686,-18.522633333333335,0.24505455628401512,-9.761999999999999,1.0366350659706622
2400000.0,-11.581050000000001,1.0912499999999996,-18.4551,0.12929999999999886,-11.629574999999999,0.7691130000689105,2400000.0,-11.9168,2.5841645535840017,-18.587233333333334,0.2439007630619933,-9.681233333333333,0.5491136211103205
2440000.0,-11.260900000000001,0.3379000000000003,-18.60195,0.42965000000000053,-10.75655,1.1555476158514628,2440000.0,-12.09835,2.6914637936446413,-18.434366666666666,0.17992769535442776,-9.707166666666668,0.34952970625621493
2480000.0,-11.320699999999999,0.1003999999999996,-18.5238,0.024199999999998667,-11.149225000000001,0.9173939104196187,2480000.0,-12.402825,2.569275009000594,-18.8767,0.40462838094561115,-9.810033333333333,0.45312315716100404
2520000.0,-11.3711,0.8179999999999996,-18.474600000000002,0.40310000000000024,-11.11765,0.5632726537832274,2520000.0,-12.813175,3.0745774737148848,-19.007033333333332,0.18170579395152875,-10.0682,0.45051109494291763
2560000.0,-11.15215,0.9337499999999999,-18.8928,0.21189999999999998,-10.7042,0.7823586741386584,2560000.0,-12.756374999999998,3.5690492265973295,-18.719633333333334,0.23310593204711827,-9.986833333333335,0.14391410246702344
2600000.0,-11.246299999999998,0.18689999999999962,-18.8471,0.42050000000000054,-10.941225000000001,1.0405640426590765,2600000.0,-12.784749999999999,3.8338730948611213,-18.883300000000002,0.232870793932315,-9.775933333333333,0.5314662757148592
2640000.0,-10.7039,0.4358999999999993,-18.301949999999998,1.0542500000000015,-10.973149999999999,1.6068621417221831,2640000.0,-12.325899999999999,3.6190816715017635,-18.946866666666665,0.17171995289488612,-10.0578,0.1897898311290672
2680000.0,-10.2287,0.5197000000000003,-18.9791,0.254900000000001,-11.057125000000001,0.765511167047875,2680000.0,-12.297450000000001,3.6236791590178075,-18.377333333333336,0.16631985917368872,-9.738433333333335,0.3857728721526295
2720000.0,-11.257200000000001,0.4950999999999999,-18.3293,0.2922000000000011,-10.7914,1.285683961943992,2720000.0,-12.638975,3.741337890631505,-18.779400000000006,0.2819465197515301,-9.159,0.6167343890741513
2760000.0,-11.55745,0.2937500000000002,-18.31565,0.21834999999999916,-10.947350000000002,1.5099931332625327,2760000.0,-12.061449999999999,3.310794968357297,-18.672533333333334,0.27556335428032136,-9.345966666666667,0.7223399307497518
2800000.0,-10.91425,0.023249999999999993,-18.4348,0.6164000000000005,-10.565225,0.9485797208853877,2800000.0,-12.112125,3.810491837266549,-18.371499999999997,0.5066428525105238,-9.616033333333334,0.5508893072922079
2840000.0,-11.2469,1.3956999999999997,-18.728150000000003,0.3808500000000006,-10.875599999999999,1.2647717817851571,2840000.0,-11.9872,4.105759756731998,-18.437766666666665,0.12174619318711957,-8.975800000000001,0.8399068559469354
2880000.0,-10.37285,0.01034999999999986,-18.951349999999998,0.023249999999999105,-10.490025,1.5204496166841575,2880000.0,-12.492775,3.7787446476144697,-18.3087,0.17667203136508888,-9.562199999999999,1.0752477326954317
2920000.0,-10.22695,0.1609499999999997,-18.92675,0.016250000000001208,-10.337875,0.988860008734806,2920000.0,-12.166699999999999,3.975156366106873,-18.367866666666668,0.3068363009090605,-9.260133333333334,0.5300232720257562
2960000.0,-10.054499999999999,0.22750000000000004,-18.541,0.7555000000000014,-10.074025000000002,1.1488940973279476,2960000.0,-12.250175,3.6700604241449484,-18.817333333333334,0.23053447078955935,-9.276533333333333,0.5979539354238659
3000000.0,-11.0908,0.28690000000000015,-18.9033,0.22089999999999854,-10.066425000000002,0.8739333680979345,3000000.0,-12.0447,3.4931176010835934,-18.7978,0.42409674210805587,-9.381,0.4063871635111849
3040000.0,-10.8951,0.42989999999999995,-18.9,0.01720000000000077,-9.780874999999998,0.869758783729719,3040000.0,-12.263275,3.4879662027999925,-18.3703,0.5302880789407456,-9.2746,0.5240445019270781
3080000.0,-10.500399999999999,0.41330000000000044,-18.609750000000002,0.22855000000000025,-10.625574999999998,1.1229146369493095,3080000.0,-12.140025000000001,3.920189531371538,-18.807833333333335,0.5251298400290058,-9.082300000000002,0.48822077656186147
3120000.0,-10.5424,0.9783,-19.234,0.0,-10.244625000000001,1.314144184203164,3120000.0,-12.34795,3.4955638175407415,-18.254299999999997,1.518348978331397,-9.530866666666668,0.5606046279588577
3160000.0,-10.4756,0.1369000000000007,-18.6133,0.03790000000000049,-9.62815,1.0941016371891594,3160000.0,-11.991599999999998,3.1242481135466824,-18.452866666666665,0.6096867902638391,-9.594133333333334,0.9878575887017088
3200000.0,-10.21735,0.006050000000000111,-18.9508,0.02030000000000065,-9.654574999999998,0.6976387187326973,3200000.0,-12.018825000000001,3.5364813666800234,-18.386466666666667,0.25309339953639404,-9.3298,0.38322455558066765
3240000.0,-10.920100000000001,0.2153999999999998,-19.19555,0.07714999999999961,-9.875300000000001,1.1295658834260178,3240000.0,-11.352350000000001,3.3338340131896187,-18.44113333333333,0.2284364584639574,-9.353766666666667,1.4044619570813903
3280000.0,-10.74065,0.17304999999999993,-19.1965,0.2944999999999993,-10.108275,1.1175910396361448,3280000.0,-11.7384,2.881198801714315,-18.200633333333332,0.5801523095034802,-9.436833333333334,0.49883433678482453
3320000.0,-10.83945,0.59415,-18.60665,0.5523499999999988,-9.28965,0.7043931732349479,3320000.0,-10.962299999999999,3.385088819366487,-18.18386666666667,0.406237839476115,-9.106233333333334,0.3000697067164372
3360000.0,-11.44065,0.25234999999999985,-18.58245,0.7695500000000006,-9.595699999999999,0.5937637619457765,3360000.0,-10.994725,2.409153125451971,-18.4147,0.589623088647881,-8.878233333333332,0.7984275476766126
3400000.0,-11.21115,0.048650000000000304,-17.976200000000002,0.8785000000000007,-9.610349999999999,0.44252462360867534,3400000.0,-11.315800000000001,2.358702371008263,-18.634500000000003,0.4438694477734039,-8.941,1.1231616298051972
3440000.0,-10.68185,0.06585000000000019,-18.4508,0.013700000000000045,-9.434075,0.6743585410410401,3440000.0,-11.174525000000001,2.7082011736344476,-18.767733333333332,0.21860537250691844,-8.734766666666665,0.6578116717994255
3480000.0,-11.0234,0.1288999999999998,-18.2586,1.1805000000000003,-9.471,0.7973270219928585,3480000.0,-11.161024999999999,2.538073498123134,-18.391499999999997,0.5478069428792119,-9.245066666666666,0.2569151524444509
3520000.0,-10.9027,0.4375,-18.03885,0.5892499999999998,-9.62415,1.2337821211624038,3520000.0,-11.15,2.352059394445642,-18.16,0.7921480459274435,-8.669166666666667,0.8984693069635468
3560000.0,-10.67775,0.6925499999999998,-18.8582,0.23399999999999999,-9.024424999999999,0.9326819993304253,3560000.0,-11.405575,3.454486540989124,-18.153666666666663,0.9981445531028493,-9.1697,0.4375423179533612
3600000.0,-10.4881,0.013499999999999623,-18.090600000000002,0.7172000000000001,-9.7823,1.1862039011063825,3600000.0,-11.349525,3.046719108627345,-18.442566666666668,0.5853671573370768,-8.999233333333335,0.6567128359404049
3640000.0,-9.8869,0.10920000000000041,-17.21755,1.673449999999999,-9.223825000000001,0.6770720654959854,3640000.0,-11.371699999999997,2.8625412983920424,-17.807566666666663,0.8087369013193073,-8.5233,0.7645340149398197
3680000.0,-10.54555,0.018950000000000244,-18.142000000000003,0.35880000000000045,-9.631250000000001,0.7995100671661368,3680000.0,-11.028875,2.396246120888879,-18.287133333333333,0.9524188936013859,-8.761566666666665,0.5928079246733765
3720000.0,-10.6115,0.3783000000000003,-18.1791,0.20610000000000106,-9.588075,0.42315547600734166,3720000.0,-11.543624999999999,2.573975329500071,-17.963000000000005,0.9351557766846471,-8.217333333333332,0.12335207965638624
3760000.0,-10.07675,0.07554999999999978,-18.46095,0.04684999999999917,-8.921,0.5536205379138317,3760000.0,-10.710175,2.174425272773245,-17.92486666666667,0.7941360812578379,-8.3987,0.5051465199985713
3800000.0,-10.3379,0.37929999999999975,-18.2879,0.18440000000000012,-9.474525,1.3725220897584853,3800000.0,-11.10105,2.3281377337906792,-18.364733333333334,0.5847350758154413,-8.687733333333332,0.5995328255307535
3840000.0,-10.5125,0.05430000000000046,-18.659550000000003,0.3888499999999997,-9.31475,0.7605170395855707,3840000.0,-10.93155,2.167538117888587,-18.08566666666667,0.7657750946299804,-8.925133333333333,0.9745654153974933
3880000.0,-10.0549,0.28030000000000044,-18.677149999999997,0.2361500000000003,-9.341525,0.6014811816466081,3880000.0,-10.283100000000001,1.724285836223217,-18.464066666666668,0.6602089887974028,-8.611333333333333,0.7775792964207721
3920000.0,-10.17035,0.28125,-18.241,0.06170000000000009,-9.677725,0.3574253864445001,3920000.0,-10.536925000000002,1.6833292790999035,-18.008200000000002,0.6819025492448799,-8.734,0.4700184889980394
3960000.0,-9.424999999999999,0.6391,-17.893,0.4089999999999989,-9.849625,0.5743370324774467,3960000.0,-10.534875,1.5333104991732762,-17.907566666666664,1.0806452712255867,-8.3005,0.35419422167317527
4000000.0,-10.0293,0.32730000000000015,-18.05505,0.2472499999999993,-8.936425,0.9561370492115658,4000000.0,-10.902650000000001,1.7724762530144096,-18.287233333333333,1.0413973347809615,-8.247933333333334,0.2711672587578051
}\dataSotaLj

%% file: data/traj_comm.tex
\pgfplotstableread[col sep=comma]{
0.0,0.6995,0.45847546281125834,0.7335,0.44212865774568194,0.6845,0.4647146974219786,0.0
1.0,0.8245,0.3803942034258698,0.8055,0.3958152978347431,0.81,0.3923009049186594,0.0
2.0,0.769,0.42147241902643934,0.771,0.42018924307982314,0.7535,0.4309730269982055,0.0
3.0,0.6405,0.47985388401054463,0.6305,0.4826694003145448,0.6175,0.48599768517967956,0.0
4.0,0.5725,0.4947158275212163,0.579,0.49371955602345224,0.577,0.4940354238311182,0.0
5.0,0.5675,0.4954227992331458,0.5775,0.493957234991051,0.567,0.4954906659060347,0.0
6.0,0.6135,0.48694737908730434,0.598,0.4903019477832021,0.587,0.492372826220129,0.0
7.0,0.6075,0.48830702431974404,0.6195,0.4855097836295304,0.608,0.4881966816765523,0.0
8.0,0.6175,0.4859976851796796,0.6285,0.48320570153920034,0.624,0.48438001610305875,0.0
9.0,0.604,0.4890644129355555,0.6025,0.48938098655343276,0.6205,0.48526255779732347,0.0
10.0,0.6035,0.4891704713083068,0.603,0.4892759957324638,0.613,0.4870636508712128,0.0
11.0,0.569,0.49521611443894126,0.5775,0.49395723499105104,0.5915,0.491556456574413,0.0
12.0,0.5505,0.49744321283942483,0.5515,0.4973406780065405,0.5735,0.4945682460490181,0.0
13.0,0.52,0.49959983987186446,0.529,0.49915829152683877,0.5395,0.4984373079936843,0.0
14.0,0.4895,0.4998897378422605,0.495,0.49997499937496426,0.52,0.4995998398718645,0.0
15.0,0.454,0.497879503494571,0.473,0.49927046778274503,0.4795,0.4995795732413483,0.0
16.0,0.419,0.49339537898119873,0.446,0.4970754469896888,0.464,0.49870231601628906,0.0
17.0,0.396,0.48906441293555547,0.426,0.4944936804449518,0.437,0.4960151207372534,0.0
18.0,0.3785,0.48501314415179175,0.39,0.48774993593028887,0.421,0.4937195560234521,0.0
19.0,0.3485,0.4764952780458615,0.3595,0.4798538840105446,0.383,0.4861182983595639,0.0
20.0,0.3045,0.46019533895945114,0.3085,0.4618741711765268,0.33,0.47021271782035096,0.0
21.0,0.271,0.444476096095166,0.274,0.44600896851969957,0.286,0.45188936699152016,0.0
22.0,0.2235,0.41659062639477806,0.228,0.4195426080864742,0.245,0.4300872004605528,0.0
23.0,0.191,0.39308904843559206,0.1875,0.3903123748998999,0.2015,0.40112061776977026,0.0
24.0,0.1575,0.36427153333743306,0.1585,0.3652091866314462,0.177,0.38166870450693346,0.0
25.0,0.124,0.3295815528818369,0.124,0.3295815528818369,0.1395,0.3464675309462678,0.0
26.0,0.0935,0.29113184298527195,0.0975,0.2966374049239196,0.1015,0.3019896521405944,0.0
27.0,0.081,0.2728351150420381,0.082,0.27436472076416973,0.086,0.28036404905051604,0.0
28.0,0.067,0.2500219990320915,0.065,0.2465258607124251,0.075,0.2633913438213165,0.0
29.0,0.0475,0.21270578271406063,0.0495,0.21690954335851453,0.0535,0.22502833154961246,0.0
30.0,0.0455,0.20839805661282315,0.047,0.21163884331567895,0.049,0.21586801523153235,0.0
31.0,0.0355,0.1850398605706378,0.0375,0.18998355191963628,0.038,0.19119623427254026,0.0
32.0,0.032,0.17600000000000207,0.035,0.18377975949488906,0.0315,0.17466467874187178,0.0
33.0,0.026,0.1591351626762578,0.0265,0.1606167799453133,0.0265,0.1606167799453127,0.0
34.0,0.0205,0.14170303454760644,0.024,0.15304901175767266,0.023,0.14990330216509642,0.0
35.0,0.0215,0.14504395885385973,0.0205,0.1417030345476061,0.023,0.1499033021650964,0.0
36.0,0.0175,0.1311249404194359,0.018,0.13295111883696015,0.0215,0.14504395885386004,0.0
37.0,0.0155,0.12353036064061435,0.0155,0.12353036064061432,0.0175,0.1311249404194366,0.0
38.0,0.0135,0.11540255629751213,0.0115,0.10661965109678447,0.0135,0.11540255629751374,0.0
39.0,0.0085,0.09180277773575347,0.0085,0.09180277773575342,0.0125,0.11110243021644435,0.0
}\dataTrajComm

%% file: data/dummy.tex
\pgfplotstableread[col sep=comma]{

0,0,0,0,0
0,0,0,0,0
0,0,0,0,0
0,0,0,0,0
0,0,0,0,0
0,0,0,0,0
0,0,0,0,0
0,0,0,0,0
0,0,0,0,0
0,0,0,0,0
0,0,0,0,0
0,0,0,0,0
0,0,0,0,0
0,0,0,0,0
0,0,0,0,0
0,0,0,0,0
0,0,0,0,0
0,0,0,0,0
0,0,0,0,0
0,0,0,0,0
0,0,0,0,0
0,0,0,0,0
0,0,0,0,0
0,0,0,0,0
0,0,0,0,0
0,0,0,0,0
0,0,0,0,0
0,0,0,0,0
0,0,0,0,0
0,0,0,0,0
0,0,0,0,0
0,0,0,0,0
0,0,0,0,0
0,0,0,0,0
0,0,0,0,0
0,0,0,0,0
0,0,0,0,0
0,0,0,0,0
0,0,0,0,0
0,0,0,0,0
0,0,0,0,0
0,0,0,0,0
0,0,0,0,0
0,0,0,0,0
0,0,0,0,0
0,0,0,0,0
0,0,0,0,0
0,0,0,0,0
0,0,0,0,0
0,0,0,0,0
0,0,0,0,0
0,0,0,0,0
0,0,0,0,0
0,0,0,0,0
0,0,0,0,0
0,0,0,0,0
0,0,0,0,0
0,0,0,0,0
0,0,0,0,0
0,0,0,0,0
0,0,0,0,0
0,0,0,0,0
0,0,0,0,0
0,0,0,0,0
0,0,0,0,0
0,0,0,0,0
0,0,0,0,0
0,0,0,0,0
0,0,0,0,0
0,0,0,0,0
0,0,0,0,0
0,0,0,0,0
0,0,0,0,0
0,0,0,0,0
0,0,0,0,0
0,0,0,0,0
0,0,0,0,0
0,0,0,0,0
0,0,0,0,0
0,0,0,0,0
0,0,0,0,0
0,0,0,0,0
0,0,0,0,0
0,0,0,0,0
0,0,0,0,0
0,0,0,0,0
0,0,0,0,0
0,0,0,0,0
0,0,0,0,0
0,0,0,0,0
0,0,0,0,0
0,0,0,0,0
0,0,0,0,0
0,0,0,0,0
0,0,0,0,0
0,0,0,0,0
0,0,0,0,0
0,0,0,0,0
0,0,0,0,0
0,0,0,0,0
0,0,0,0,0
}\dataDummy

%% file: data/bw_limit_steps.tex
\pgfplotstableread[col sep=comma]{

0.0,44.893233333333335,0.15099086800936892,1.0,0.0,45.0,0.0,1.0,0.0,45.0,0.0,0.107,0.13961002829309938,45.0,0.0,0.3333333333333333,0.4714045207910317
40000.0,44.90103333333334,0.1399600022228577,1.0,0.0,44.8164,0.18359999999999843,1.0,0.0,44.83593333333334,0.23202530513334432,0.4907333333333333,0.171655397688379,45.0,0.0,0.44693333333333335,0.08371978393559207
80000.0,44.78906666666666,0.22068120798009871,1.0,0.0,44.886700000000005,0.011700000000001154,1.0,0.0,44.643233333333335,0.5045442586026427,0.48269999999999996,0.20376535524961056,45.0,0.0,0.4733,0.08402289370562446
120000.0,44.7682,0.19142211645122653,1.0,0.0,44.41015,0.5898499999999984,1.0,0.0,44.05466666666666,1.0186511648035146,0.44980000000000003,0.08351099727979941,45.0,0.0,0.6003333333333333,0.15140476287826027
160000.0,44.625,0.3724341910548332,1.0,0.0,44.51175,0.2929499999999976,1.0,0.0,43.770833333333336,0.9325861616434648,0.5124333333333334,0.04524395601133434,44.86976666666666,0.09405829161866774,0.5144333333333333,0.05812000994111714
200000.0,44.604166666666664,0.2157789815734819,1.0,0.0,44.84765,0.15234999999999843,1.0,0.0,41.89066666666667,1.328697009186905,0.4764,0.11343229992672575,44.091133333333325,0.42121988979101443,0.5453,0.06409560567360811
240000.0,44.44793333333333,0.7642535763010143,1.0,0.0,44.328149999999994,0.10154999999999959,1.0,0.0,42.5651,1.2031834855914534,0.49763333333333337,0.029680109313964628,44.28906666666666,0.5390464069900546,0.49720000000000003,0.1608838918806562
280000.0,43.890633333333334,0.6086405306546352,1.0,0.0,44.90625,0.09375,1.0,0.0,42.8099,0.3161503123515764,0.6234666666666667,0.10675911618634208,43.359399999999994,1.0006783832314303,0.5030333333333333,0.10277160221686835
320000.0,43.34633333333334,0.9331523145886846,1.0,0.0,44.8125,0.1875,1.0,0.0,41.911433333333335,0.46020137862558486,0.48216666666666663,0.08616667311411966,43.33596666666667,0.6222617098581237,0.6182,0.06588914933431757
360000.0,44.0052,0.39226328743162664,1.0,0.0,44.457049999999995,0.5429500000000012,1.0,0.0,40.8802,0.7064869472726748,0.5236333333333333,0.18437144271521252,42.872366666666665,1.1131963239049782,0.49720000000000003,0.06984110537498675
400000.0,42.395833333333336,1.8835453420492843,1.0,0.0,44.480450000000005,0.347649999999998,1.0,0.0,41.4141,1.7860862707794007,0.6278,0.12507920157510868,42.625,1.8416830400478816,0.4968333333333333,0.03150347423521556
440000.0,41.96616666666667,1.199372964881603,1.0,0.0,44.66015,0.15234999999999843,1.0,0.0,40.312533333333334,0.4106359404089657,0.5582333333333334,0.046692635631566365,41.934866666666665,0.6417085856437412,0.6177,0.10014392975446221
480000.0,42.57816666666667,0.27739546179096597,1.0,0.0,44.33985,0.15234999999999843,1.0,0.0,39.46093333333334,1.6211574966328492,0.5072,0.09722976224730094,40.47656666666666,1.0164623794754486,0.5461666666666667,0.07603605869731953
520000.0,41.5,0.816496580927726,1.0,0.0,44.0625,0.04690000000000083,1.0,0.0,38.4427,1.3196324968212414,0.5936,0.11851838113417963,40.427099999999996,0.7080753820509983,0.6328333333333332,0.022808818957197753
560000.0,41.958333333333336,0.575734605139864,1.0,0.0,44.4453,0.46089999999999876,1.0,0.0,38.18226666666667,0.9881516291653917,0.5623333333333334,0.16449478884011964,39.638,0.8169291809371632,0.5657666666666666,0.04850136309653805
600000.0,41.7396,1.3129920131770296,1.0,0.0,44.96875,0.01565000000000083,1.0,0.0,36.4271,1.347765093280972,0.5912333333333333,0.11307520604545553,41.38803333333333,1.2770944896739458,0.5990333333333334,0.03392052803572231
640000.0,41.182300000000005,0.9100506396166455,1.0,0.0,44.68355,0.15235000000000198,1.0,0.0,35.54683333333333,0.988267798839069,0.6156333333333334,0.07370266088970076,40.7526,1.0559491275624957,0.5725333333333333,0.06202329311548113
680000.0,41.0469,1.9284686688320007,1.0,0.0,44.074200000000005,0.8632999999999988,1.0,0.0,35.90106666666667,1.101202507362848,0.6491666666666668,0.14348594201368842,39.22656666666666,0.6041848962767016,0.6246666666666667,0.04975046621780429
720000.0,39.79946666666667,0.7538166060492145,1.0,0.0,43.6797,0.28910000000000124,1.0,0.0,35.2448,2.5187605933606867,0.6144333333333333,0.12953921757093056,39.31513333333333,0.8140088956652867,0.6212333333333334,0.03939630552334682
760000.0,39.96353333333334,0.89654672431998,1.0,0.0,43.1758,0.04299999999999926,1.0,0.0,33.958333333333336,0.9761323999449161,0.6470333333333333,0.15489085045784842,37.59113333333334,0.4280378825394904,0.5986666666666667,0.08973829852533546
800000.0,39.9948,1.2940906021862093,1.0,0.0,42.82425,0.7304499999999976,1.0,0.0,34.14843333333334,2.956610894626182,0.6533666666666667,0.07383451466323566,38.21353333333334,1.0343103896907473,0.5662999999999999,0.09686334015852784
840000.0,40.169266666666665,1.879405917363842,1.0,0.0,43.699200000000005,0.6992000000000012,1.0,0.0,32.38023333333333,0.8838965261210648,0.6361,0.11264066761165792,35.674499999999995,0.8398964340917293,0.5825,0.10388147091757993
880000.0,39.624966666666666,0.5113852841927395,1.0,0.0,43.417950000000005,0.29295000000000115,1.0,0.0,32.52343333333334,0.7655422580686774,0.6797,0.07431769641209289,37.177099999999996,0.9923480572191723,0.5827333333333333,0.10833839988152347
920000.0,38.7552,1.228035213935931,1.0,0.0,43.875,0.28119999999999834,1.0,0.0,31.367166666666666,1.519023230749142,0.6805666666666667,0.08999412326492336,34.205733333333335,0.6118570711821135,0.5881666666666666,0.09261174631522481
960000.0,37.89843333333334,2.3122359226995464,1.0,0.0,44.4297,0.0,1.0,0.0,32.28386666666666,2.3574296940146953,0.6343,0.1220705806763721,34.49216666666667,1.1344456188915453,0.5583666666666666,0.13729846644769522
1000000.0,38.90103333333334,1.7348823770567936,1.0,0.0,44.015649999999994,0.5234500000000004,1.0,0.0,30.625033333333334,2.5494141867408575,0.6718999999999999,0.09684671737682528,33.8099,0.8193187698732803,0.5808666666666666,0.07199635793257568
1040000.0,37.77863333333334,2.8153203492478243,1.0,0.0,44.015649999999994,0.3359500000000004,1.0,0.0,30.526033333333334,3.393128827825762,0.6735333333333333,0.08481919332058961,32.61976666666667,1.5082406093046163,0.5454,0.08960773776112568
1080000.0,36.080733333333335,2.022071812660359,1.0,0.0,43.28905,0.7578500000000012,1.0,0.0,29.54686666666667,3.5099172092167024,0.6832666666666666,0.07035340945698525,31.895866666666667,1.2874883697425088,0.5507,0.09641441800892646
1120000.0,35.627633333333335,0.9716085848849949,1.0,0.0,44.13675,0.2070500000000024,1.0,0.0,26.637999999999995,3.415630562575526,0.6851666666666666,0.093667722413979,30.885433333333335,0.7174406100081652,0.5365666666666666,0.07317423651046104
1160000.0,35.46353333333334,3.4946637800445925,1.0,0.0,42.35155,1.53125,1.0,0.0,27.091133333333335,2.4116716834778495,0.7206,0.07367554999229166,30.4896,1.0915951294626898,0.5617333333333333,0.10635814757485933
1200000.0,34.69533333333334,3.1696556377976175,1.0,0.0,44.53905,0.20314999999999728,1.0,0.0,24.54946666666667,1.3348235572122955,0.6866,0.06085195148883886,29.6771,1.3748071282910928,0.5563666666666667,0.10319858956831189
1240000.0,34.916666666666664,3.7432859542860935,1.0,0.0,42.9922,0.6327999999999996,1.0,0.0,24.854166666666668,2.0452908682032374,0.6729333333333333,0.07385446650150702,29.265633333333337,0.16278120967182455,0.5239999999999999,0.08445795798305014
1280000.0,32.27603333333334,2.8954633496020783,1.0,0.0,43.0078,0.7968999999999973,1.0,0.0,25.148433333333333,2.1437469631983666,0.6846000000000001,0.08586116700814171,27.567700000000002,1.3082649604214982,0.5377666666666667,0.13522382268750657
1320000.0,30.924466666666664,2.7711831664872353,1.0,0.0,43.15235,0.29295000000000115,1.0,0.0,23.458333333333332,0.4235577043200706,0.6844,0.07136306234086837,26.666666666666668,1.1214133027370223,0.5067,0.13145845985202576
1360000.0,31.067733333333337,4.03364472103773,1.0,0.0,43.886700000000005,0.6992000000000012,1.0,0.0,22.721333333333334,0.8463998516593029,0.6758333333333334,0.07149826726727176,25.968766666666667,1.275437144232866,0.4808666666666667,0.10382884420472419
1400000.0,29.8229,4.1833253885714745,1.0,0.0,43.7461,0.6445000000000007,1.0,0.0,22.096333333333334,1.1397167903572463,0.6798000000000001,0.0627094889151554,24.51823333333333,1.4627783526183618,0.5081333333333333,0.10946528622759313
1440000.0,30.01563333333333,4.673076064958593,1.0,0.0,43.980450000000005,1.0195499999999988,1.0,0.0,23.28643333333333,2.1024383373174635,0.6809666666666666,0.0806106830526961,25.09113333333333,0.6346169150667901,0.5455333333333333,0.11639600031310735
1480000.0,28.1771,2.26336694476761,1.0,0.0,43.1797,0.07029999999999959,1.0,0.0,22.960966666666668,0.7957944472175097,0.6753,0.07389591779433198,24.60936666666667,0.9221963071325369,0.4839,0.11309335376876337
1520000.0,27.645833333333332,4.100070902096967,1.0,0.0,43.7461,0.13670000000000115,1.0,0.0,22.20053333333333,1.0042180683939572,0.7080333333333333,0.07060284854184158,24.15363333333333,1.1491664873096308,0.5415666666666666,0.11807278357954566
1560000.0,28.718733333333333,3.992514350075095,1.0,0.0,43.6172,0.8281000000000027,1.0,0.0,21.76563333333333,0.5399124640984767,0.6941666666666667,0.0796617990144726,23.067733333333337,1.3864725537212124,0.5165,0.09563601134858493
1600000.0,27.710966666666668,3.3011743590560148,1.0,0.0,43.078149999999994,0.023450000000000415,1.0,0.0,21.804666666666666,0.8573171577012138,0.6943333333333334,0.07378822549853209,23.0625,0.8683283979386303,0.5367666666666667,0.09410410311045009
1640000.0,27.625,3.9923810497830328,1.0,0.0,44.171899999999994,0.5547000000000004,1.0,0.0,20.8854,0.013257450735340535,0.6802,0.05025740409797015,23.317700000000002,0.7718692505858745,0.5651,0.10500688866291898
1680000.0,26.718766666666667,3.4348841846883604,1.0,0.0,43.28125,0.32815000000000083,1.0,0.0,20.984366666666663,0.40339969701966305,0.6542333333333333,0.059020919078652825,22.7526,0.8307970911520259,0.5321333333333333,0.13337866729312034
1720000.0,25.440133333333335,3.3718290756732543,1.0,0.0,43.6172,0.8047000000000004,1.0,0.0,21.039033333333332,0.8456716160675037,0.6855000000000001,0.05706353184536224,23.374966666666666,0.7188686497236861,0.48163333333333336,0.10146350192173649
1760000.0,24.960966666666668,1.7905183762127546,1.0,0.0,42.863299999999995,0.441399999999998,1.0,0.0,21.8021,0.13279023558480005,0.6688999999999999,0.054533903827496746,22.119799999999998,1.4883521648678,0.5052666666666666,0.08746356701824796
1800000.0,25.6797,2.598117056639289,1.0,0.0,43.8203,0.5702999999999996,1.0,0.0,21.786466666666666,0.5436818084954554,0.6677,0.02894627206855256,22.447933333333335,0.8035054628867744,0.5433666666666667,0.11437366635530904
1840000.0,23.9974,1.29690312925317,1.0,0.0,43.265649999999994,0.4765499999999996,1.0,0.0,21.088533333333334,0.9130904202517713,0.6744666666666667,0.04246483512533898,21.739566666666665,1.1159860821513667,0.5507333333333334,0.09020888106069282
1880000.0,25.645833333333332,3.498737433544976,1.0,0.0,43.8047,1.0781000000000027,1.0,0.0,20.984400000000004,0.5516797863495342,0.6686,0.02126703238974981,23.1901,1.9312508709814655,0.5218666666666666,0.09518124231637709
1920000.0,23.76823333333333,2.3119654788849147,1.0,0.0,43.480450000000005,0.6679500000000012,1.0,0.0,20.9297,0.23385358667337133,0.6439666666666666,0.0039718453589679984,21.48696666666667,1.1247164096883353,0.5739,0.10213729322175452
1960000.0,22.84116666666667,1.274396788898793,1.0,0.0,43.0078,0.24220000000000041,1.0,0.0,20.6224,0.6899033120662633,0.6398333333333334,0.050935863811485765,21.6901,0.7055320687254417,0.5598,0.11305193496796066
2000000.0,22.960933333333333,1.6620325354483558,1.0,0.0,42.9844,0.25,1.0,0.0,21.04426666666667,0.2606102624396987,0.6305666666666666,0.03628492187611204,21.52603333333333,0.5811357978610119,0.5821333333333333,0.07712899728521189
2040000.0,22.098966666666666,1.087643818331882,1.0,0.0,43.324200000000005,1.6757999999999988,1.0,0.0,20.361966666666664,0.2862841517715492,0.6770666666666667,0.03828945314603256,22.02606666666667,0.3067405019809995,0.5772,0.07398436771823265
2080000.0,22.893199999999997,1.751001047020436,1.0,0.0,43.292950000000005,1.3554500000000012,1.0,0.0,20.1328,0.4651681918044996,0.6503,0.019617509186098732,21.351566666666667,0.36483167199256294,0.5762333333333334,0.11627671974886271
2120000.0,22.445300000000003,1.5587558265060844,1.0,0.0,43.54685,0.65625,1.0,0.0,20.549466666666667,0.803428944517739,0.6802666666666667,0.032954750121273184,21.46356666666667,0.6496809285248336,0.5834,0.0917751963586386
2160000.0,20.9375,1.185137370378078,1.0,0.0,44.2578,0.22660000000000124,1.0,0.0,20.156233333333333,0.7707841908659574,0.6603333333333333,0.02772920642371305,21.062466666666666,0.574345925572927,0.6012666666666667,0.09505652119777068
2200000.0,21.73696666666667,0.6118467146452795,1.0,0.0,43.88285,0.59375,1.0,0.0,19.7604,0.43195222729680055,0.6638666666666667,0.05915754859318932,20.90103333333333,0.3352675979307013,0.6406666666666667,0.12821443843117758
2240000.0,21.619799999999998,0.6646790854740848,1.0,0.0,43.0195,0.8789000000000016,1.0,0.0,20.0495,0.8473271269114429,0.6687,0.019337700656144916,20.838533333333334,0.3147017247419452,0.5729666666666666,0.08333747989683606
2280000.0,20.78386666666667,0.7168974465632371,1.0,0.0,42.99215,0.40625,1.0,0.0,20.487000000000002,0.11455691453014487,0.6648000000000001,0.032254405383864475,20.648433333333333,0.4783500972672172,0.5898,0.05284814723967782
2320000.0,21.351566666666667,0.45506402724110073,1.0,0.0,43.47655,0.44535000000000124,1.0,0.0,20.554666666666666,1.419842297189688,0.614,0.047305602205235664,20.367166666666666,0.5513487845476962,0.586,0.0721295131459146
2360000.0,21.023433333333333,0.6324067379639653,1.0,0.0,43.453149999999994,0.10154999999999959,1.0,0.0,19.6927,0.7934933017990758,0.6482333333333333,0.04027077793581289,20.5599,0.3904462660426745,0.5934333333333334,0.07485596539725847
2400000.0,21.2604,0.4543801345422863,1.0,0.0,43.828100000000006,0.4922000000000004,1.0,0.0,20.117166666666666,0.4425261147347372,0.6742333333333334,0.03187102480658919,20.4844,0.9863108941910765,0.5673,0.106696422932855
2440000.0,21.182266666666667,0.7259549358526934,1.0,0.0,43.1719,1.2030999999999992,1.0,0.0,19.557299999999998,0.6740630880464132,0.6756666666666667,0.017721988852521332,20.929666666666666,0.634727785922613,0.6066666666666666,0.09692589379979373
2480000.0,20.958333333333332,0.4671049156476755,1.0,0.0,43.207049999999995,0.8945499999999988,1.0,0.0,19.2266,0.1823199385695374,0.6764000000000001,0.01774711244118322,21.3698,0.6611437413049198,0.5946666666666667,0.09252136089694217
2520000.0,21.041666666666668,0.254840425015777,1.0,0.0,43.46095,0.039049999999999585,1.0,0.0,19.161466666666666,0.3274199783492489,0.6861,0.06349887138104633,20.606766666666665,0.6166422968157659,0.6118333333333333,0.03514506445513444
2560000.0,20.726566666666667,0.4530415163708021,1.0,0.0,42.86325,0.3320500000000024,1.0,0.0,19.450533333333336,0.20503323849778268,0.6723666666666666,0.03634742845863455,20.40886666666667,0.4807769152343125,0.5481,0.082256468843895
2600000.0,20.765633333333334,0.8686537719687607,1.0,0.0,43.7344,0.8125,1.0,0.0,20.3021,1.3780672625093457,0.6957999999999999,0.023076105968439866,20.309866666666668,0.5584523932758779,0.6051333333333333,0.08028965617617807
2640000.0,20.489566666666665,1.2954479851473082,1.0,0.0,42.476600000000005,0.3672000000000004,1.0,0.0,19.90103333333333,0.36097118382989446,0.6830999999999999,0.055278265771156986,20.4974,0.6655414988313406,0.6269666666666667,0.0872877361883609
2680000.0,21.3203,0.2917248475304529,1.0,0.0,42.7461,0.5585999999999984,1.0,0.0,20.1198,0.3026116433098154,0.6914333333333333,0.04583858151770794,20.51823333333333,1.0093922538945015,0.5885000000000001,0.053434321055541294
2720000.0,20.5365,0.40188088616736495,1.0,0.0,44.14845,0.46875,1.0,0.0,19.73436666666667,0.7863798672114856,0.6953999999999999,0.04729545432702811,20.3021,0.5981837566054984,0.6223,0.07900649762308581
2760000.0,21.145866666666667,0.9273811202640596,1.0,0.0,43.16015,0.8242499999999993,1.0,0.0,20.067733333333333,0.32022476828350105,0.6571666666666667,0.007286669716376293,19.854166666666668,0.5364448423546347,0.6257,0.07649710234163559
2800000.0,20.6354,1.0336152701400405,1.0,0.0,42.52345,0.44535000000000124,1.0,0.0,19.999966666666666,0.0710431949982235,0.7264333333333334,0.014693384754900951,21.3073,0.5007077457626022,0.6104333333333333,0.062239876463744875
2840000.0,20.679733333333335,0.8579738198544029,1.0,0.0,42.6836,0.8242000000000012,1.0,0.0,18.981800000000003,0.6539150760356174,0.6818333333333334,0.02899498040849293,20.523433333333333,0.22378855099301947,0.5855,0.08984835372262905
2880000.0,20.541700000000002,0.7106033633469514,1.0,0.0,41.7422,1.2031000000000027,1.0,0.0,19.450533333333336,0.34271020151466514,0.6998666666666667,0.033490330279383966,20.588533333333334,0.7849951903603533,0.5972666666666667,0.06350434805761115
2920000.0,20.835933333333333,0.39516305776504856,1.0,0.0,43.03125,0.04684999999999917,1.0,0.0,19.51823333333333,0.5585250178421337,0.6797333333333334,0.01306403034629395,19.908833333333334,0.9645571568802374,0.6035666666666667,0.05758850193880333
2960000.0,20.7526,0.4259051850666619,1.0,0.0,43.542950000000005,0.285149999999998,1.0,0.0,20.112,0.4580889506053016,0.6814333333333332,0.014121457274500955,19.778666666666666,0.5195587957317457,0.5754,0.039430276015603394
3000000.0,20.184900000000003,0.7509153125796982,1.0,0.0,43.2148,0.19139999999999802,1.0,0.0,19.8203,1.2651459230723803,0.6813333333333333,0.03593784758286014,20.0625,0.6622172654549767,0.5987333333333335,0.059576524086440444
3040000.0,20.651033333333334,0.3912642272883575,1.0,0.0,44.2539,0.26950000000000074,1.0,0.0,18.7474,0.4338764878011561,0.6679666666666667,0.036618786920862845,20.182299999999998,0.8245646608968873,0.6152000000000001,0.07983612382040271
3080000.0,20.320333333333334,0.7305457792332769,1.0,0.0,44.0703,0.13279999999999959,1.0,0.0,19.007833333333334,0.49460810300232094,0.6917333333333334,0.039492812285556775,19.820333333333334,0.2180237958470493,0.5827000000000001,0.011203868379567244
3120000.0,19.971366666666665,0.66673481151721,1.0,0.0,43.7617,0.23049999999999926,1.0,0.0,19.648433333333333,1.0164153656628554,0.6674000000000001,0.05357854048030796,20.085933333333333,0.33973143641542736,0.5983666666666666,0.074851868528596
3160000.0,20.09633333333333,0.3431713889913057,1.0,0.0,44.375,0.07030000000000314,1.0,0.0,19.3802,0.3218798222939734,0.6822333333333334,0.03405342208289141,19.888,1.1048957597891305,0.5755666666666667,0.07781560826009706
3200000.0,20.700533333333333,0.681708864219585,1.0,0.0,43.6875,0.6171999999999969,1.0,0.0,18.875,0.36925680314202286,0.6763666666666667,0.044581112094199196,20.218766666666667,0.28406656888052795,0.5655,0.01925062769539387
3240000.0,19.9375,0.8739659070390938,1.0,0.0,43.7578,0.14060000000000272,1.0,0.0,19.132833333333334,0.315326078140639,0.6570333333333334,0.026486516485856634,20.3177,0.5688410908739511,0.5708333333333333,0.02029931580675127
3280000.0,19.3151,0.5965171637653573,1.0,0.0,42.921850000000006,0.3515499999999996,1.0,0.0,19.1432,0.3572361777125417,0.6632000000000001,0.036187659038222804,19.58073333333333,0.649260647882565,0.5694333333333333,0.04296799842777048
3320000.0,20.3854,0.7916013685351138,1.0,0.0,43.6875,0.7344000000000008,1.0,0.0,18.682266666666667,0.48586379664355434,0.6660333333333334,0.027682645987854686,20.3073,0.24358189587898307,0.5708666666666667,0.02584625483293251
3360000.0,20.612000000000002,0.14690861104782146,1.0,0.0,43.703100000000006,0.4452999999999996,1.0,0.0,19.40363333333333,0.4953621054900706,0.6280666666666667,0.041325563785896774,20.364566666666665,0.45668123334431837,0.5714333333333333,0.051121902242472254
3400000.0,20.28386666666667,0.31719169107794937,1.0,0.0,43.488299999999995,0.5117000000000012,1.0,0.0,18.937466666666666,0.2906963402284623,0.6886,0.0062631195634976185,19.721366666666665,0.5943654898760148,0.6240666666666667,0.05771656800453592
3440000.0,20.10676666666667,0.3941456358025837,1.0,0.0,42.457049999999995,0.23045000000000115,1.0,0.0,19.093733333333333,0.8037944278372556,0.6865666666666667,0.04531595255046014,19.6198,1.0585406211698571,0.5644999999999999,0.024978791003569423
3480000.0,20.614566666666665,0.33810247296082035,1.0,0.0,43.58985,0.5195499999999988,1.0,0.0,19.174466666666664,0.4748681524615244,0.6617333333333333,0.03238850962232681,19.76563333333333,0.9549889225651901,0.5641999999999999,0.04432245480566255
3520000.0,19.822933333333335,0.5618268554951377,1.0,0.0,42.257850000000005,0.16404999999999959,1.0,0.0,19.257833333333334,0.1053957726329139,0.6598,0.03269342441531632,19.95573333333333,0.45953894527256656,0.5734333333333334,0.03353011913025197
3560000.0,20.1875,0.6411368392680834,1.0,0.0,43.73825,0.0429499999999976,1.0,0.0,19.1172,0.29732006323152943,0.6424333333333333,0.04482486908947853,19.9349,0.5888061537269007,0.5957,0.025515616133393027
3600000.0,19.2422,0.35223532853288136,1.0,0.0,42.855450000000005,0.285149999999998,1.0,0.0,18.890599999999996,0.2801691036975112,0.6603,0.013305888420795785,19.08073333333333,0.8733724075228286,0.5974666666666666,0.023374820260746876
3640000.0,20.04426666666667,0.7199779826880502,1.0,0.0,43.66405,0.53125,1.0,0.0,19.049500000000002,0.6174939729800343,0.6419666666666667,0.027399553929864547,19.984333333333332,0.871150049583244,0.5621666666666667,0.030605700267905794
3680000.0,19.4844,0.412352009816855,1.0,0.0,43.320350000000005,0.14845000000000041,1.0,0.0,19.0755,0.18336653638727723,0.6823,0.03882275964774614,20.0,0.4783136209643203,0.5881333333333334,0.015390545438316623
3720000.0,19.734366666666666,0.7765600484758974,1.0,0.0,42.480450000000005,0.7929500000000012,1.0,0.0,18.7474,0.27252659809029023,0.6697666666666667,0.04269335103060219,19.997400000000003,0.9443189397655851,0.5618,0.024476927911811153
3760000.0,19.640600000000003,0.6804004164215856,1.0,0.0,42.93355,0.652350000000002,1.0,0.0,18.518233333333335,0.6835657799769946,0.6846333333333333,0.04718822828724225,19.729166666666668,0.15286251629777914,0.5688000000000001,0.006011655346075648
3800000.0,20.135433333333335,0.9716085848849948,1.0,0.0,42.82035,0.34375,1.0,0.0,19.09636666666667,0.22612023841802079,0.6712333333333333,0.02341457286012757,19.411466666666666,0.4411789533612043,0.5779,0.04677456573822999
3840000.0,20.080699999999997,0.4668296548706669,1.0,0.0,42.55465,0.28125,1.0,0.0,19.015633333333334,0.2008265807993449,0.6231,0.06627805066535979,19.28643333333333,0.5478501031811124,0.5848666666666666,0.03467874789486431
3880000.0,20.23436666666667,0.3156671805698032,1.0,0.0,43.58985,0.11325000000000074,1.0,0.0,18.73176666666667,0.20829713285486073,0.6659666666666667,0.0041354833118055255,19.453133333333334,0.8138757986053864,0.5964333333333333,0.04391843449952293
3920000.0,19.710933333333333,0.6744514083477198,1.0,0.0,42.88675,0.0429499999999976,1.0,0.0,18.75,0.3181400425389222,0.6562666666666667,0.06073825446582705,19.5677,0.4663733268530696,0.5508000000000001,0.026501320721805565
3960000.0,19.364566666666665,0.4564786109142717,1.0,0.0,43.21875,1.0546500000000023,1.0,0.0,19.499966666666666,0.21545539883903742,0.6289,0.024544245761481467,19.627566666666663,0.19993819600622176,0.5631,0.003053959178945702
4000000.0,19.747400000000003,0.405152177171327,1.0,0.0,42.578149999999994,1.2265499999999996,1.0,0.0,18.791666666666668,0.4774405745453793,0.6467666666666667,0.043652440430697076,19.2005,0.29650741418498644,0.5720999999999999,0.023394016328967517
}\dataBwLimitSteps

%% file: 1_intro.tex
\section{Introduction}  
\label{sec: intro}

In Multi-Agent Reinforcement Learning (MARL), communication plays a key role in facilitating knowledge sharing and collaboration \citep{dial,flow}. 
The information in agents' messages alleviates the limitation of \emph{partial observation}. 
Communication leads to better exploration on the huge state and action space, and thus improves the training quality and convergence speed.

While there have been numerous works in the literature to improve communication in MARL, most are based on \emph{unrealistic} modeling of the wireless network environment (\eg, assuming perfect transmission). 
In practical applications, however, transmission can fail due to many factors such as limited bandwidth, signal path loss and fading, medium contention, interference, \textit{etc}. Moreover, in applications involving navigation (\eg, Search-And-Rescue (SAR) \cite{sar}, fire fighting \cite{firefight}, battlefield defense \cite{battlefield}), link conditions are dynamic as they vary with agents' mobility and relative positions. 
It remains an open problem to learn a practical policy by addressing such complexity in realistic wireless communications.

The following challenges exist to train communicating agents in a realistic wireless environment. The first challenge, \emph{environment coupling}, is due to the mutual influence between the game environment and the wireless environment. 
On the one hand, agents' mobility in the game environment leads to dynamic network connectivity and agents' communication actions have significant impact on medium contention and signal interference. 
On the other hand, changes in the wireless environment can affect agents' actions in the game environment. 
For example, when an agent generating an important message is far away from other agents, it may intentionally approach others to increase the receivers' signal strength. 
The second challenge, \emph{channel stochasticity}, means it is a non-deterministic process whether an agent can successfully receive the message from others. As described in the previous paragraph, many practical factors
can result in failed transmission.

The stochasticity not only increases the complexity of the environment, but also makes it more challenging for agents to extract useful information.

\paragraph{Proposed work} We propose a framework, {\coolname}, to \underline{L}earn pr\underline{A}ctical comm\underline{U}nication strategies in cooperative Multi-Agent \underline{RE}inforcement \underline{L}earning. 
Our framework optimizes three fundamental aspects of communication (\ie, when, what and how) without unrealistic simplifications on the wireless network. 
To address the environment coupling challenge, we propose to delay the message sending time by one step so as to align agents' interactions in the game and wireless environments. 
The alignment solves the issue of training non-differentiability in \cite{schednet}. 
More importantly, it leads to a ``meta-environment'' abstraction, where the Markov Decision Process (MDP) controlling the agents' behaviors can be reformulated, and both the action and observation spaces can be expanded.
Specifically, we augment the agents' observation by important network measurements, so that they can understand and predict the dynamics due to coupling. 
Then we end-to-end train a practical communication strategy with neither pre-defined schedule nor prior knowledge on the wireless condition (unlike \cite{schednet,nips16_commnet}). 
As a result, agents learn when to send messages based on both the content relevance and the wireless channel conditions.
To further improve {\coolname} under high channel \emph{stochasticity}, we propose a novel neural architecture to encode the received messages. 
The encoder takes as input \emph{any} number of messages and theoretically guarantees a lossless encoding process. 
Finally, we generalize {\coolname} with popular backbone MARL algorithms, and build both the on- and off-policy variants in a plug-and-play fashion. 
On standard benchmarks, we achieve significantly better performance with respect to game performance, convergence rate and communication efficiency, compared with state-of-the-art methods.

%% file: 5_related.tex
\section{Related Work} 
\label{sec: related}

CommNet \citep{nips16_commnet} and BiCNet \citep{bicnet} propose to learn the communication contents in the MARL system to boost performance. However, both works perform all-to-all communication at each step, under the assumption of \emph{ideal wireless network} conditions (\ie whatever messages sent out will be successfully received by all agents). In terms of message aggregation under such assumption, a simple mean or sum function is used. TarMAC \citep{tarmac} and MAGIC \citep{magic} later introduce attention-based mechanisms to improve message aggregation.

Recent works consider \emph{resource constrained wireless networks}, which improves the ideal model by considering limited bandwidth. 
When multiple agents simultaneously send $k'$ messages, the environment ensures $\min\{k, k'\}$ messages to be successfully received by all agents (where $k$ is a manually configured constant characterizing the bandwidth limit). 
VBC \citep{vbc} limits the variance of the message and only those with large variance are exchanged during execution.
SchedNet \citep{schednet} generates an importance weight for the message, and exactly $k$ most important messages are exchanged in each step.
RMADDPG \citep{rmaddpg} adopts a recurrent multi-agent actor-critic framework which limits the number of communication actions during training.
Note that these works still assume perfect transmission, and ignore many critical factors in a realistic wireless environment.
All works above assume at least one of the following: 
    (1) Dedicated and perfect wireless channels exist between any pair of agents. Any transmitted message can be delivered to all other agents successfully. 
    (2) Message importance is the only factor when deciding ``when to send a message''. 
    (3) Channel condition is fixed and known in advance to allow scheduling agents' communication with a fixed budget. 
In comparison, our framework does not have any of the above assumptions and considers realistic wireless channels.

%% file: 3_problem.tex
\section{Preliminaries} 
\label{sec: problem defn}

\parag{Dec-POMDP}
Agents solve a problem formulated as a Decentralized Partially-Observable Markov Decision Process (Dec-POMDP), which is represented by $\paren{\mathcal{S}, \bm{\mathcal{A}}, \mathcal{T}, \bm{{\Omega}}, \mathcal{O}, \mathcal{R}, \mathcal{\gamma}}$. 
$\mathcal{S}$, $\bm{\mathcal{A}}$ and $\bm{{\Omega}}$ are the state space, joint action space, and joint observation space for the $N$ agents, respectively.
At step $t$, each agent $i\in N$ takes action $a_{i}$ by following policy $\pi_{\theta_{i}}$. The joint action $\bm{a}$ leads to state transition captured by function $\mathcal{T}(s'|s,\bm{a}): \mathcal{S} \times \bm{\mathcal{A}} \times  \mathcal{S} \rightarrow [0,1]$.
The agent also draws local observation $o_{i}\in{\Omega}_{i}$, with joint observation probability $P(\bm{o}|s', \bm{a}) = \mathcal{O}(\bm{o}, s', \bm{a})$, where $\bm{o}$ is the joint observation.
Meanwhile, agents receive a joint reward $r_{t} = \mathcal{R}(s, \bm{a})$.  The goal is to find an optimum policy $\bm{\pi_{{\theta}^{*}}}$ that maximizes the expected long-term reward (\ie, return), with discount factor $\gamma\in [0,1]$:
$J = \mathbb{E}_{\bm{\pi_{\theta}}}\left[\sum_{t=1}^{\infty}\gamma^{t} r_t\right]$.

\paragraph{Communication Environment}

In MARL applications with mobile agents (\eg, SAR), agents altogether formulate a Mobile Adhoc network (MANET) \citep{manet_overview} to communicate. In a large and complicated terrain, network conditions vary with agents' mobility. \eg, there may be obstacles of different materials blocking the signal propagation. Under popular wireless protocols (\eg, $p$-CSMA \cite{gaiyi}), before agents send, they contend to access the medium with certain probability to avoid collision. The signal strength of a transmitted data packet may be affected by path loss, attenuation, interference with other packet signals, \textit{etc}.  In particular, for path loss, log-normal fading model (see Appendix \ref{appendix: pathloss}) is widely used to capture the Radio Signal Strength (RSS) variation over the spatial domain. 
For interference, it happens when receiver is in range of multiple senders at the same time. When SINR (Signal-to-Interference-plus-Noise-Ratio, in terms of signal power level) is below a given threshold, the packet cannot be decoded correctly. In such a case, a broadcast message cannot be received by all agents.

%% file: 4_method.tex
\section{Method}
\label{sec: method}

\subsection{When: MDP Re-Formulation}
\label{sec: method mdp}

\begin{figure}
    \floatconts
    {fig: 2env}
	{\caption{Two ways of agent-environment interaction}}
    {
    \subfigure[Original: non-differentiable training]{
    \label{fig: 2env orig}
    \input{./diagrams/env_old.tex}
    }
        \subfigure[New: alignment after lagging]{
        \label{fig: 2env new}
        \input{./diagrams/env_new.tex}
    }
    }
\end{figure}
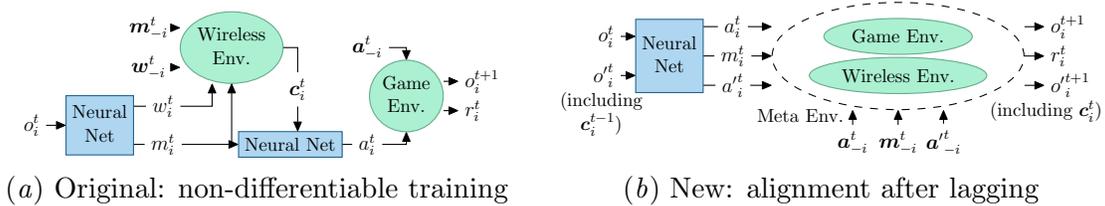

Use subscript ``$-i$'' to denote variables from all agents other than $i$. 
Communication can facilitate exchange of \emph{current} observations. A natural design, followed by most existing works \citep{dial,nips16_commnet,ic3net}, is to send and receive a message within the same step. 
In Figure \ref{fig: 2env}(a), 
at the beginning of step $t$, agent $i$ makes local observation $o_i^t$ on the game environment. 
From $o_i^t$, a neural network generates a message $m_{i}^{t}$ with weight $w_{i}^{t}$ measuring message importance / relevance.
Then based on $w_{i}^{t}$ and $\bm{w}_{-i}^{t}$, the messages $m_i^t$ and $\bm{m}_{-i}^t$ contend to go through the wireless channel. 
Since some transmissions may fail, we use $\bm{c}_i^t$ to denote the messages successfully received by $i$. 
Next, agent $i$'s own message and all its received ones go through another neural network to generate the next action $a_i^t$. 
Finally, all agents interact with the game environment via $a_i^t$ and $\bm{a}_{-i}^t$. 
The game environment returns the current reward $r_i^t$ and next observation $o_i^{t+1}$. 

A major drawback of the Figure \ref{fig: 2env}(a) setting is training non-differentiability. 
When optimizing policy $\pi_{\theta}$ via back-propagation, the gradients of the parameters $\theta$ need to flow from $a_i^t$ back to $o_i^t$. 
However, the wireless environment lies in the middle of step $t$. 
The mapping implemented by the realistic wireless channel, ${\bm{m}^t, \bm{w}^t} \rightarrow \bm{c}_i^t$, is non-differentiable \citep{schednet} (due to ``stochasticity'' in transmission).

\paragraph{Communication Lagging}
We propose ``communication lagging'' to make training differentiable. 
Denote $a'$ as communication actions and $o'$ as wireless observations (see Section \ref{sec: method what}). 
As shown in Figure \ref{fig: 2env}(b), an agent observes at the beginning of step $t$. Yet it does not generate the message $m_i^t$ until the end of step $t$. 
After lagging, the ``message-wireless environment'' and ``agent-game environment'' interactions happen simultaneously (enclosed by dotted circle). 
Thus, the two environments are \emph{aligned}. 
Lagging leads to the following tradeoffs:
(1) \emph{Ease of training}: the aligned environments enable end-to-end policy training, as the gradients can flow backward from the end of step $t$ (\ie, $a_i^t$ and ${a'}_i^t$) to the beginning of step $t$ (\ie, $o_i^t$ and ${o'}_i^t$). 
We thus develop a general framework (Section \ref{sec: method complete}) supporting any existing MARL algorithm on top of realistic wireless environment. 
(2) \emph{Staleness of messages}: action $a_i^t$ is generated with stale information $\bm{c}_i^{t-1}$ since the step-$(t-1)$ messages are received at step $t$. 
Such staleness may only have slightly negative impact when the step duration is short in practice \citep{vbc}. 

\paragraph{Meta-Environment}
\label{sec: method when meta}
We abstract a ``meta-environment'' from aligned game and wireless environments. 
Meta-environment is the \emph{new} game environment: any algorithm on the original game environment can be directly applied on the meta-environment. 
To define agents' interaction with the meta-environment, we reformulate the Markov Decision Process (MDP) with expanded state and action spaces: $\paren{\mathcal{S}, \bm{\mathcal{A}}, \mathcal{T}, \bm{\Omega}, \mathcal{O}, \mathcal{R}, \mathcal{\gamma}}$. 
Denote $\mathcal{S}$ as the state of the meta-environment. 
The \textit{augmented} action space
$\bm{\mathcal{A}}$ consists of the game actions $\bm{\mathcal{A^{T}}}$ and communication actions $\bm{\mathcal{A^{C}}}$. \ie, $\bm{\mathcal{A}} = \bm{\mathcal{A^{T}}} \times \bm{\mathcal{A^{C}}}$. 
The \textit{augmented} observation space $\bm{\Omega}$ consists of the game observations $\bm{\Omega^{T}}$ and wireless observations $\bm{\Omega^{C}}$. \ie, $\bm{\Omega} = \bm{\Omega^{T}} \times \bm{\Omega^{C}}$. 
$\mathcal{T}$ and $\mathcal{O}$ are defined on the meta-state and augmented observation and action spaces.

\paragraph{Learning Communication Actions}
\label{sec: method when action}
Under the reformulated MDP, agents can learn a policy on communication actions to decide ``when to communicate''. 
The simplest is the binary action for ``send / no-send'': $a_{i}^{C}\in \mathcal{A}_{i}^{C}=\left\{0,1\right\}$. 
Other examples of learnable $a_{i}^{\mathcal{C}}$ include:
(1) \emph{Transmission power} $P_t$: larger $P_t$ improves the received signal strength at the cost of higher probability of interference. 
(2) \emph{Medium contention probability} $p$: larger $p$ increases the probability of securing the medium by suppressing other agents' chance.

\subsection{What: Observation Augmentation}
\label{sec: method what}

Efficiently exploring the state space $\mathcal{S}$ and action space $\mathcal{A}$ relies on suitable observations, especially considering the complexity in wireless environment. 
However, since agents may be far away, it is infeasible for them to directly observe the full network condition (\ie, the wireless environment is partially observable). 
To address the issue, we first augment the observation space $\mathcal{O}$ by observation / measurement $o_i^{\mathcal{C}, t}$ on the wireless environment. 
\ie. for each agent, we concatenate the wireless and game observations $o_{i}^{t} = o_i^{\mathcal{T}, t} \|o_i^{\mathcal{C}, t}$. 
Then we let the message sent by $i$ contain $i$'s hidden state $\bm{h}_i^t$, where $\bm{h}_i^t$ embeds information of $o_i^t$. 
When $j$ receives from $i$, 
agent $j$ predicts the current network condition. 
If such condition is believed as poor, $j$ may suppress its sending action. 

\paragraph{Network Measurements}
Common factors resulting in failed transmission include limited bandwidth, noises, signal attenuation due to obstacles and packet collision. 
The network condition is complicated: some of the above factors are related to the wireless environment (\eg, limited bandwidth), some are related to the game environment (\eg, obstacles), and others are even related to agents' policies (\eg, collision due to simultaneous sendings). 
Agents can understand the network condition from various types of network measurements, such as \emph{Radio Signal Strength} (RSS) and \emph{Acknowledgement packets} (ACK). The sender agent can use ACK to infer the packet receiving rate and the data rate. For brevity, we 
now only illustrate the benefits of augmenting observation via RSS $P_s$ in detail. 
When agent $i$ receives the message from $j$, it
measures $P_s$ and also infers $j$'s position (thus distance $d$ between $i$ and $j$). 
When agent $i$ receives more messages from other agents across multiple steps, it essentially collects many $\paren{P_s, d}$ data points. 
The agent can then learn the function $P_s = g\paren{d}$ corresponding to the path loss model of the particular wireless environment. 
With $g$, when an agent knows (or predicts) the position of others, it immediately knows the corresponding RSS if it sends a message. 
Ignoring collisions, the agent knows if its message can be successfully received or not, \emph{even before it sends out the message}. 
Such knowledge clearly helps with a better policy on communication actions. 
In addition, RSS information can also reveal terrain information on the game environment. If there is an obstacle between $i$ and $j$, the received RSS will decrease due to attenuation. 
If the measured $P_s$ is significantly lower than the estimated $g\paren{d}$, the agent detects an obstacle. 
Communication then serves as a form of long-range sensing on the terrain, and helps agents better plan their future movements.

\subsection{How: Message Encoder}
\label{sec: method how}
\parag{Stochasticity in Communication}\label{sec: method how stochasticity}
An agent does not know 
(1) who will send a message at what time, 
(2) whether the message will have large enough signal strength to be successfully received, 
(3) at what order messages will arrive, 
(4) what the message will contain. 
For point 1, agents in non-deterministic RL algorithms select communication actions by random sampling from a probability distribution. 
For point 2, many factors (\eg, path loss, fading, noise, interference) can cause packet loss. 
For point 3, due to medium contention and varying network latency, the multiple messages to be received by a single agent in one step can arrive in an arbitrary order. 
For point 4, for the hidden state embedding contained in a message, 
the vector elements can take arbitrary numerical values representable by a floating point number.

\paragraph{Message Encoder}
\label{sec: method how arch}
We design a neural architecture to encode received messages by addressing all the stochasticity. 
In the following, we consider a single step, and thus omit superscript $\paren{t}$. 
Denote vector $\bm{m}_{ij}\in\R[l]$ as the message from agent $j$ to $i$. 
Denote $\mathcal{N}_i$ as the set of agents successfully sending to $i$. 
Cardinality of $\mathcal{N}_i$ is from $0$ (when $i$ receives no message) to $n-1$ (when $i$ receives from all other agents).
Define $c_i = \left\{\bm{m}_{ij}: j\in\mathcal{N}_i\right\}$ as the set of messages received by $i$. 
The encoder performs a function $\Phi\paren{c_i}$ to map the set of messages to a vector in $\R[d']$. 
$\Phi$ should satisfy the following two properties.

\subparag{\textbf{Permutation invariance}}{}
This property addresses the stochasticity 3 above. 
Since we define the input to $\Phi$ as an \emph{un-ordered} set $c_i$, we need to preserve a well-known property of such a set function, namely permutation invariance \citep{deepsets}. 
By definition, for any $c_i = \left\{\bm{m}_{i1}, \hdots, \bm{m}_{ik}\right\}$ and any permutation $\rho$ on the indices $\left\{1,\hdots,k\right\}$, a permutation invariant function satisfies 
\begin{align}
	\label{eq: perm}
  \Phi\paren{\left\{\bm{m}_{i1}, \hdots, \bm{m}_{ik}\right\}}=\Phi\paren{\left\{\bm{m}_{i\rho\paren{1}},\hdots,\bm{m}_{i\rho\paren{k}}\right\}}
\end{align}

We show an example in our scenario. 
Suppose agent 1 receives messages from agents 2, 3 and 4. 
Imagine two sequences of message arrival, $\paren{2, 3, 4}$ and $\paren{3, 4, 2}$. 
The encoder should not care about the sequence and should generate the same output for both cases. Thus, $\Phi\paren{\left\{\bm{m}_{12}, \bm{m}_{13},\bm{m}_{14}\right\}}=\Phi\paren{\left\{\bm{m}_{13},\bm{m}_{14},\bm{m}_{12}\right\}}$.

\subparag{\textbf{Injectiveness}}{}
This property addresses stochasticity 1, 2, and 4. 
To preserve all information in the encoding process, we should have
\begin{align}
	\label{eq: inj}
 c_i= c_i'\quad\Leftrightarrow\quad\Phi\paren{c_i}= \Phi\paren{c_i'}
\end{align}
\noindent 
This way, agent $i$ generates a unique encoding for $c_i$ so that it knows exactly the status of \emph{all} the agents who successfully sends a message to it. 
For a vector function, injectiveness is easy to satisfy. 
For example, a linear mapping $\Phi'(\bm{m'})=\bm{W}\bm{m}'$ is injective when $\bm{W}$ has rank $l$. 
However, it is non-trivial to ensure injectiveness of $\Phi$ whose input is a \emph{set} with non-fixed size. In the literature \citep{nips16_commnet,ic3net,schednet}, a common way to implement $\Phi$ is by first aggregating the messages into a single vector and then encode the aggregated vector by a vector function. \eg, sum aggregation $\Phi\paren{c_i}=\Phi'\paren{\sum_{j\in \mathcal{N}_i}\bm{m}_{ij}}$. 
Unfortunately, such a $\Phi$ is not injective since regardless of $\Phi'$, we can find cases where $c_i\neq c_i'$ yet $\sum_{j\in \mathcal{N}_i}\bm{m}_{ij} = \sum_{j\in \mathcal{N}_i'}\bm{m}_{ij}'$.

Different from the literature, we implement the message encoder as follows, 

\begin{align}
	\label{eq: how gnn}
    \text{Message encoder:}\hspace{.7cm}\textstyle\Phi\paren{c_i} = \sum_{j\in\mathcal{N}_i}\func[MLP]{\bm{m}_{ij}}
\end{align}

\noindent
where  $\func[MLP]{\cdot}$ is a Multi-Layer Perceptron. 
All received messages go through a shared MLP.  
Our encoder theoretically satisfies the above two properties and preserves all information.

\begin{theorem}
\label{thm: gin emb}
There exists a set of parameters for the encoder architecture defined by Equation \ref{eq: how gnn}, such that $\Phi$ is both permutation invariant and injective. 
\end{theorem}

\begin{proof}
    Please see Appendix~\ref{appendix: proof}.
  \end{proof}

\paragraph{Relation with GNNs} Our encoder can be seen as a single-layer Graph Neural Network tailored for real-life communication stochasticity. 
Correspondingly, ``sum of MLPs'' of {\coolname}'s Equation \ref{eq: how gnn} is the GNN aggregation function, 
which preserves the theoretical properties while being simple to implement in practice. 
Communication design of many related works can also be related with GNNs. 
For example, 
TarMAC \citep{tarmac} performing ``weighted mean'' of messages can be seen as a single-layer Graph Attention Network \citep{gat}; IC3Net \citep{ic3net} performing ``unweighted average'' is equivalent to GraphSAGE \citep{graphsage}. 
Other related works (\eg, MAGIC \citep{magic}) deploy multi-layer GNNs that require multi-hop communication in each step. Such communication can be too expensive under constraints of realistic wireless networks, and thus we restrict {\coolname} to the single-layer design. 
Importantly, for all the above related works, their GNN encoders satisfy permutation invariance but \emph{not injectiveness}. As a result, there can be significant performance drop due to the aforementioned stochasticity.

\subsection{General Framework}
\label{sec: method complete}

We propose a \textbf{\emph{general}} framework to learn practical communication strategies, since our techniques can be applied to various MARL training algorithms in a plug-and-play fashion. 

\vspace{.3cm}
\noindent
\begin{minipage}{0.58\linewidth}
\begin{algorithm2e}[H]
\DontPrintSemicolon
\caption{Example off-policy {\coolname}}
\label{algo: qmix-wink}
Init. all neural networks and replay buffer $\mathcal{D}$\;
\For{episode$=1$ to $M$}{
    Init. observation $\bm{o}$; init. messages $\bm{c}$ to all 0\;
    \For{$t=1$ to end of episode}{
        Receive msg $c_{i}^{t-1}$; get network measurement $\bm{o}_{i}^{\mathcal{C},t}$\;
        $\phi_{i}^{t}\gets$ Encode message by $\Phi_i\paren{c_i^{t-1}}$\;
        $a_{i}^t, m_{i}^t, Q_{i}^t\gets$ $\epsilon$-greedy $\arg\max_{a_i}Q^*_i\paren{o_i^t, \phi_i^t, a_i}$\;
        Take action $\bm{a}^{\mathcal{T},t}$ in game environment\;
        Send $\bm{\phi}^t$ by $\bm{a}^{\mathcal{C},t}$ in wireless environment\;
        Observe reward $r^t$; transit to state $\bm{s}^{t+1}$\;
        Add transition to replay buffer $\mathcal{D}$\;
    }
    Randomly sample trajectory from $\mathcal{D}$\;
    Update $\theta_{\Phi}$, $\theta_{Q}, \theta_\text{mix}, \theta^{'}_{Q}, \theta^{'}_\text{mix}$\;
}
\end{algorithm2e}
\end{minipage}
\begin{minipage}{0.38\linewidth}
    \input{diagrams/qmix_emb}
	
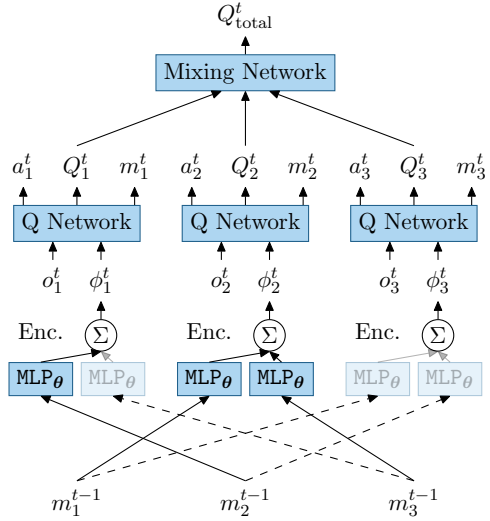
\captionof{figure}{Example architecture of off-policy {\coolname}}
	\label{fig: method eg framework}
\end{minipage}
\vspace{.3cm}

For illustration, we describe an example implementation based on value function factorization \citep{vdn,qmix}, which can be trained in the \textbf{\emph{off-policy}} manner. 
As shown in Figure \ref{fig: method eg framework},
Three neural networks interact with each other:
\emph{message encoder}, \emph{$Q$ network} and  \emph{mixing network}.
The message encoder follows the Section \ref{sec: method how} design. 
The $Q$ network estimates the $Q$ value for each individual agent.
The mixing network combines each agent's $Q$ value to calculate a global $Q$ for all agents. 
This design follows the ``centralized training, decentralized execution'' paradigm \citep{ctde}. 
During centralized training, the mixing network optimizes the policy via maximizing the global $Q$ value. 
During decentralized execution, each agent outputs its next action with its own $Q$ network alone, without the mixing network. 
The inputs to $i$'s $Q$ network include $i$'s current augmented observation $o_i^t$ and the encoded message vector $\phi_i^t$. 
The output is the augmented action $a_i^t$ and the $Q$ value of $\bm{a}_i^t$. 
The inputs to the mixing network are the output $Q$ of each $Q$ network. The output is the global $Q$ value for the current step $t$. 

Algorithm \ref{algo: qmix-wink} shows the training algorithm. 
Since we use GRU \citep{gru} in $Q$ network, we perform gradient update at the end of rollout episode instead of every step (as in \cite{rmaddpg}). 
In line 13, we perform gradient descent from the loss  
\begin{align}
	\label{eq: loss}
	\mathcal{L}\paren{\theta} = \sum_{t}\left[(y_t^{\text{tot}}-Q_{\text{tot}}\paren{\bm{o},\bm{c}, \bm{a};\theta})^{2} \right]
\end{align}
\noindent 
where $y_t^{\text{tot}}=r+\gamma\max_{\bm{a}'}Q_\text{tot}\paren{\bm{o}', \bm{c}', \bm{a}';\theta'}$, $\bm{c}$ is the received message. $Q_\text{tot}$ is the full architecture in Figure \ref{fig: method eg framework}
with parameter $\theta$.
$\paren{\bm{o}, \bm{c}, \bm{a}, r, \bm{o}', \bm{c}'}$ is transition-$t$ in the sampled trajectory.

We can similarly integrate \textbf{\emph{on-policy}} algorithms into our framework. 
In Section \ref{sec: exp}, we evaluate REINFORCE \citep{reinforce} based algorithms. 
The gradient function is 
$\nabla_\theta J(\theta) = \sum_{i}\sum_{t} \left[ \nabla_\theta \log\pi_\theta(a_i|o_i, c_i) G_t  \right]$,
where $G_t$ is the sampled returns. 
Other variants (\eg, based on Actor-Critic) can also be derived similarly. 
We omit the details for brevity.

%% file: diagrams/env_old.tex
\tikzset{cross/.style={cross out, draw=black, minimum size=2*(#1-\pgflinewidth), inner sep=0pt, outer sep=0pt},
cross/.default={1pt}}
\begin{tikzpicture}[
    >={Stealth[inset=0pt,length=4pt,angle'=45,round]},
    env_block/.style={ellipse,draw=color_env_outer,fill=color_env_inner},
    nn_block/.style={draw=color_nn_outer,fill=color_nn_inner},
    scale=0.65, every node/.style={transform shape}
]
    \node[nn_block] (nn msg) at (0, 0) {\makecell{Neural\\Net}};
    \node[left=.4cm of nn msg.west] (obs t) {\large$o_i^t$};
     \draw[->] (obs t.east) -- (nn msg.west);
     \node[right=.6cm of nn msg.east,yshift=-.4mm,anchor=north] (msg self) {\large$m_i^t$};
     \node[right=.6cm of nn msg.east,yshift=.4mm,anchor=south] (weight self) {\large$w_i^t$};

     \node[above=.8mm of weight self.north,anchor=south,xshift=-.3cm] (msg others) {\large$\bm{w}_{-i}^t$};
     \node[above=.8mm of msg others] (weight others) {\large$\bm{m}_{-i}^t$};
     \node[env_block,inner sep=.4mm] (env wifi) at ($(weight others)!.50!(msg others)+(1.7,0)$) {\makecell{Wireless\\Env.}};
     
    \node[nn_block] (nn as) at ($(msg self.east-|env wifi.south)+(1.2,0)$) {Neural Net};
    \node (msg recv) at ($(env wifi.east-|nn as.north)!.5!(nn as.north)+(.15,0)$) {\large$\bm{c}_{i}^t$};
    
     \node[right=.2cm of nn as.east] (act self) {\large$a_i^t$};
    \node (act others) at (act self.north |- env wifi.east) {\large$\bm{a}_{-i}^t$};
    \node[env_block,inner sep=.4mm] (env game) at ($(act others)!.5!(act self)+(.8,0)$) {\makecell{Game\\Env.}};
    
    \node[right=.3cm of env game.east,yshift=.3cm] (obs t1) {\large$o_i^{t+1}$};
    \node[right=.3cm of env game.east,yshift=-.3cm] (reward t) {\large$r_i^{t}$};
    
    \foreach \type in {weight self,msg self} {
        \draw (nn msg.east|-\type.west) -- (\type.west);
    }
    \draw[->] (msg self.east) -- (nn as.west);
    \draw[->] ($(env wifi.south|-msg self.east)-(0,0)$) -- ($(env wifi.south)-(0,0)$);
    \draw[->] (weight self.east) -- ($(env wifi.south|-weight self.east)-(.4,0)$) -| ($(env wifi.south)-(.4,0)$);
    \foreach \type in {weight others,msg others} {
        \draw[->] (\type.east) -- (env wifi.west|-\type.east);
    }
    \draw (env wifi.east) -| (msg recv.north);
    \draw[->] (msg recv.south) -- (nn as.north-|msg recv.south);
    \draw (nn as.east) -- (act self.west);
    \draw[->] (act self.east) -| (env game.south);
    \draw[->] (act others.east) -| (env game.north);
    \foreach \out in {obs t1,reward t} {
        \draw[->] (env game.east|-\out.west) -- (\out.west);
    }
\end{tikzpicture}

%% file: diagrams/env_new.tex
\tikzset{cross/.style={cross out, draw=black, minimum size=2*(#1-\pgflinewidth), inner sep=0pt, outer sep=0pt},
cross/.default={1pt}}
\begin{tikzpicture}[
    >={Stealth[inset=0pt,length=4pt,angle'=45,round]},
    env_block/.style={ellipse,draw=color_env_outer,fill=color_env_inner},
    nn_block/.style={draw=color_nn_outer,fill=color_nn_inner},
    scale=0.65, every node/.style={transform shape}
]
    \node[nn_block, minimum height=1.5cm] (nn msg) at (0, 0) {\makecell{Neural\\Net}};
    \node[left=.3cm of nn msg.west,yshift=.4cm] (obs t) {\large$o_i^t$};
    \node[left=.3cm of nn msg.west,yshift=-.4cm] (obs t wifi) {\large${o'}_i^t$};
    \node[below =.4cm of obs t wifi,xshift=-.05cm,anchor=center] (note input) {\makecell{(including\\$\bm{c}_i^{t-1}$)}};
    \draw[->] let
        \p1 = (obs t.east), \p2 = (nn msg.west)
      in (\x1,\y1) -- (\x2,\y1);
    \draw[->] let
        \p1 = (obs t wifi.east), \p2 = (nn msg.west)
      in (\x1,\y1) -- (\x2, \y1);
    \node[right=.6cm of nn msg.east,yshift=.6cm,anchor=center] (act) {\large$a_i^t$};
    \node[right=.6cm of nn msg.east,anchor=center] (msg) {\large$m_i^t$};
    \node[right=.6cm of nn msg.east,yshift=-.6cm,anchor=center] (act others) {\large${a'}_i^t$};
    \draw let
        \p1 = ($(nn msg.east)+(0,.6cm)$), \p2 = (act.west)
      in (\x1,\y1) -- (\x2,\y1);
    \draw let
        \p1 = (nn msg.east), \p2 = (msg.west)
      in (\x1,\y1) -- (\x2,\y1);
    \draw let
        \p1 = ($(nn msg.east)-(0,.6cm)$), \p2 = (act others.west)
      in (\x1,\y1) -- (\x2,\y1);
    \node[env_block,right=3cm of msg,yshift=.4cm,anchor=center] (env game) {Game Env.};
    \node[env_block,right=3cm of msg,yshift=-.4cm,anchor=center] (env wifi) {Wireless Env.};
    \node[ellipse,fit=(env game) (env wifi),inner sep=0mm,draw,dashed] (env meta) {};
    \draw[->] (msg) -- (env meta);
    \draw[->] let
        \p1 = (act.east), \p2 = ($(env meta.west)+(0,.6cm)$)
      in (\x1,\y1) -- (\x2,\y2);
    \draw[->] let
        \p1 = (act others.east), \p2 = ($(env meta.west)+(0,-.6cm)$)
      in (\x1,\y1) -- (\x2,\y2);
    
    \node[right=.4cm of env meta] (reward) {\large$r_i^t$};
    \node[right=.4cm of env meta,yshift=.6cm] (obs t1) {\large$o_i^{t+1}$};
    \node[right=.4cm of env meta,yshift=-.6cm] (obs wifi t1) {\large${o'}_i^{t+1}$};
    \node[below=.15cm of obs wifi t1,xshift=-.5cm,anchor=center] (note output) {(including $\bm{c}_i^{t}$)};
    \draw[->] let
        \p1 = (env meta.east), \p2 = (obs t1.west)
      in (\x1,\y2) -- (\x2,\y2);
    \draw[->] (env meta.east) -- (reward.west);
    \draw[->] let
        \p1 = (env meta.east), \p2 = (obs wifi t1.west)
      in (\x1,\y2) -- (\x2,\y2);
    
    \node[below=.3cm of env meta.south] (msg input) {\large$\bm{m}_{-i}^t$};
    \node[left=.4cm of msg input,anchor=center] (act wifi others input) {\large${\bm{a}}_{-i}^t$};
    \node[right=.4cm of msg input,anchor=center] (act others input) {\large${\bm{a}'}_{-i}^t$};
    \draw[->] (msg input) -- (env meta);
    \draw[->] let
        \p1 = (act wifi others input.north), \p2 = (env meta.south)
      in (\x1,\y1) -- (\x1,\y2);
    \draw[->] let
        \p1 = (act others input.north), \p2 = (env meta.south)
      in (\x1,\y1) -- (\x1,\y2);
      
    \node (note meta env) at ($(env meta.south)+(-2cm,-.1cm)$) {Meta Env.};
\end{tikzpicture}

%% file: diagrams/qmix_emb.tex
\tikzset{cross/.style={cross out, draw=black, minimum size=2*(#1-\pgflinewidth), inner sep=0pt, outer sep=0pt},
cross/.default={1pt}}
\begin{tikzpicture}[
    >={Stealth[inset=0pt,length=4pt,angle'=45,round]},
    env_block/.style={ellipse,draw=color_env_outer,fill=color_env_inner},
    nn_block/.style={draw=color_nn_outer,fill=color_nn_inner},
    scale=0.8, every node/.style={transform shape}
]
    \def\xNNsep{2.8cm}
    \def\yNNsep{2.2cm}
    \def\widthNN{2.1cm}
    \foreach \i/\opa\opb in {1/0.3/1,2/1/1,3/0.3/0.3}
    {
        \node[nn_block,minimum width=\widthNN] (nn q \i) at ($(\i*\xNNsep,0)$) {Q Network};
        \node[above=.3cm of nn q \i] (Q \i) {$Q_\i^t$};
        \node[left=.2cm of Q \i] (a \i) {$a_\i^t$};
        \node[right=.2cm of Q \i] (m \i) {$m_\i^t$};
        \node[below=.3cm of nn q \i,xshift=.4cm] (phi \i) {$\phi_{\i}^t$};
        \node[below=.3cm of nn q \i,xshift=-.4cm] (o \i) {$o_\i^t$};
        \draw[->] let
            \p1=(nn q \i.north),\p2=(a \i.south)
          in (\x2,\y1) -- (\x2,\y2);
        \draw[->] let
            \p1=(nn q \i.north),\p2=(m \i.south)
          in (\x2,\y1) -- (\x2,\y2);
        \draw[->] let
            \p1=(nn q \i.north),\p2=(Q \i.south)
          in (\x2,\y1) -- (\x2,\y2);
        \draw[->] let
            \p1=(phi \i.north), \p2=(nn q \i.south)
          in (\x1,\y1) -- (\x1,\y2);
        \draw[->] let
            \p1=(o \i.north), \p2=(nn q \i.south)
          in (\x1,\y1) -- (\x1,\y2);
        
        \node[nn_block,below=2cm of nn q \i.south,xshift=.6cm,opacity=\opa] (enc mlp1\i) {\text{\fontfamily{lmtt}\selectfont MLP}\textsubscript{$\bm{\theta}$}};
        \node[nn_block,below=2cm of nn q \i.south,xshift=-.6cm,opacity=\opb] (enc mlp2\i) {\text{\fontfamily{lmtt}\selectfont MLP}\textsubscript{$\bm{\theta}$}};
        \node[below=.3cm of phi \i.south,circle,draw,inner sep=.5mm] (sum \i) {$\mathlarger{\Sigma}$};
        \draw[->,opacity=\opa] (enc mlp1\i.north) -- (sum \i.south);
        \draw[->,opacity=\opb] (enc mlp2\i.north) -- (sum \i.south);
        \draw[->] (sum \i) -- (phi \i);
        \node[above=.25cm of enc mlp2\i] (text enc) {Enc.};
        \begin{pgfonlayer}{background}
            \node[inner sep=0pt, fit=(sum \i) (enc mlp1\i) (enc mlp2\i),fill=white,opacity=0.4] (nn encoder) {};
        \end{pgfonlayer}
        
        \node[below=\yNNsep of nn encoder.south,yshift=.8cm] (msg \i) {$m_\i^{t-1}$};
    }
    \foreach \from/\to/\tomlp/\isdash in {1/2/2/0, 1/3/2/1, 2/1/2/0, 2/3/1/1, 3/1/1/1, 3/2/1/0}
    {
        \ifthenelse{\isdash = 0}
        {
            \draw[->] (msg \from.north) -- (enc mlp\tomlp\to.south);
        }
        {
            \draw[->,dashed] (msg \from.north) -- (enc mlp\tomlp\to.south);
        }
    }
    
    \node[nn_block,above=\yNNsep of nn q 2,yshift=.0cm,anchor=center] (nn qmix) {Mixing Network};
    \foreach \i/\offx in {1/-.4,2/0,3/.4} {
        \draw[->] let
            \p1=(Q \i.north),\p2=($(nn qmix.south)+(\offx,0)$)
          in (\x1,\y1) -- (\x2,\y2);
    }
    \node[above=.3cm of nn qmix] (Q tot) {$Q_\text{total}^t$};
    \draw[->] (nn qmix.north) -- (Q tot.south);
    
                
\end{tikzpicture}

%% file: 6_exp.tex
\section{Experiments}
\label{sec: exp}
\subsection{Setup}

\parag{Wireless Environment}
\label{sec: env wifi setup}
Following most existing works \citep{dial,nips16_commnet,schednet}, 
we let agents perform single-hop broadcast.
We implement a 1-hop mobile network without Access Points (AP) as in \cite{planning}.
For communication model, we follow ``log distance path loss''. We model interference as the receiver hearing multiple signals in range. We also consider background noise and attenuation due to obstacles. For communication protocol, we implement the slotted $p$-CSMA \citep{gaiyi}.
Agents following the protocol perform fully distributed execution without the need of a centralized controller. See Appendix \ref{appendix: communication protocols} for details.

\paragraph{MARL Algorithms \& Hyperparameters}
We evaluate both on- and off-policy designs.

\subparag{Off-policy}{}
We compare three designs. 
The first is QMix \cite{qmix}, the state-of-the-art value decomposition based MARL training algorithm. 
The second enhances the communication part of the first design, by integrating the TarMAC message aggregation function \citep{tarmac} into the original QMix training. 
The third is the {\coolname} version of QMix, whose implementation follows Section \ref{sec: method complete} and Algorithm \ref{algo: qmix-wink}. 
Note that TarMAC is the state-of-the-art design performing attention-based message aggregation. 

\subparag{On-policy}{}
We again compare three designs.
Among them, CommNet \citep{nips16_commnet} and IC3Net \citep{ic3net} are two state-of-the-art MARL algorithms with learned communication schemes. 
The {\coolname} variant of IC3Net follows the description at the end of Section \ref{sec: method complete}. 
All three are trained with REINFORCE \citep{reinforce}. 

We conduct most experiments with off-policy models due to its significantly higher sample efficiency and faster convergence time.
See Appendix \ref{appendix: training hyperparameters} for training hyperparameters.

\subsection{Predator Prey}
\label{sec: exp pp}
``Predator-Prey'' is a standard MARL benchmark \citep{schednet,ic3net,vbc}. 
In this game, $m$ predators and $1$ prey are initially randomly placed in an $n\times n$ grid world (denoted as $\text{PP}_{n}$). 
In one step, each predator can take one of the five possible \emph{game actions}: moving to an adjacent grid or staying still. 
\emph{Reward} is given when a predator catches the prey by moving to the prey's position.
The game \emph{terminates} when all predators catch the prey. 
For \emph{observation}, each predator has vision $v$.

We further propose variants of the PP environment to better simulate realistic applications (\eg, SAR with complex field terrain). 
In the vanilla setting used by \cite{schednet,nips16_commnet,ic3net}, the grid world is flat and does not contain any obstacles. 
In the advanced setting, we introduce $k$ obstacles. 
An obstacle is specified by size $\ell$ and can be either horizontal or vertical.  
Obstacles affect both the game and wireless environments, since they 
block movement of agents and also have an attenuation effect on the wireless signals passing them.
To initialize each episode, we randomly generate the obstacle orientation and positions. 
We denote such an environment as \ppw{n}.
For all experiments, we set grid size $n=10$. There are 3 agents, each with vision $v=0$ (\ie, similar to \cite{ic3net}, agents only see the grid it is currently in), number of obstacles $k=1$ and obstacle size $\ell=9$. 
Detailed parameter settings of wireless environment can be found in Appendix \ref{appendix: wireless parameters}.

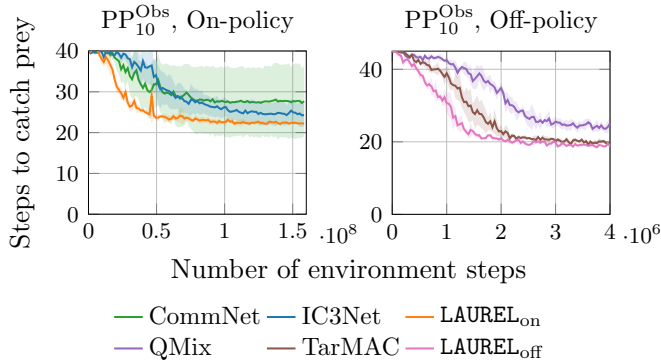
\begin{figure}[!htb]
    \centering
    \input{./plots/sota_steps}
    \vspace{-.25cm}
    \caption{Comparison with state-of-the-art methods}
    \label{fig: sota steps}
\end{figure}

\paragraph{Comparison with State-of-the-Art}
{\coolname} uses RSS as the wireless observation. 
From Figure \ref{fig: sota steps}, we observe that for both the on-policy and off-policy algorithms, {\coolname} significantly shortens the number of steps to catch the prey. In addition, the variances of the {\coolname} curves are very small. 
For CommNet, the model hard-codes an all-to-all communication scheme: each step, all agents broadcast messages. 
In a realistic wireless network environment, sending more messages can reduce the number of successfully received messages due to increased chance of collision and interference. 
Therefore, it can be hard for the algorithm always performing broadcasting to stably learn a good policy -- as reflected by the large variance of the CommNet curve. 
For IC3Net, the performance is better than CommNet since IC3Net has gated communication which mutes  unimportant messages. 
{\coolname} further improves upon IC3Net due to its intelligence in all the ``when'', ``what'' and ``how'' aspects. 
For the off-policy comparisons, QMix converges to a policy with significantly more steps. This shows the importance of communication. 
For {\coolname} and TarMAC, both converge to similar number of steps. However, {\coolname} converges faster and with smaller variance. 
Note that in the Lumberjacks environment (Section \ref{sec: exp lj}), TarMAC fails to learn a good policy while {\coolname} can.

\begin{figure}[!htb]
    \centering
    \input{./plots/traj_comm}
    \vspace{-.3cm}
    \caption{Average communication action in one trajectory}
    \label{fig: exp when traj}
\end{figure}
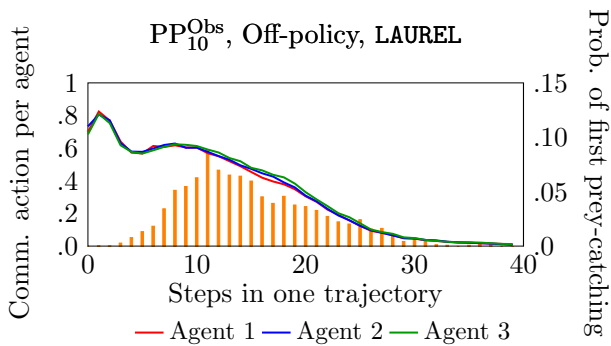

\paragraph{Learned Communication Strategy}
We analyze the communication scheme learned by our framework. 
We first answer ``\emph{at which steps of the trajectory the agents are more likely to send messages}''. 
In Figure \ref{fig: exp when traj}, we use the same configuration of {\coolname} as Figure 3. 
Once the training converges, we freeze the model and evaluate it on 2000 trajectories. 
We record two metrics: 
(1) the binary communication action $a_i^t$ of each agent $i$ at each step $t$, and 
(2) the first step $t'$ where at least one predator catches the prey. 
The left vertical axis (corresponding to the curves) of Figure \ref{fig: exp when traj} records the average $a_i^t$ over the 2000 trajectories. 
The right vertical axis (corresponding to the bars)  records the probability of $t'$. 
We have the following observations:
(1) The three curves almost overlap with each other since agents are homogeneous. 
(2) At the initial few steps, the agents are more likely to communicate. This allows agents to know others' initial positions, as well as to probe the wireless and game environments. The sooner they know such information, the better they collaborate. 
(3) The agents are also more likely to communicate when they catch the prey (as can be seen from the similar shape of the curves and the bars after step 10). After the prey-catching agent informs others of the prey's position, the other agents can directly approach the prey following the shortest path.

\begin{figure}[!htb]
    \centering
    \input{./plots/ablation_when}
    \caption{Communication adapted to limited bandwidth}
    \label{fig: exp when conv}
\end{figure}
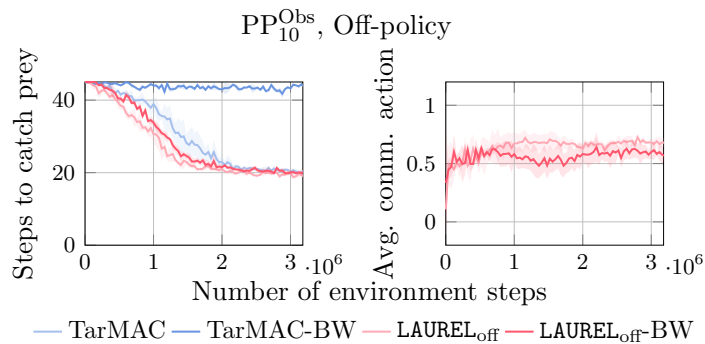

\paragraph{Adapting to Complicated Wireless Environment}
We further study the effect of bandwidth limit on agents' communication behaviors. In Figure \ref{fig: exp when conv}, we consider two wireless environments: one with more bandwidth resources and the other with fewer bandwidth resources (achieved by reducing the number of time slots in slotted $p$-CSMA \citep{gaiyi}). All other parameters remain as default. 
The curves marked by ``-BW'' correspond to those evaluated in the bandwidth reduced environment. 
We observe the different behaviors of TarMAC and {\coolname}. In the right plot, we plot the average communication action $\alpha$. \eg, when $\alpha=0.5$, each agent has 50\% probably of sending a message at each step. 
Since TarMAC agents always send a message in each step ($\alpha=1$), we do not plot their communication curves. 
We observe that: 
(1) The performance of TarMAC is very sensitive to the wireless environment. When the wireless environment becomes more complicated, the total number of steps to catch the prey increases significantly. 
(2) For {\coolname}, making the wireless environment more complicated only results in slightly slower convergence. From the right plot, we observe that {\coolname} successfully adapts its policy to the changed environment by reducing the number of communications.

\setlength{\columnsep}{8pt}%
\begin{wraptable}{r}{7cm}
\centering
\vspace{-8pt}
\caption{Comparison on number of steps}
\vspace{-.2cm}
\resizebox{0.45\textwidth}{!}{
\input{tables/ablation_mlp_gin}
}
\label{tab: exp ablation enc}
\vspace{-.1cm}
\end{wraptable}

\parag{Ablation Study}
We show how our encoder improves the quality of the learned policy. 
In Table \ref{tab: exp ablation enc}, we compare three different encoding architectures, with all other configurations equivalent. 
For the MLP encoder, suppose each message is a length $d$ vector. 
For 3 agents, the input $\bm{e}$ is a length-$3d$ vector. 
The subvector $\left[\bm{e}\right]_{i d:\paren{i+1} d}$ is filled with $i$'s message if message from agent $i$ is received, otherwise, 0s. 
Then $\bm{e}$ is fed into a 2-layer MLP to generate the encoded vector. 
For the ``Avg'' encoder, we first aggregate the received raw messages by vector mean, and then feed the aggregated vector into a 2-layer MLP. 
Clearly, the proposed encoder based on Equation \ref{eq: how gnn} leads to significantly better agents' performance.

\subsection{Lumberjacks}
\label{sec: exp lj}
Lumberjacks \citep{lj_paper} is a multi-resource spatial coverage problem \citep{lumber_liuyan}. In a grid world, $p$ lumberjacks cooperate to chop $q$ trees. Each tree has a ``strength'' $k$, meaning the tree will be chopped when at least $k$ agents are adjacent to it. We set $k=2$ for all trees. 
The agent obtains reward $r_1=0.05$ and $r_2=0.5$ for observing and chopping the tree. 
Step penalty $r_3=-0.1$ is used to encourage agents to chop trees within minimum number of steps. 
Similar to PP, each lumberjack agent can move to its adjacent grid or stay still in a step. 
We modify the original setting \citep{lj_code} by assigning trees signal attenuation $\delta=4.5$. The game is terminated once all trees are chopped or the max number of steps is reached. 
Denote {\lj{n}{p}{q}} as a Lumberjacks game within $n\times n$ grid.

\begin{figure}[!htb]
    \centering
    \input{./plots/sota_lj}
    \vspace{-.3cm}
    \caption{Comparison with state-of-the-art methods}
    \label{fig: sota lj steps}
\end{figure}
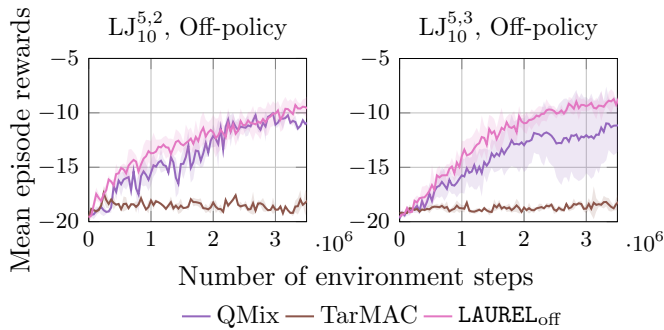

Figure \ref{fig: sota lj steps} compares the convergence of reward between {\coolname} and state-of-the-art methods. 
Both environments contain 5 agents. 
In \lj{10}{5}{2} with 5 agents and 2 trees, QMix achieves similar reward as {\coolname}. 
In \lj{10}{5}{3}, {\coolname} achieves significantly higher reward than QMix. 
In both environments, TarMAC fails to learn a good policy. We conclude the following. 
When there are many agents and few trees, collaboration is less critical -- Even if each agent searches for a tree on its own and then waits there, there is still a high chance that two agents land adjacent to the same tree by coincidence. 
This is why the vanilla QMix performs well in \lj{10}{5}{2}. 
However, when we increase the number of trees in the environment, collaboration becomes more critical, and the communication strategy becomes important. 
So we see a performance drop in QMix in contrast to a performance boost in {\coolname}. 
Finally, the poor performance of TarMAC shows it is important to consider the complicated wireless environment in the algorithm design. 
Otherwise, \emph{communication can be even more harmful than no communication at all}, as shown by the TarMAC-QMix comparison.

%% file: plots/sota_steps.tex
\begin{tikzpicture}[
    scale=0.95, every node/.style={transform shape}
]

\def\colIc{color_ic3net_}
\def\colCommnet{color_commnet_}
\def\colQmix{color_qmix_}
\def\colProposedre{color_laurel_on}
\def\colProposedac{color_laurel_off}
\def\colTarmac{color_tarmac_}
\begin{groupplot}[
    group style={
        group size=2 by 1,
        horizontal sep=12mm,
        vertical sep=12mm
    },
    scale only axis,
    height=0.15\textwidth,
    width=0.2\textwidth,
    tick label style={font=\footnotesize},
    title style={font=\small,at={(axis description cs:0.5, 0.95)}},
    grid,
    y label style={at={(axis description cs:0.2,.5)},anchor=south},
    every x tick scale label/.style={
        at={(xticklabel* cs:1.15,0cm)},
        anchor=near xticklabel
    },
]
\nextgroupplot[
    title={\ppw{10}, On-policy},
    ymin=0,ymax=40,xmin=0,xmax=160000000,
    ylabel={Steps to catch prey},
]

\errorband[thick,color=\colCommnet]{\dataSotaStepsA}{0}{1}{2}{CommNet}
\errorband[thick,color=\colIc]{\dataSotaStepsA}{0}{3}{4}{IC3Net}
\errorband[thick,color=\colProposedre]{\dataSotaStepsA}{0}{5}{6}{\proposedre}
\legend{}

\nextgroupplot[
    title={\ppw{10}, Off-policy},
    ymin=0,ymax=45,xmin=0,xmax=4000000,
    xlabel={Number of environment steps},
    x label style={at={(axis description cs:-.2,-.25)},anchor=south},
    legend cell align=left,
    legend style={at={(0.75,-0.5)},draw=none,font=\small},
    legend columns=3
]
\errorband[thick,color=\colCommnet]{\dataDummy}{0}{1}{2}{CommNet}    
\errorband[thick,color=\colIc]{\dataDummy}{0}{1}{2}{IC3Net}   
\errorband[thick,color=\colProposedre]{\dataDummy}{0}{3}{4}{\proposedre} 
\errorband[thick,color=\colQmix]{\dataSotaPpTarmac}{0}{1}{2}{QMix}
\errorband[thick,color=\colTarmac]{\dataSotaPpTarmac}{0}{3}{4}{TarMAC}
\errorband[thick,color=\colProposedac]{\dataSotaPpTarmac}{0}{5}{6}{\proposedac}

\end{groupplot}

\end{tikzpicture}

%% file: plots/traj_comm.tex
\begin{tikzpicture}[
    scale=0.95, every node/.style={transform shape}
]
\def\colA{red}
\def\colB{blue}
\def\colC{green!60!black}
\def\datafilecnt{./data/term_bincount.csv}

\pgfplotsset{
    scale only axis,
    xmin=0, xmax=40,
    title={\ppw{10}, Off-policy, {\coolname}},
    height=0.15\textwidth,
    width=0.4\textwidth,
    title style={at={(axis description cs:0.5, 1.05)}},
    x label style={at={(axis description cs:0.5,.06)}},
    xlabel={Steps in one trajectory},
    ytick style={draw=none},
    xtick style={draw=none},
}

\begin{axis}[
    axis y line*=right,
    axis x line=none,
    ymin=0, ymax=0.15,
    ylabel=Prob. of first prey-catching,
    y label style={at={(axis description cs:1.32,.45)},anchor=south,rotate=180},
    ybar=\pgflinewidth,
    bar width=1pt,
    ytick={0, 0.05, 0.10, 0.15},
    yticklabels={.0, .05, .10, .15},
]
\addplot[color=orange,fill=orange] table[x=step, y=term_cnt, col sep=comma] {\datafilecnt};
\end{axis}

\begin{axis}[
    ymin=0, ymax=1,
    ylabel={Comm. action per agent},
    y label style={at={(axis description cs:0.1,.45)},anchor=south},
    yticklabels={.0, .2, .4, .6, .8, 1},
    ytick={0, 0.2, 0.4, 0.6, 0.8, 1},
    legend cell align=left,
    legend style={at={(1.,-0.4)},draw=none,font=\small},
    legend columns=3
]

\errorband[thick,color=\colA]{\dataTrajComm}{0}{1}{7}{Agent 1}  
\errorband[thick,color=\colB]{\dataTrajComm}{0}{3}{7}{Agent 2}   
\errorband[thick,color=\colC]{\dataTrajComm}{0}{5}{7}{Agent 3}  
\end{axis}

\end{tikzpicture}

%% file: plots/ablation_when.tex
\begin{tikzpicture}[
    scale=0.95, every node/.style={transform shape}
]
\def\colAlways{green!60!black}
\def\colNever{blue}
\def\colPolicy{red}

\definecolor{colTarFull}{HTML}{A6C0ED}
\definecolor{colTarLimited}{HTML}{6A96E1}
\definecolor{colWinkFull}{HTML}{FFABB6}
\definecolor{colWinkLimited}{HTML}{FF576D}

\begin{groupplot}[
    group style={
    group size=2 by 1,
    horizontal sep=20mm,
    vertical sep=7mm},
    scale only axis,
    height=0.15\textwidth,
    width=0.2\textwidth,
    tick label style={font=\footnotesize},
    title style={at={(axis description cs:1.2, 1.1)}},
    grid,
    every x tick scale label/.style={
        at={(xticklabel* cs:1.1,0cm)},
        anchor=near xticklabel
    },
]
\nextgroupplot[
    title={\ppw{10}, Off-policy},
    ymin=0,ymax=45,xmin=0,xmax=3180000,
    ylabel={Steps to catch prey},
    y label style={at={(axis description cs:0.23,.5)},anchor=south},
]
\errorband[thick,color=colTarFull]{\dataBwLimitSteps}{0}{1}{2}{TarMAC}      
\errorband[thick,color=colTarLimited]{\dataBwLimitSteps}{0}{5}{6}{TarMAC-BW}       
\errorband[thick,color=colWinkFull]{\dataBwLimitSteps}{0}{9}{10}{\proposedac}    
\errorband[thick,color=colWinkLimited]{\dataBwLimitSteps}{0}{13}{14}{{\proposedac}-BW}
\legend{}

\nextgroupplot[
    ymin=-0.2,ymax=1.2,xmin=0,xmax=3180000,
    xlabel={Number of environment steps},
    ylabel={Avg. comm. action},
    y label style={at={(axis description cs:.23,.5)},anchor=south},
    x label style={at={(axis description cs:-.35,-.2)},anchor=south},
    legend cell align=left,
    legend style={at={(1.2,-0.4)},draw=none,font=\small},
    legend columns=4]
\errorband[thick,color=colTarFull]{\dataDummy}{0}{1}{2}{TarMAC}    
\errorband[thick,color=colTarLimited]{\dataDummy}{0}{1}{2}{TarMAC-BW}        
\errorband[thick,color=colWinkFull]{\dataBwLimitSteps}{0}{11}{12}{\proposedac}    
\errorband[thick,color=colWinkLimited]{\dataBwLimitSteps}{0}{15}{16}{{\proposedac}-BW}    

\end{groupplot}

\end{tikzpicture}

%% file: tables/ablation_mlp_gin.tex
\begin{tabular}{lccc}
    \toprule
    Enc. & MLP & Avg & Eq. \ref{eq: how gnn} \\
    \midrule
    \midrule
    \ppw{10} & 24.03\std{5.12} & 21.24\std{0.98} & 18.71\std{0.58}\\

    \bottomrule
\end{tabular}

%% file: plots/sota_lj.tex
\begin{tikzpicture}[
    scale=0.95, every node/.style={transform shape}
]

    \def\colIc{black}
    \def\colCommnet{blue}
    \def\colQmix{color_qmix_}
    \def\colProposedre{orange}
    \def\colProposedac{color_laurel_off}
    \def\colTarmac{color_tarmac_}
    \begin{groupplot}[
        group style={
        group size=2 by 1,
        horizontal sep=13mm,
        vertical sep=12mm},
        scale only axis,
        height=0.15\textwidth,
        width=0.2\textwidth,
        tick label style={font=\footnotesize},
        title style={font=\small,at={(axis description cs:0.5, 0.95)}},
        grid,
        y label style={at={(axis description cs:0.2,.5)},anchor=south},
        every x tick scale label/.style={
            at={(xticklabel* cs:1.13,0cm)},
            anchor=near xticklabel
        },
    ]
    \nextgroupplot[
        title={\lj{10}{5}{2}, Off-policy},
        ymin=-20,ymax=-5,xmin=0,xmax=3500000,
        ylabel={Mean episode rewards},
    ]
    
    \errorband[thick,color=\colQmix]{\dataSotaLj}{0}{1}{2}{QMix}   
    \errorband[thick,color=\colTarmac]{\dataSotaLj}{0}{3}{4}{TarMAC}    
    \errorband[thick,color=\colProposedac]{\dataSotaLj}{0}{5}{6}{\proposedac} 
    \legend{}
    
    \nextgroupplot[
        title={\lj{10}{5}{3}, Off-policy},
        ymin=-20,ymax=-5,xmin=0,xmax=3500000,
        xlabel={Number of environment steps},
        x label style={at={(axis description cs:-.2,-.25)},anchor=south},
        legend cell align=left,
        legend style={at={(.8,-0.45)},draw=none,font=\small},
        legend columns=6]
    \errorband[thick,color=\colQmix]{\dataSotaLj}{7}{8}{9}{QMix}     
    \errorband[thick,color=\colTarmac]{\dataSotaLj}{7}{10}{11}{TarMAC} 
    \errorband[thick,color=\colProposedac]{\dataSotaLj}{7}{12}{13}{\proposedac} 
    
    \end{groupplot}
    
    \end{tikzpicture}
    

%% file: 7_appendix.tex
\appendix
\section{Log Distance Path Loss with Fading}
\label{appendix: pathloss}
\begin{align}
\label{eq: path loss model}
P_r = P_t - K_\text{ref} - 10\eta\log_{10}\frac{d}{d_{0}}+\psi
\end{align}

where $P_{r}$ is the received power in dBm, $P_t$ is the transmission power in dBm, $K_\text{ref}$ is the loss at reference distance $d_{0}$, $\eta$ is the path loss component depending on the propagation medium, and $\psi$ is a log normal variable for multi-path fading. This equation describes the change of signal strength during the propagation.

\section{Proofs}
\label{appendix: proof}

\begin{proof}[Proof of Theorem \ref{thm: gin emb}]
To prove permutation invariance, note that vector addition is commutative. 
So for any permutation $\rho$, the sum of the sequence $\paren{\func[MLP]{\bm{m}_{i\rho\paren{1}}},\hdots,\func[MLP]{\bm{m}_{i\rho\paren{k}}}}$ is always the same. 

For injectiveness, we recall the conclusion from Lemma 5 of \cite{gin}. 

\begin{lemma}
Assume the space of messages, $\mathcal{M}$, is countable. 
There exists $f : \mathcal{M}\rightarrow \R[d']$, s.t. $\sum_{j\in\mathcal{N}_i} f\paren{\bm{m}_{ij}}$ is unique for each $c_i\subset \mathcal{M}$. 
\end{lemma}

Then we only need to show an MLP can express the function $f$ specified by the above lemma. 
This is a direct consequence of the universal approximation theorem \citep{mlp_universal}. 
In other words, there exists a set of MLP parameters such that $\Phi\paren{c_i}=\sum_{j\in\mathcal{N}_i}\func[MLP]{\bm{m}_{ij}}=\sum_{j\in\mathcal{N}_i}f\paren{\bm{m}_{ij}}$, where $c_i=c_i'\Leftrightarrow \Phi\paren{c_i}=\Phi\paren{c_i'}$. 
So $\Phi$ is injective. 

Note that the requirement on the input space $\mathcal{M}$ being countable is automatically satisfied in realistic applications, where the message $\bm{m}$ is quantized to fixed bit widths. 
\end{proof}

\section{Details in Experiments}

\subsection{Communication Protocols}
\label{appendix: communication protocols}
\parag{Slotted $p$-CSMA} When an agent intends to transmit a packet, it first randomly choose a counter in the range of ``counter window size''. When counter is down to 0, it senses the medium.
If the medium is free (\ie, sensed signal strength lower than the threshold $\Theta_f$), the agent transmits with probability $p$.
Otherwise, it chooses a random counter again and repeats the process.

\parag{Arbitration of successful receiving} We use SINR (signal-to-interference-plus-noise ratio). The receiver agent will sense the arrived packet's signal strength as $\theta$. Denote $\theta'$ as the interfering signal strength at the agent if exist, $N$ as the background noise, $\Theta_{r}$ as the SINR threshold. If $\frac{\theta}{\theta' + N} < \Theta_{r}$, the packet cannot be decoded correctly.  
In other words, there are three cases for the receiving signal strength:
(1) If $\theta > \Theta_r\times (\theta' + N) $, then the agent successfully receives the packet;
(2) If $\Theta_f \leq \theta \leq \Theta_r \times (\theta' + N)$, then the agent knows there is a packet coming, but it cannot successfully receive the contents;
(3) If $\theta < \Theta_f$, then the agent doesn't even know there is a packet arrived.

\subsection{Parameters}
\label{appendix: wireless parameters}
\parag{Default settings used in Section \ref{sec: exp}} Obstacle attenuation $\att_0 = 4.5$; noise $N=-95$dBm; RSS threshold $\Theta_f=-78$dBm; SINR threshold $\Theta_r=15$ dB  for \ppw{n} and $20$dBm for others; contention probability in $p$-CSMA protocol $p = 0.3$; counter window size in $p$-CSMA protocol $W = 15$ time slots.

\parag{Training Hyperparameters}
\label{appendix: training hyperparameters}
We repeat all experiments 5 times with different random seeds.
We report the average performance and variance of the metrics. 
For the on-policy algorithms, we use a LSTM cell with hidden dimension 128 in the policy net. 
For the off-policy algorithms, we use a GRU cell with a 2-layer MLP of hidden dimension 128 as the $Q$-value network, and 3-layer MLP with hypernetwork \citep{hypernet} as the mixing network. 
For both versions of {\coolname}, we use a 2-layer MLP with hidden dimension 128 to implemented the message encoder according to Equation \ref{eq: how gnn}. 
All algorithms are trained with Adam optimizer \citep{adam} with learning rate 0.0005.

\subsection{More Complicated Environments}
\label{sec: exp complicated}

\noindent
\begin{minipage}[c]{\columnwidth}
\centering
    \captionof{table}{Comparison on number of steps in more complicated environment}
    \label{tab: exp complicated env}
        \input{tables/complicated_env}
\end{minipage}

We further introduce more complexity into the environment to see if we can still achieve performance gain. 
The complexity of the three environment in Table \ref{tab: exp complicated env} increases from left to right.
In \ppw{n}, the wireless channel has a small noise level fixed across episodes. In \ppwn{10}, we introduce random background noises
$N \sim \textbf{U}(-95, -90)$
across episodes (within an episode, the noise is fixed). This is inspired by the fact that background noise at the same location can vary significantly throughout the day\citep{noise_24}. 
 In \ppwna{10}, we further introduce obstacles with randomly selected attenuation from the candidate set $\attset$. The motivation is that obstacle of different materials have different attenuation effect on the wireless signals passing it. Specifically, for \ppw{10} and \ppwn{10}, we set the attenuation (unit: dB) $\attset=\{4.5\}, \attset=\{2.5\}$ respectively. For \ppwna{10}, $\attset=\{2.5, 3.5, 4.5\}$.

For all three environments, our {\proposedre} achieves the best game performance, since the predators can catch the prey in significantly less number of steps. This means agents are able to understand and adapt its communication to the wireless network condition.
The performance gain is higher when the wireless environment is more complicated. 
In real-life applications, the environments are even more complicated. 
In such cases, we expect it may be necessary to further generalize our framework following ideas in Section \ref{sec: method complete}.

%% file: tables/complicated_env.tex
\begin{tabular}{lccc}
    \toprule
    Method & \ppw{10} & \ppwn{10}& \ppwna{10} \\
    \midrule
    \midrule
    IC3Net & 25.59\std{7.80} & 26.63\std{8.58} & 32.76\std{9.26}\\
    {\proposedre} & 20.51\std{0.62} & 20.95\std{0.80} & 21.09\std{0.85}\\
    \midrule
    Gain & 5.08 & 5.68 & 11.67\\
    \bottomrule
\end{tabular}